\documentclass{article}

 \usepackage[main, final]{neurips_2026}

\usepackage[utf8]{inputenc} 
\usepackage[T1]{fontenc}    
\usepackage{hyperref}       
\usepackage{url}            
\usepackage{booktabs}       
\usepackage{amsfonts}       
\usepackage{nicefrac}       
\usepackage{microtype}      
\usepackage{xcolor}         
\usepackage{amsmath}
\usepackage{amssymb}
\usepackage{optidef}
\usepackage{mathtools}
\usepackage{bbm}
\usepackage{amsthm}
\usepackage[most]{tcolorbox}
\usepackage{multirow}
\usepackage{graphicx}
\usepackage{subcaption}
\usepackage[export]{adjustbox}
\usepackage{caption}
\usepackage{enumitem}

\usepackage[capitalize]{cleveref}
    \crefname{section}{Sec.}{Secs.}
    \Crefname{section}{Section}{Sections}
    \Crefname{table}{Table}{Tables}
    \crefname{table}{Tab.}{Tabs.}
    
\newtheorem{theorem}{Theorem}

\newif\ifdraft
\draftfalse

\ifdraft
  \newcommand{\PF}[1]{{\color{red}{\bf PF: #1}}}
  \newcommand{\pf}[1]{{\color{red} #1}}
  \definecolor{fsgreen}{rgb}{0.3, 0.7, 0.3}
  \newcommand{\FS}[1]{{\color{fsgreen}{\bf FS: #1}}}
  \newcommand{\fs}[1]{{\color{fsgreen} #1}}
  \newcommand{\DM}[1]{{\color{blue}{\bf DM: #1}}} 
  \newcommand{\dm}[1]{{\color{blue} #1}}
  \newcommand{\NT}[1]{{\color{violet}{\bf NT: #1}}} 
  \newcommand{\nt}[1]{{\color{violet} #1}}
\else
  \newcommand{\PF}[1]{}
  \newcommand{\pf}[1]{#1}
  \newcommand{\FS}[1]{}
  \newcommand{\fs}[1]{#1}
  \newcommand{\DM}[1]{}
  \newcommand{\dm}[1]{#1}
  \newcommand{\NT}[1]{}
  \newcommand{\nt}[1]{#1}
\fi

\newcommand{\parag}[1]{\vspace{-3mm}\paragraph{#1}}

\DeclareMathOperator{\mint}{min^{(2)}} 
\DeclareMathOperator{\mino}{min^{(1)}} 
\DeclareMathOperator{\argmino}{argmin^{(1)}}

\newcommand{\bp}[0]{\mathbf{p}}
\newcommand{\vp}[0]{\mathbf{p}}
\newcommand{\vq}[0]{\mathbf{q}}
\newcommand{\bu}[0]{\mathbf{u}}
\newcommand{\by}[0]{\mathbf{y}}

\newcommand{\acro}[0]{S2MDF}
\newcommand{\real}[0]{\mathbb{R}}

\title{S2MDF: A Plug-And-Play Layer for Intersection-Free Multi-Object Signed Distance Fields}

\author{
	\textbf{Deniz Sayin Mercadier}\textsuperscript{\textbf{*}}\hspace{1em}\textbf{Federico Stella}\textsuperscript{\textbf{*}}\hspace{1em}\textbf{Aurel Bizeau}\hspace{1em}\textbf{Nicolas Talabot}\hspace{1em}\textbf{Pascal Fua} \\
	CVLab, Ecole Polytechnique Fédérale de Lausanne (EPFL) \\
	Lausanne, Switzerland \\
	\texttt{\{firstname.lastname\}@epfl.ch}
}

\begin{document}

\maketitle

\begingroup
\renewcommand{\thefootnote}{\fnsymbol{footnote}}
\footnotetext[1]{Equal contribution.}
\endgroup


\begin{abstract}

Compositional implicit surface representations model scenes as collections of objects, each encoded by a Signed Distance Field (SDF). A fundamental limitation of this approach is that multiple SDFs can produce geometries that interpenetrate, violating physical plausibility. Existing mitigation strategies rely on soft penalty terms that reduce but do not eliminate intersections, and require careful loss weighting. To truly prevent interpenetration, we propose a hard constraint on vector-valued SDFs and  introduce S2MDF, a lightweight plug-and-play module that enforces the constraint on any object-compositional SDF representation without architectural modifications. It introduces negligible computational overhead and is compatible with linearly-interpolated standard meshing algorithms such as Marching Cubes. It can be applied during training or as a post-processing step. Experiments on multiple state-of-the-art compositional methods show that S2MDF reduces intersections to numerical precision while preserving reconstruction quality, outperforming existing mitigation strategies.

   \end{abstract}

\section{Introduction}
\label{sec:intro}

Over the past few years, implicit neural representations (INRs) have reshaped the landscape of 3D vision and graphics. By encoding geometry or appearance as continuous functions parameterized by neural networks, methods such as DeepSDF~\citep{Park19c}, 3DShape2VecSet~\citep{Zhang23d}, NeRF~\citep{Mildenhall20}, NeuS~\citep{Wang21f} and their many extensions have demonstrated remarkable performance in tasks such as novel view synthesis~\cite{Tancik23}, high-fidelity surface reconstruction~\citep{Ge23}, shape generation~\citep{Lai25a}, manipulation~\citep{Hao20a} and optimization~\citep{Jiang20a}. Unlike older discrete representations such as meshes and voxels, INRs provide resolution-independent modeling, differentiability, and a natural framework for learning complex geometric priors. This paradigm shift has fostered significant progress in fields as diverse as robotics, medical imaging, and digital content creation, where accurate and flexible 3D representations are paramount. 

As these methods mature, a growing body of work has focused on extending INRs beyond single-object descriptions towards more realistic multi-object representations of scenes~\citep{Wu22,Tertikas23,Wu23b} and multi-part models of composite objects~\citep{Deng22b,Petrov23,Talabot25}. These are particularly valuable for downstream tasks such as manipulation, simulation, and design, where reasoning about individual entities or parts is key. For example, when designing the 3D shape of a car to improve its aerodynamics, one has to consider its body, its wheels, and all the accessories that are usually attached to it. Similarly, when modeling a heart from medical imagery, one should consider its four chambers and the blood vessels attached to them. Numerous object-compositional implicit representations~\citep{Zhi21,Zhang22d,Liu23j,Wu23b,Talabot25,Le25a} have been introduced to this end. A common approach is to rely on a set of Signed Distance Fields (SDF)~\citep{Wang21f}, one per object~\citep{Wu22,Wu23b,Talabot25}, and to combine them appropriately. 

While effective and popular, this approach faces a fundamental challenge: Even though physical objects cannot share the same volume in space, multiple SDFs can easily produce configurations that violate this. A number of strategies have been proposed to mitigate this problem. They include minimizing loss terms that penalize overlap~\citep{Talabot25}, visibility-based constraints to handle occlusions~\citep{Wu23b}, and regularization techniques inspired by physical reasoning~\citep{Le25a}. While these approaches reduce the frequency and severity of intersections, they do not guarantee their complete removal. Furthermore, because the intersection loss terms are usually minimized as part of a weighted sum of several other loss terms, careful tuning of their respective weights is often required. 

In this work, we propose an approach that strictly prevents intersections without needing a weighting scheme. To this end, we define a hard-constraint on vector-valued SDFs that produces an intersection-free geometry. We will refer to such a constrained field as a \textit{Multi-object Distance Field (MDF)}. Given this formalism, we implement \textit{S2MDF}, a simple yet effective module that enforces the constraint and can be seamlessly integrated into existing pipelines with a negligible computational overhead. It can be applied either during training or as a post-processing step to enforce the intersection-free constraint on virtually any object-compositional SDF-based representation, without requiring any architectural change. Furthermore, it remains compatible with standard linearly-interpolated meshing techniques such as Marching Cubes~\citep{Lorensen87} and Marching Tetrahedra~\citep{Chan98}. 

To summarize our contributions:
\begin{itemize}[nosep]
    \item We formalize the problem of generating intersection-free multi-object implicit surface representations by constraining vector-valued SDFs to encode intersection-free geometry as a Multi-object Distance Field (MDF); 
    \item We propose S2MDF, a fast, plug-and-play module that projects object-compositional SDF-based representations into an MDF, ensuring intersection-free geometry with negligible computational overhead.
\end{itemize}
We validate our approach on several state-of-the-art SDF-based object-compositional methods, demonstrating that S2MDF, paired with existing linearly-interpolated meshing techniques, reduces intersections to numerical precision while preserving reconstruction quality. Through both qualitative and quantitative evaluations, we show that our method outperforms existing mitigation strategies and provides a robust, generic solution to a longstanding limitation of compositional neural surface representations.

\section{Related work}
\label{sec:related}

\pf{We first put Implicit Neural Representations in the more general context of representing the 3D shape of single objects and then briefly review existing approaches to jointly model several objects.}


\parag{Representing the Shape of 3D Objects.} Many different representations have been proposed over the years. Voxel-based methods represent the space as a grid~\citep{Wu15b,Wu16b,Choy16b,Dai17a}, but suffer from high memory requirements and limited resolution. Point-cloud-based methods represent surfaces as sets of points~\citep{Fan17a,Yang18a,Achlioptas18b,Peng21a,Zeng22}, which improves the memory footprint but ignores connectivity, as they do not explicitly model the surface as a continuous entity. Mesh-based methods represent surfaces as collections of vertices, edges and faces~\citep{Groueix18a,Wang18e,Kanazawa18b,Pan19}, which allows for efficient rendering and manipulation, but generally require a fixed topology that is difficult to modify. Implicit Neural Representations represent a powerful alternative, encoding geometry as continuous functions parameterized by neural networks~\citep{Park19c,Mescheder19,Chen19c}. Among these, Signed Distance Fields (SDFs) and occupancy fields are particularly popular for representing watertight surfaces, which can then be extracted from the field using standard meshing techniques~\citep{Lorensen87,Doi91,Lewiner03,Ju02}. When the surface is not watertight, Unsigned Distance Fields (UDFs) can be used instead~\citep{Chibane20b,Long22,Liu23a}, with specially designed meshing techniques~\citep{Guillard22b,Zhang23b,Stella25}. Differentiability between the explicit surface and the field can be achieved thanks to the implicit function theorem, whose practical viability has been shown for both SDFs~\citep{Remelli20b,Guillard24a} and UDFs~\citep{Guillard22b}, allowing shape optimization tasks such as drag reduction. Recent works extend these approaches using new latent spaces~\citep{Takikawa21,Mittal22,Zhang22b,Zhang23d} and more complex network architectures~\citep{Zhang24a,Chen25b}, allowing for highly detailed generative modeling and shape manipulation.

\parag{Handling Multiple Objects Jointly.}

Object-aware implicit representations have been proposed in the context of NeRFs~\citep{Niemeyer21,Tertikas23}, often involving a separate semantic branch to predict object labels~\citep{Wang22g,Dumery25a}, and in some cases two separate radiance fields, one for the scene and one for the objects~\citep{Xie21c}. However these approaches do not directly model geometry. ObjectSDF~\citep{Wu22} and ObjectSDF++~\citep{Wu23b} are among the first works to propose object-compositional implicit surface representations, where multiple SDFs are used to represent different objects. ObjectSDF~\citep{Wu22} defines an approach to directly render the semantic field, coupling semantics and geometry, while ObjectSDF++~\citep{Wu23b} supervises the geometry with semantic information and adds a loss term to penalize object collisions. However, neither approach prevents interpenetrations between objects. \nt{Part-based generative models~\citep{Wu20c,Deng22b,Petrov23} represent individual parts of objects but typically output independent, potentially non-physical, SDFs. The more recent PartSDF~\citep{Talabot25} proposes a single decoder to produce all fields jointly,} with a loss term to penalize intersections, but again without guarantees. Other works in the medical and robotics domains, where physical plausibility is key, also use special losses to reduce penetration between parts~\citep{Karunratanakul20,Zhang22d,Le25a}, while sometimes enforcing contact~\citep{Le25a}, but they either do not guarantee the absence of intersections or they are not based on INRs. Instead, we propose a simple module that can be applied to any object-compositional SDF-based representation, either during training or as a post-processing step, to prevent interpenetration while preserving the quality of the reconstructions.

\newcommand{\vd}{\mathbf{d}}
\newcommand{\vu}{\mathbf{u}}

\section{Method}

Signed Distance Fields (SDFs) naturally represent single, watertight, objects by encoding the positive or negative distance to the closest surface for each point in space, the typical convention being that negative values are inside the object and positive ones outside. For a watertight object $O$ and its boundary $\partial O$, this can be written as
\begin{align}
	SDF_O: \mathbb{R}^3 \to \mathbb{R} \; , \;\;\;
	SDF_O(\mathbf{p}) = \begin{cases}
		-\min_{\mathbf{q} \in \partial O} \| \mathbf{p} - \mathbf{q} \|, & \text{if } \mathbf{p} \in O \\
		\phantom{-}\min_{\mathbf{q} \in \partial O} \| \mathbf{p} - \mathbf{q} \|, & \text{if } \mathbf{p} \notin O
	\end{cases}
	\label{eq:sdf}
\end{align}
In an object-compositional setting, each one of $K$ objects is \fs{typically} assigned its own $SDF_k$ for $1 \leq k \leq K$ and we can define the composite $SDF = [SDF_1, \ldots , SDF_K] : \mathbb{R}^3 \to \mathbb{R}^K$. The problem with this \fs{is that each SDF is independent from the others:} Nothing stops two of the $SDF_k$ from returning negative values at the same spatial location, which would mean that this location is inside two different objects, a configuration that should not be allowed. This is what we want to prevent by appropriately constraining $SDF$. We first formalize this constraint and then discuss various approaches to enforcing it.

\subsection{Multi-Object Signed Distance Fields}

Let $\min^{(n)}(\cdot)$ be the operator that returns the $n$-th smallest element of its input vector. Given a spatial location $\bp$, let $\bu = SDF(\bp)$ be its vector-valued composite SDF. To ensure that $\bp$ is not inside more than one object \dm{it} should satisfy
\begin{equation}
\label{eq:og_constraint} 
\mint(\bu) \ge 0 \; .
\end{equation}
This constraint ensures that the second smallest signed distance value at any point is non-negative, which means that at most one of the $SDF_k$ can be negative at any location. If we could truly enforce it everywhere, this would be enough. But, in practice, we have to sample the space more or less densely and we can only enforce the constraint at those sample points. Furthermore, the deep networks used to implement the $SDF_k$ functions do not guarantee that they satisfy the eikonal property $\|\nabla SDF_k\| = 1$. Taken together, these two limitations mean that SDF \fs{approximations} can satisfy the constraint of \cref{eq:og_constraint} on sample points and yet produce interpenetrations, as shown in \cref{fig:min2_counterexample}. To remedy this, we propose to use an equivalent constraint based on the following theorem 

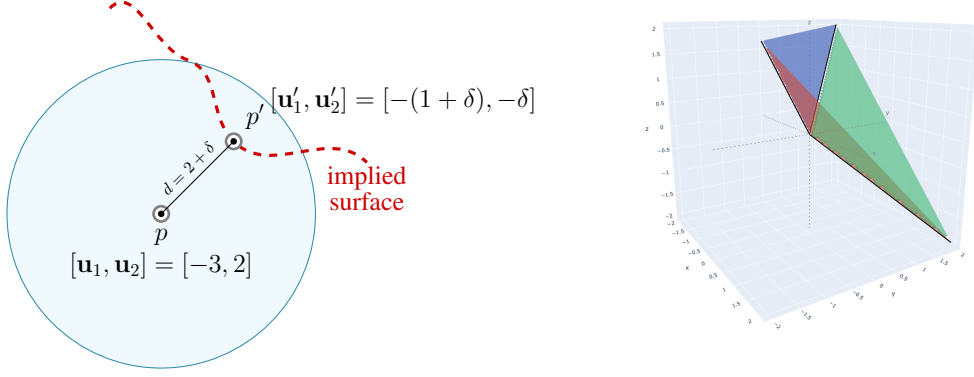
\begin{figure}[h]
    \centering
    \begin{subfigure}[b]{0.48\textwidth}
        \centering
        \begin{tikzpicture}[>=stealth, font=\Large, scale=0.68, transform shape]
    \path[use as bounding box] (-3, -3) rectangle (7.05, 4);

    \draw[line width=0.3pt, cyan!60!black, fill=cyan!5!white] (0,0) circle (3);


    \coordinate (p0) at (0, 0);
    \node[circle, draw=gray, line width=1pt, fill=white, inner sep=3pt, label=below:{$p$}] at (p0) {};
    \node at (0, -1.0) {$[\vu_1, \vu_2] = [-3, 2]$};
    \coordinate (p1) at (1.61, 1.61);

    
    \draw[line width=1.2pt, dashed, red!80!black] 
        (-1, 4) 
        to[out=45, in=160]  (0.5, 3)       
        to[out=-20, in=135] (1.41, 1.41)   
        to[out=-45, in=135] (4, 1);        
    \node[text=red!80!black] at (4, 0.7) {implied};
    \node[text=red!80!black] at (4, 0.25) {surface};
    \node[circle, draw=gray, line width=1pt, fill=white, inner sep=3pt, label=above right:{$p'$}] at (p1) {};
    \node at (4.71, 2.21) {$[\vu_1', \vu_2'] = [-0.9, -0.1]$};

    
    \filldraw[black] (0, 0) circle (1.5pt) ;
    \filldraw[black] (1.61, 1.61) circle (1.5pt) ;
    \draw[black] (0,0) -- node[above, sloped, text=black, font=\small] {$d=2.1$} (1.61, 1.61);

    \coordinate (v0) at (6,1.5);
    \coordinate (vprime_circle) at (7.5,1.5);


\end{tikzpicture}
        \caption{Given a point $\mathbf{p}$ inside the circle and its SDF $\vu$ satisfying $\mint(\vu) \ge 0$, there is an implication of a nearby point $\mathbf{p'}$ with SDF $\vu'$ having both entries negative, thus $\mint(\vu') \ngeq 0$. This value of $\vu$ would not be allowed if \cref{eq:og_constraint} was satisfied everywhere.}
        \label{fig:min2_counterexample}
    \end{subfigure}
    \hfill
    \begin{subfigure}[b]{0.48\textwidth}
        \centering
        \raisebox{15pt}{\includegraphics[width=\textwidth]{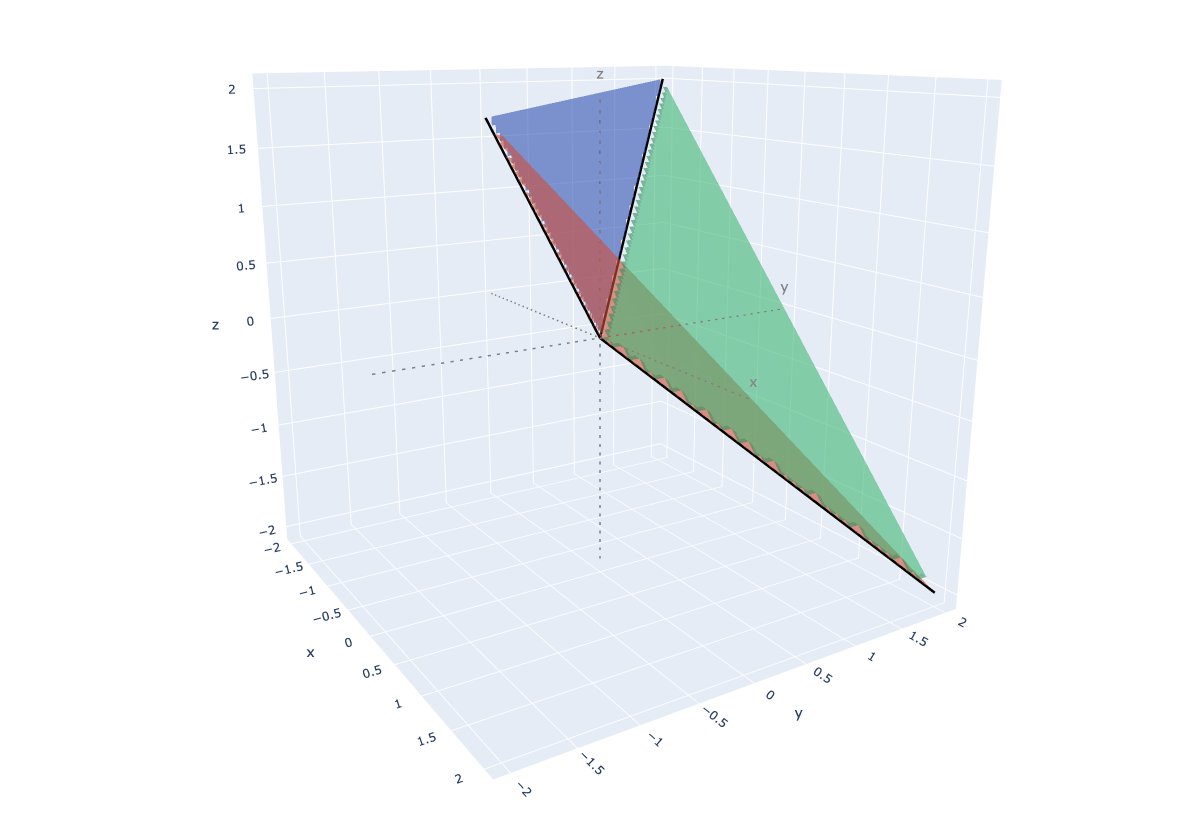}}
        \caption{Illustration of the feasible space for the MDF constraint in \cref{eq:mdf_constraint} for $K=3$ objects: an intersection of $C(K, 2)$ hyperplanes in $\mathbb{R}^K$ \nt{delimits the feasible region on the right}. Each plane corresponds to the constraint imposed by the sum of one pair of SDFs.}
        \label{fig:mdf_projection_planes}
    \end{subfigure}
    \caption{Illustrations of properties related to the constraints in \cref{eq:og_constraint} and \cref{eq:mdf_constraint}.}
    \label{fig:method_support}
\end{figure}

\fs{
\begin{theorem}
\label{thm:constraint_equivalence}
A vector-valued $SDF = [SDF_1, \ldots , SDF_K] : \mathbb{R}^3 \to \mathbb{R}^K$ whose components are exact SDFs as in \cref{eq:sdf} satisfies constraint \ref{eq:og_constraint} at every point $\mathbf{p} \in \mathbb{R}^3$, \nt{with $\bu = SDF(\bp)$}, if and only if it satisfies the following constraint at every point $\mathbf{p} \in \mathbb{R}^3$:
\begin{equation}
    \mino(\bu) + \mint(\bu) \ge 0 \; .
    \label{eq:mdf_constraint}
\end{equation}
Proof: Appendix~\ref{app:constraint-equivalence}, using Lipschitz continuity of exact SDFs.
\end{theorem}
}

\fs{This constraint is visualized in \cref{fig:mdf_projection_planes}.} \fs{Even though \cref{eq:og_constraint} and \cref{eq:mdf_constraint} are equivalent for exact SDFs sampled at every point in space, the second constraint is \textit{stronger} for neural, imperfect SDFs, sampled on a discrete set of points, which is a typical case in practice. In such case, \cref{eq:mdf_constraint} remains sufficient for \cref{eq:og_constraint}, but it is no longer necessary. In fact, the following holds: }
\begin{theorem}
\label{th:linearity}
Given two points $\mathbf{p_1}, \mathbf{p_2} \in \mathbb{R}^3$ and an arbitrary vector-valued function $f : \mathbb{R}^3 \to \mathbb{R}^K$ evaluated at those points, $\bu_1 = f(\mathbf{p_1})$ and $\bu_2 = f(\mathbf{p_2})$ that satisfy constraint \ref{eq:mdf_constraint}, then the constraint is also satisfied for every point along the line segment connecting $\mathbf{p_1}$ and $\mathbf{p_2}$, where intermediate vectors are obtained via linear interpolation as $\bu_{\alpha} = \alpha \bu_1 + (1 - \alpha) \bu_2$ for $\alpha \in [0, 1]$. Proof: Appendix~\ref{app:linterp-proof}.
\end{theorem}

\fs{
In other words, a linear interpolation of an imperfect SDF that satisfies constraint \ref{eq:mdf_constraint} on a finite set of points, also satisfies the constraint everywhere. The same is not true for constraint \ref{eq:og_constraint}, as shown in Appendix~\ref{app:og_constraint_weaker} and hinted in \cref{fig:min2_counterexample}. This is important because, as mentioned above, SDF parametrizations are not perfect, and are often evaluated at a finite number of points and interpolated in between, for example by iso-surface extraction algorithms.} Thus the formulation of~\cref{thm:constraint_equivalence} is the one we use in practice and we define our Multi-object Distance Fields as follows:



%
%


\parag{Definition 1.} A Multi-object Distance Field (MDF) is a vector-valued SDF $\mathbb{R}^3 \to \mathbb{R}^K$ that maps each point in space to a vector of $K$ signed distance values, one for each of $K$ objects, and satisfies the constraint of \cref{eq:mdf_constraint}.


\subsection{Enforcing the MDF Constraint}
\fs{Given a vector-valued SDF, independently from its parametrization, }we want to ensure that the $K$-dimensional vector $\bu=SDF(\bp)$ meets the constraint of ~\cref{eq:mdf_constraint} for all $\bp$ it is evaluated at. The natural way to do this is to perform a projection, that is, replace $\bu$ by the vector $\vd$ that is closest to $\bu$ and meets the constraints. We write this as
\begin{equation}
{\rm minimize} \| \vd - \vu \|^2  \quad \mbox{subject to} \quad {\mino (\vd) + \mint (\vd)}{\ge 0 } \; . \label{eq:proj}
\end{equation}
\fs{\cref{th:linearity} ensures that the linear interpolation of those samples \dm{satisfies} the MDF constraint}. Thus ~\cref{eq:proj} gives us a separate problem to solve for each sample point.
As we normally perform computations on batches of points, the solution must be regular enough to be easy to parallelize for a fast computation. When $K=2$, ~\cref{eq:proj}  has a simple analytical solution and this is no issue. However, for $K>2$, the non-linearity of $\mino$ and  $\mint$ becomes a problem and there are no standard solvers that can be used for our purposes. We implemented two distinct solutions to this problem. 

\subsubsection{Quadratic Programming}

The minimization of \cref{eq:proj} is subject to a non-linear constraint but we can rewrite it as
\begin{equation}
{\rm minimize} \| \vd - \vu \|^2  \quad \mbox{subject to} \quad d_i + d_j{\ge 0, \quad 1 \le i < j \le K} \; . \label{eq:qp}
\end{equation}
We prove in Appendix~\ref{app:pairwise-proof} that the \dm{constraints} of \cref{eq:proj,eq:qp} are {\it strictly} the same, \dm{meaning the problem and by extension its solutions are the same}. While the formulation of \cref{eq:qp} may seem more complex, it has the advantage to be a standard Quadratic Program (QP) for which there are  mature solvers~\citep{Amos17,Bambade24} \pf{whose solution can be readily differentiated}\NT{needs a ref or is accepted?}. Thus, they can be directly integrated inside neural networks for an acceptably small runtime cost \pf{as long as $K$ is not too large}. In our implementation, we embed one in the final layer of the neural network parameterizing all the signed-distance values, effectively projecting SDF values onto the manifold of intersection-free MDFs.


\subsubsection{Analytical Solution}

The drawback \dm{of our QP formulation is that the number of constraints, and by extension the runtime cost of solving the problems}, scales quadratically with the number of objects $K$, which can become prohibitive, especially during training when they are invoked at every iteration. This can be addressed using the following heuristic. When $K=2$, the analytical solution to \cref{eq:proj} is to add $-(\mino (\vu) + \mint (\vu)) / 2$ to each $\vu$, \fs{see Appendix~\ref{app:analytical-2obj-proof} for the proof}. When $K>2$, we can still replace each $\vu$ by
\begin{equation}
\text{Shift-All}(\vu) = 
\begin{dcases}
\vu & \text{if } \mino (\vu) + \mint (\vu) \ge 0 \; ,\\
\vu - \frac{\mino (\vu) + \mint (\vu)}{2} & \text{otherwise.}
\end{dcases}
\end{equation}
This is no longer the optimal solution to the problem of \cref{eq:proj}\fs{, but it is a feasible one}, and we show in Appendix~\ref{app:analytical-qpmod-proof} that it is the solution to the problem of \cref{eq:proj} under the additional constraint that the pairwise differences between the different SDF values be preserved, a reasonable constraint to avoid disrupting the ordering of the original fields during the projection. It has the added benefit to be much simpler and faster to compute than solving the QP problem. 

\nt{This mapping is continuous, differentiable almost everywhere, and the non-differentiable points are typically not on the object surfaces.}
In other words, it does not create issues either for meshing or training, unlike the non-differentiability of Unsigned Distance Fields on the surface that is well known to create problems~\cite{Guillard22b,Zhang23b,Stella25}. Moreover, being an optimal solution to problem \ref{eq:proj} does not necessarily mean that the QP solution is better in practice, as optimality is defined in terms of the $L_2$ distance to the original vector $\vu$, which is an arbitrary metric. In practice, we show in the results section that this heuristic performs equally well as the QP solution, and even better in some cases, \pf{at a much lower computational cost}. 

\subsection{Meshing}

If explicit representations are required by downstream applications, our MDF needs to be turned into multi-object 3D meshes that preserve their non-intersecting nature. Doing this for standard SDFs usually relies on isosurface extraction algorithms such as Marching Cubes~\citep{Lorensen87,Lewiner03}, Marching Tetrahedra~\citep{Chan98,Treece99}, or Dual Contouring~\citep{Ju02}. When dealing with multiple SDFs, these algorithms are often applied independently to each SDF, and the resulting meshes are interpreted as lying in the same physical space, with potential intersections. A partial exception to this is Dual Contouring, for which a multi-object-aware version exists~\citep{Rashid23}, but we do not know of any readily accessible implementation of this algorithm. 

Fortunately, this is not an issue thanks to \cref{th:linearity} that states that the linear interpolation of an MDF is still an MDF. This guarantees that, when the MDF constraint is satisfied at the grid points, it is possible to extract intersection-free meshes using linearly interpolated meshing algorithms, including their adaptive variations~\citep{Ren24}, on each SDF independently, \fs{as we prove in \cref{app:crossing-proof}}. In other words, we can use \fs{some of the} existing meshing methods as they are, \fs{on each MDF component independently}.

\section{Experiments}
\label{sec:exps}
\begin{table}[ht]
  \centering
  \resizebox{\textwidth}{!}{%
  \begin{tabular}{cc|ccccc}
    Data & Model & MCD $(\times 10^{-4})$ $\downarrow$ & MNC $\uparrow$ & F1 (0.01) $\uparrow$ & IoU $\uparrow$ & IV $\downarrow$ \\
    \midrule
    \multirow{5}{*}{\rotatebox[origin=c]{90}{\shortstack{TS-Lung}}} & Vanilla MedTet & $5.462 \pm 0.118$ & $96.729 \pm 0.059$ & $99.601 \pm 0.008$ & $93.737 \pm 0.038$ & $(2.755 \pm 0.767) \times 10^{-4}$ \\
    \cmidrule(lr){2-7}
     & \acro{} (QP) - PP & $5.800 \pm 0.314$ & $96.766 \pm 0.004$ & $99.575 \pm 0.035$ & $94.011 \pm 0.850$ & $(1.190 \pm 0.407) \times 10^{-6}$ \\
     & \acro{} (QP) - Train & $5.357 \pm 0.309$ & $96.791 \pm 0.019$ & $99.594 \pm 0.007$ & $94.255 \pm 0.033$ & $(1.907 \pm 0.444) \times 10^{-7}$ \\
    \cmidrule(lr){2-7}
     & \acro{} (Shift-All) - PP & $5.600 \pm 0.289$ & $96.708 \pm 0.023$ & $99.579 \pm 0.036$ & $94.000 \pm 0.207$ & $(1.132 \pm 0.245) \times 10^{-6}$ \\
     & \acro{} (Shift-All) - Train & $5.394 \pm 0.056$ & $96.788 \pm 0.055$ & $99.620 \pm 0.010$ & $93.863 \pm 0.397$ & $(1.391 \pm 0.934) \times 10^{-7}$ \\
    \specialrule{\heavyrulewidth}{\aboverulesep}{\belowrulesep}
    \multirow{5}{*}{\rotatebox[origin=c]{90}{\shortstack{MM-WHS \\ (Heart)}}} & Vanilla MedTet & $23.863 \pm 1.330$ & $94.333 \pm 0.157$ & $97.470 \pm 0.142$ & $83.580 \pm 0.473$ & $(3.445 \pm 1.390) \times 10^{-4}$ \\
    \cmidrule(lr){2-7}
     & \acro{} (QP) - PP & $29.420 \pm 4.430$ & $94.293 \pm 0.062$ & $97.367 \pm 0.073$ & $83.591 \pm 0.081$ & $(8.446 \pm 3.149) \times 10^{-11}$ \\
     & \acro{} (QP) - Train & $24.796 \pm 0.670$ & $94.386 \pm 0.106$ & $97.465 \pm 0.308$ & $83.544 \pm 0.481$ & $(2.674 \pm 3.486) \times 10^{-9}$ \\
    \cmidrule(lr){2-7}
     & \acro{} (Shift-All) - PP & $31.099 \pm 7.072$ & $94.316 \pm 0.151$ & $97.297 \pm 0.229$ & $83.694 \pm 0.375$ & $(9.888 \pm 1.793) \times 10^{-12}$ \\
     & \acro{} (Shift-All) - Train & $23.094 \pm 0.331$ & $94.407 \pm 0.089$ & $97.556 \pm 0.173$ & $83.901 \pm 0.039$ & $(4.635 \pm 5.606) \times 10^{-11}$ \\
  \end{tabular}%
  }
  \caption{Experimental results on MedTet. Means are across all organ components and samples, averaged over runs with 3 different random seeds. The IV of our variants is limited by the 16-bit storage of the predicted SDFs (\cref{sec:residual_iv}).}
  \label{tab:medtet}
\end{table}

\begin{table}[ht]
\centering
\resizebox{\textwidth}{!}{%
  \begin{tabular}{c|cccc}
     & MCD ($\times 10^{-2}$) $\downarrow$ & MNC $\uparrow$ & F1 (0.01) $\uparrow$ & IV $\downarrow$ \\
    \hline
    Vanilla ObjectSDF++ & $10.677 \pm 5.817$ & $77.953 \pm 0.331$ & $66.946 \pm 0.097$ & $(1.391 \pm 0.124) \times 10^{-3}$ \\
    \hline
    \acro{} (QP) - PP & $4.145 \pm 1.234$ & $78.023 \pm 0.221$ & $68.275 \pm 0.153$ & $(1.834 \pm 0.584) \times 10^{-7}$ \\
    \acro{} (QP) - Train & - & - & - & - \\
    \hline
    \acro{} (Shift-All) - PP & $3.345 \pm 0.315$ & $78.058 \pm 0.207$ & $68.325 \pm 0.131$ & $(1.264 \pm 0.237) \times 10^{-7}$ \\
    \acro{} (Shift-All) - Train & $6.869 \pm 0.633$ & $76.411 \pm 0.223$ & $66.646 \pm 0.666$ & $(2.278 \pm 0.199) \times 10^{-7}$ \\
    \acro{} (Shift-All Skip-Conn.) - Train & $3.942 \pm 1.000$ & $78.430 \pm 0.528$ & $68.623 \pm 0.585$ & $(1.492 \pm 1.925) \times 10^{-6}$ \\
  \end{tabular}%
}
\caption{ObjectSDF++~\cite{Wu23b} trained for object-compositional 3D scene reconstruction on 8 Replica~\cite{Straub19} scenes. Mean and standard deviation across 3 runs. QP could not be run for the training setting due to the time complexity of the optimization for a large number of objects.}
\label{tab:objectsdfplus_scenes_all}%
\end{table}
\newcommand{\incr}[1]{\includegraphics[width=0.185\textwidth, keepaspectratio, valign=c]{#1}}

\begin{figure}[h]
\centering
\renewcommand{\arraystretch}{1.0}
\makeatletter
\renewcommand\small{\@setfontsize\small{6}{8.5}}
\makeatother
\setlength{\tabcolsep}{1pt}

\begin{tabular}{c c c c c c}
    & \small Vanilla
    & \small QP (PP)
    & \small Shift-All (PP)
    & \small Shift-All (Train)
    & \small Shift-All Skip-Conn. (Train) \\

    \raisebox{-0.5\height}{\rotatebox{90}{\small }} &
    \incr{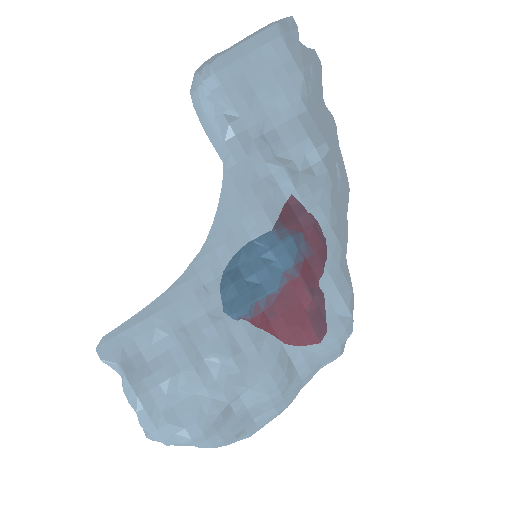} &
    \incr{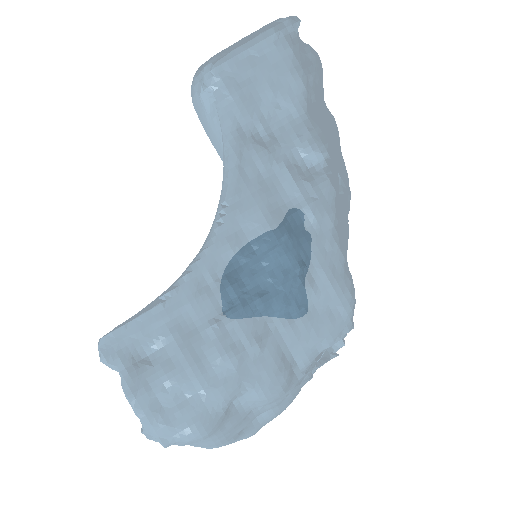} &
    \incr{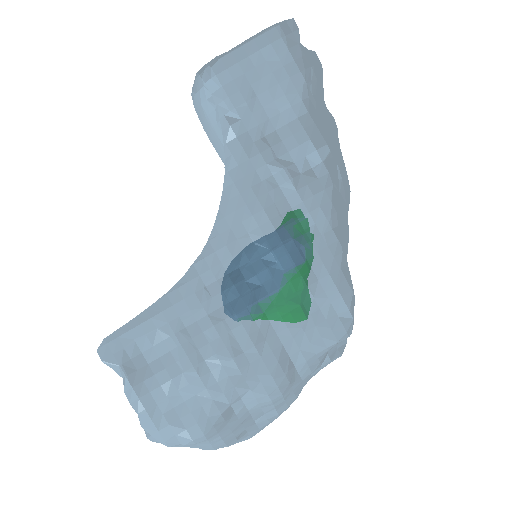} &
    \incr{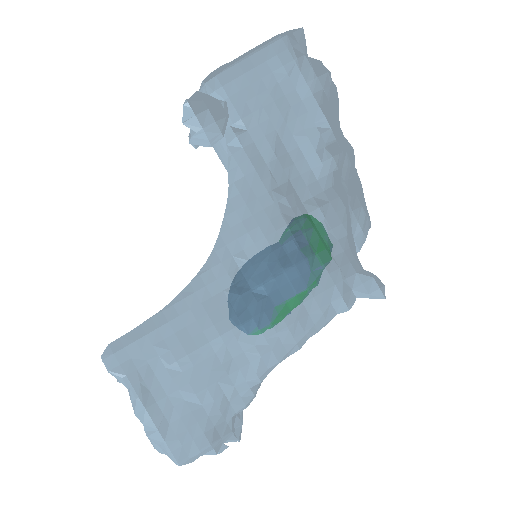} &
    \incr{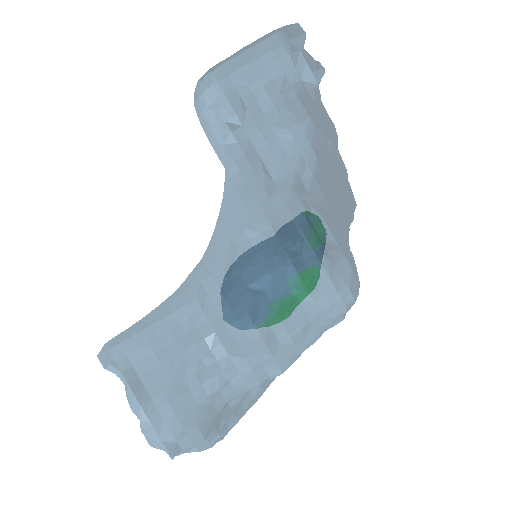} \\

    \raisebox{-0.5\height}{\rotatebox{90}{\small}} &
    \incr{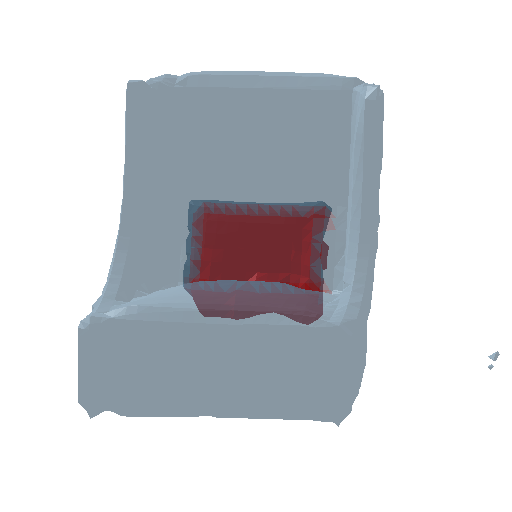} &
    \incr{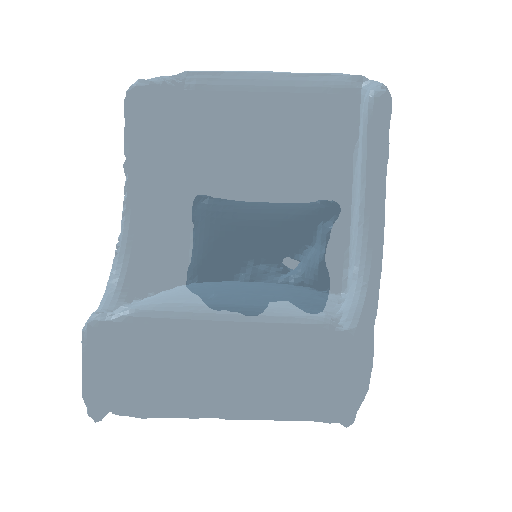} &
    \incr{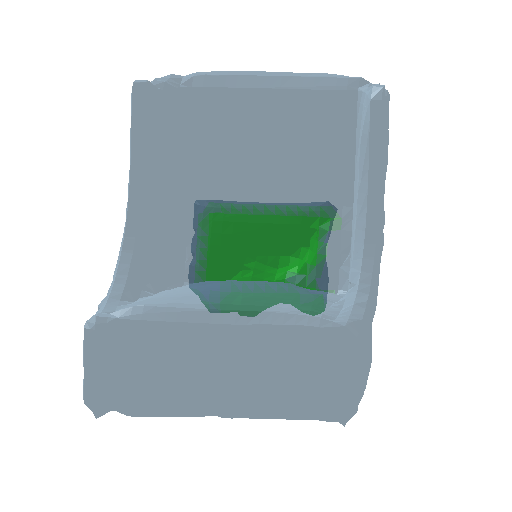} &
    \incr{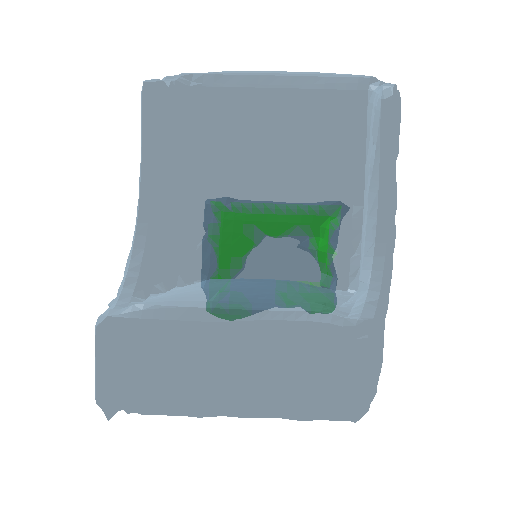} &
    \incr{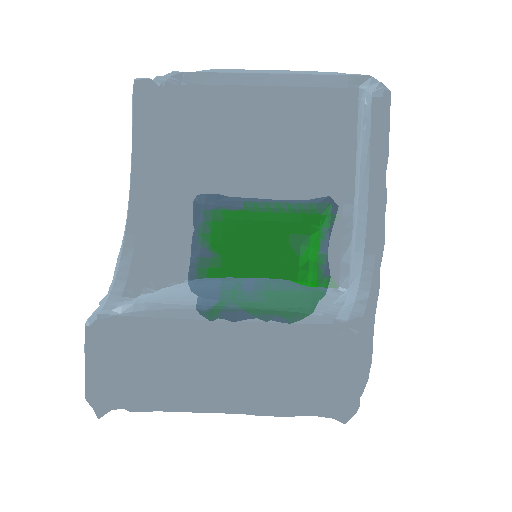} \\
\end{tabular}

\caption{Transversal (top) and frontal (bottom) views of a chair (light grey) and a pillow (dark grey) from Room 0 of the Replica dataset, reconstructed under five scenarios for ObjectSDF++~\cite{Wu23b}, illustrating inter-object intersections in red and contacts in green. The vanilla version has significant intersections, while the constrained versions do not.}
\label{fig:objectsdf_plus_render_comparisons}
\end{figure}

\begin{table}[ht]
  \centering
  \resizebox{\textwidth}{!}{%
  \begin{tabular}{cc|ccccc}
    Data & Model & MCD $(\times 10^{-4})$ $\downarrow$ & MNC $\uparrow$ & F1 (0.01) $\uparrow$ & IoU $\uparrow$& IV $\downarrow$ \\
    \midrule
    \multirow{5}{*}{\rotatebox[origin=c]{90}{\shortstack{Chairs}}} & Vanilla PartSDF & $4.331 \pm 0.690$ & $95.119 \pm 0.015$ & $96.731 \pm 0.044$ & $85.630 \pm 0.101$ & $(1.876 \pm 0.062) \times 10^{-3}$ \\
    \cmidrule(lr){2-7}
     & \acro{} (QP) - PP & $2.126 \pm 0.165$ & $94.773 \pm 0.014$ & $97.902 \pm 0.042$ & $89.151 \pm 0.107$ & $(6.157 \pm 1.394) \times 10^{-9}$ \\
     & \acro{} (QP) - Train & $2.189 \pm 0.364$ & $94.885 \pm 0.078$ & $98.047 \pm 0.031$ & $89.289 \pm 0.126$ & $(7.000 \pm 0.816) \times 10^{-9}$ \\
    \cmidrule(lr){2-7}
     & \acro{} (Shift-All) - PP & $1.973 \pm 0.166$ & $94.824 \pm 0.008$ & $98.005 \pm 0.022$ & $89.173 \pm 0.104$ & $(5.194 \pm 1.040) \times 10^{-9}$ \\
     & \acro{} (Shift-All) - Train & $2.120 \pm 0.117$ & $94.823 \pm 0.032$ & $98.039 \pm 0.017$ & $89.297 \pm 0.262$ & $(5.665 \pm 2.370) \times 10^{-9}$ \\
    \specialrule{\heavyrulewidth}{\aboverulesep}{\belowrulesep}
    \multirow{5}{*}{\rotatebox[origin=c]{90}{\shortstack{Mixers}}} & Vanilla PartSDF & $1.386 \pm 0.090$ & $88.013 \pm 0.074$ & $98.441 \pm 0.107$ & $76.667 \pm 0.168$ & $(4.317 \pm 0.221) \times 10^{-5}$ \\
    \cmidrule(lr){2-7}
     & \acro{} (QP) - PP & $1.362 \pm 0.102$ & $88.000 \pm 0.135$ & $98.441 \pm 0.107$ & $75.769 \pm 0.892$ & $(1.018 \pm 0.349) \times 10^{-7}$ \\
     & \acro{} (QP) - Train & $1.247 \pm 0.032$ & $88.147 \pm 0.090$ & $98.595 \pm 0.026$ & $76.648 \pm 0.170$ & $(1.239 \pm 0.831) \times 10^{-8}$ \\
    \cmidrule(lr){2-7}
     & \acro{} (Shift-All) - PP & $1.362 \pm 0.102$ & $88.000 \pm 0.135$ & $98.441 \pm 0.107$ & $75.774 \pm 0.887$ & $(1.020 \pm 0.349) \times 10^{-7}$ \\
     & \acro{} (Shift-All) - Train & $1.234 \pm 0.137$ & $88.175 \pm 0.045$ & $98.593 \pm 0.153$ & $76.595 \pm 0.201$ & $(1.680 \pm 0.851) \times 10^{-8}$ \\
   \end{tabular}%
  }
  \caption{Shape reconstruction metrics for PartSDF on 100 chairs and mixers from the test set. Mean and standard deviation across 3 runs.}
  \label{tab:partsdf}
\end{table}

\newcommand{\chairincrender}[1]{\includegraphics[width=0.15\textwidth,trim=280pt 80pt 260pt 70pt,clip]{#1}}
\newcommand{\chairincinter}[1]{\includegraphics[width=0.15\textwidth,trim=400pt 0pt 550pt 0pt,clip]{#1}}
\newcommand{\chairgt}[1]{%
  \begin{tabular}{@{}c@{}}
    \vphantom{\chairpair{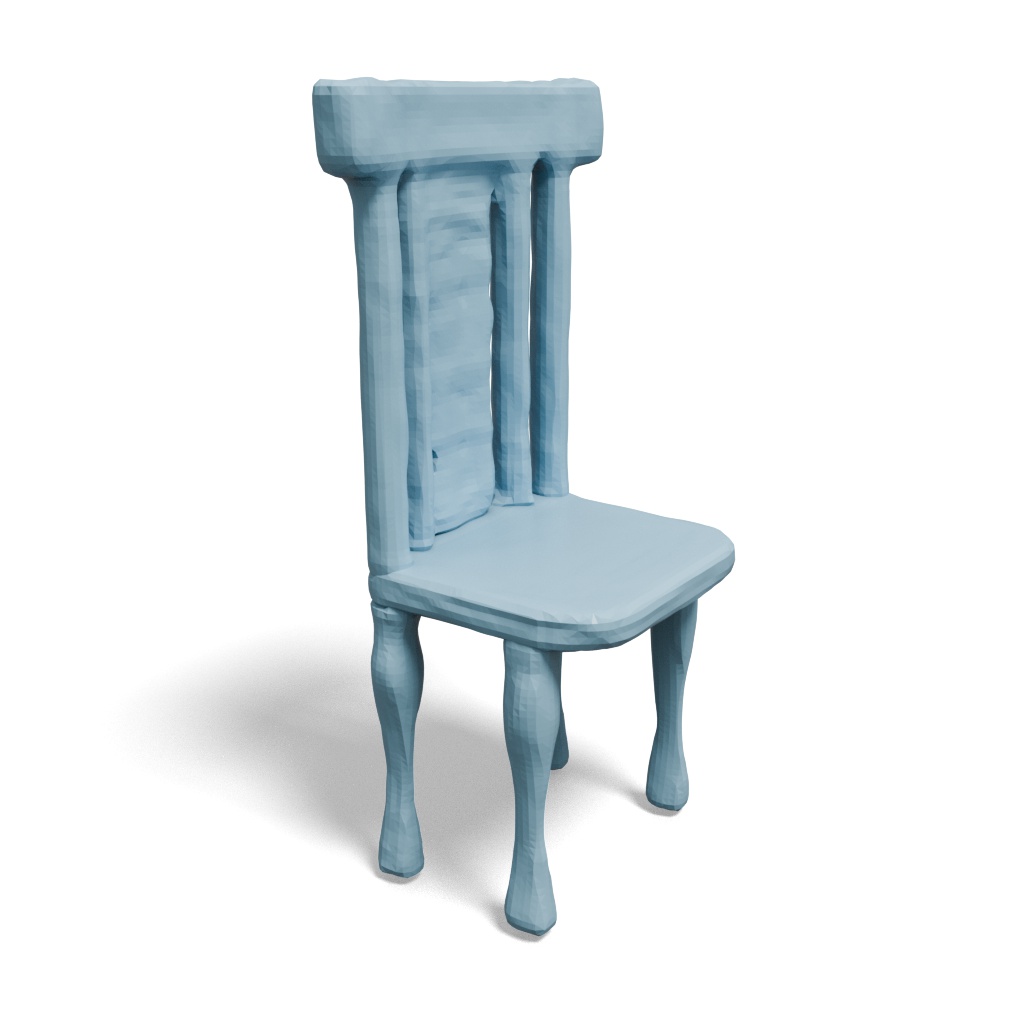}{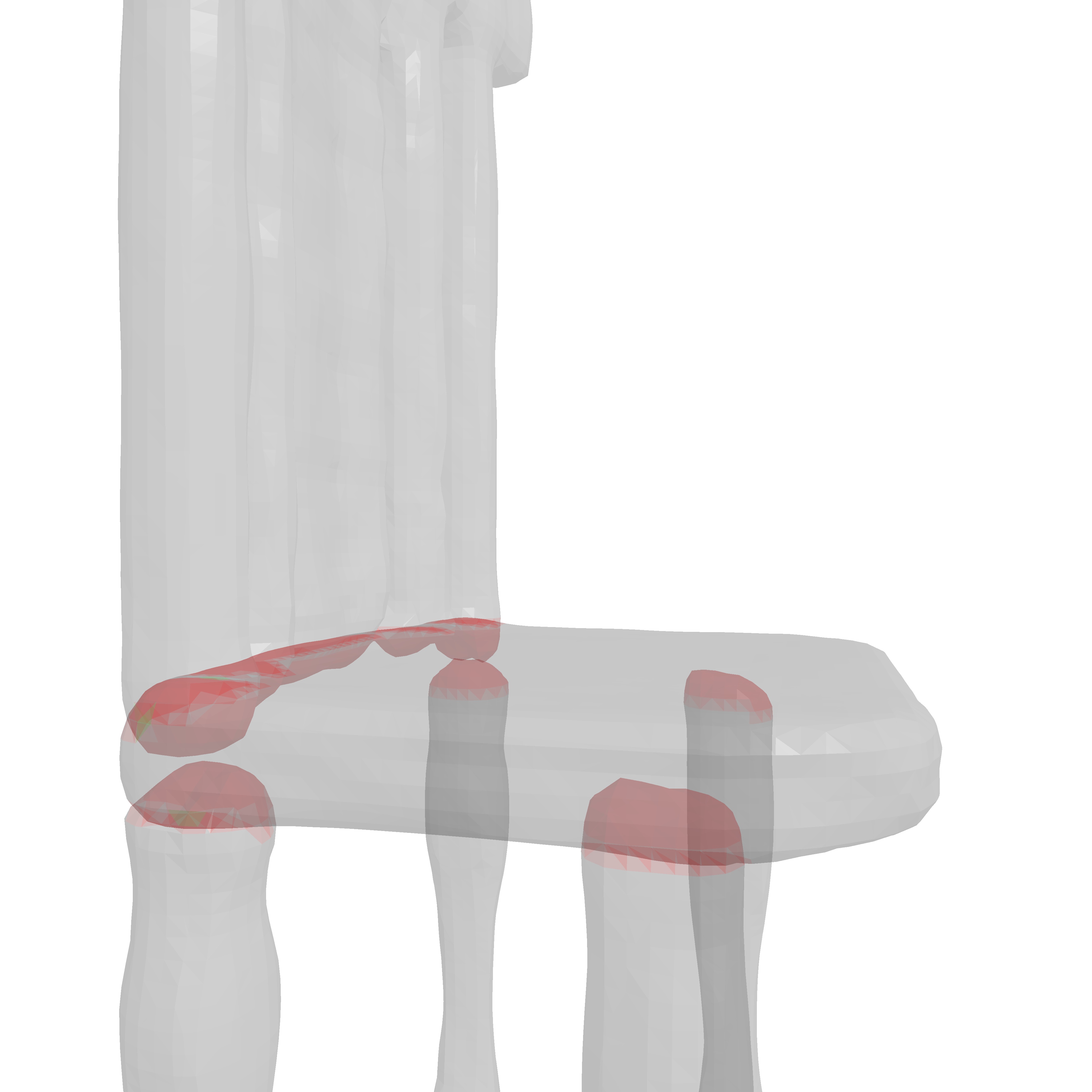}}\\[-1pt]
    \chairincrender{#1}\\[-1pt]
    \vphantom{\chairpair{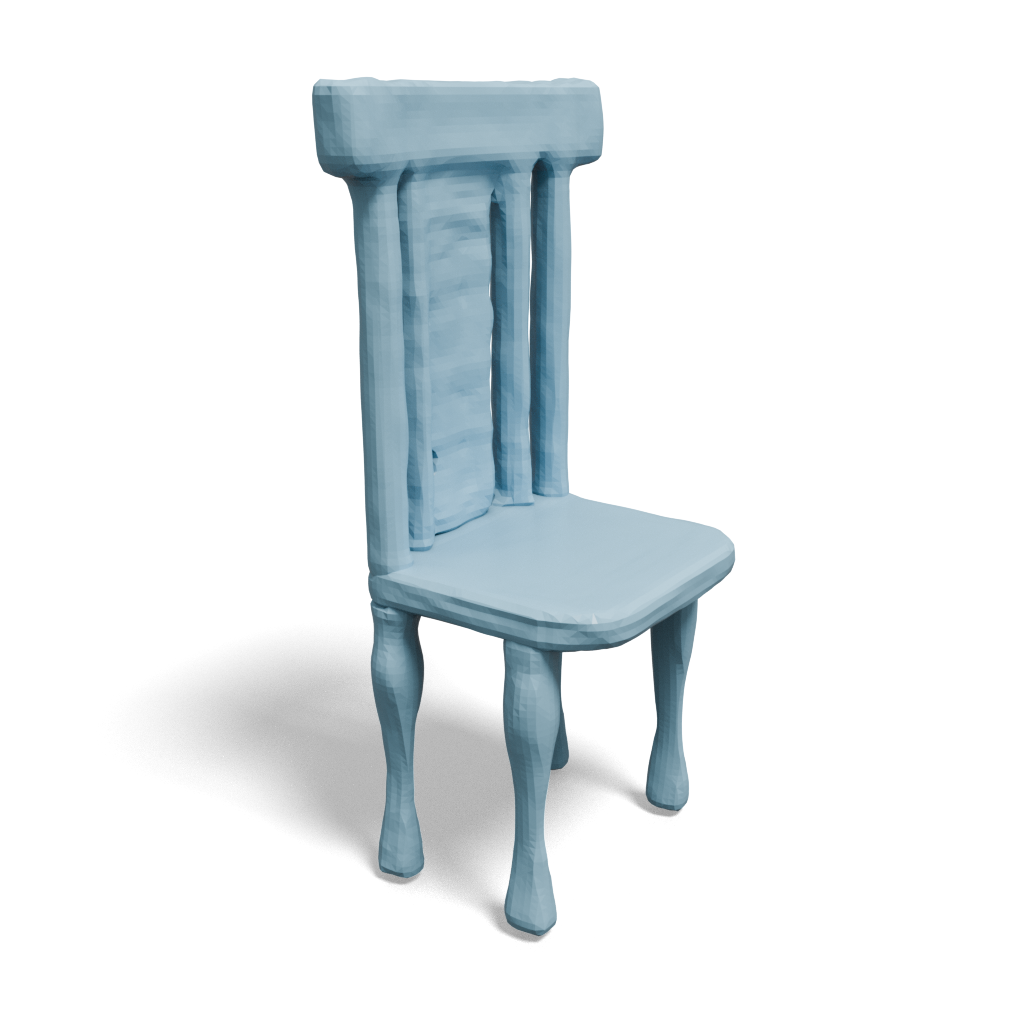}{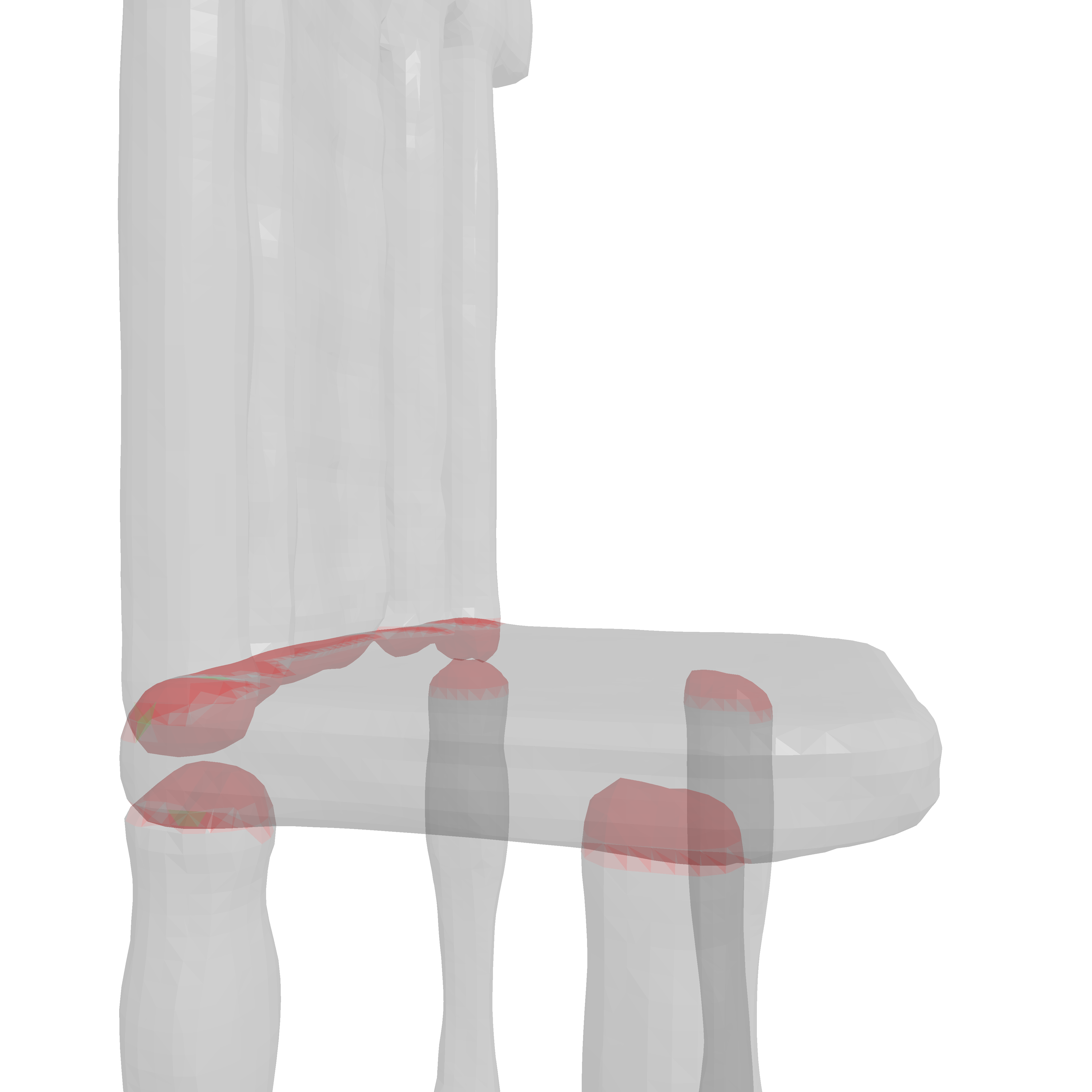}}
  \end{tabular}%
}
\newcommand{\chairpair}[2]{%
  \begin{tabular}{@{}c@{}}
    \chairincrender{#1}\\[-1pt]
    \chairincinter{#2}
  \end{tabular}%
}

\begin{figure*}[t]
  \centering
  \setlength{\tabcolsep}{0pt}
  \begin{tabular*}{\textwidth}{@{\extracolsep{\fill}}@{}c@{}c@{}c@{}c@{}c@{}c@{}c@{}}
    & \shortstack{Ground Truth} & \shortstack{Vanilla} & \shortstack{QP (PP)} & \shortstack{QP (Train)} & \shortstack{Shift-All (PP)} & \shortstack{Shift-All (Train)} \\[-145pt]
    & \chairgt{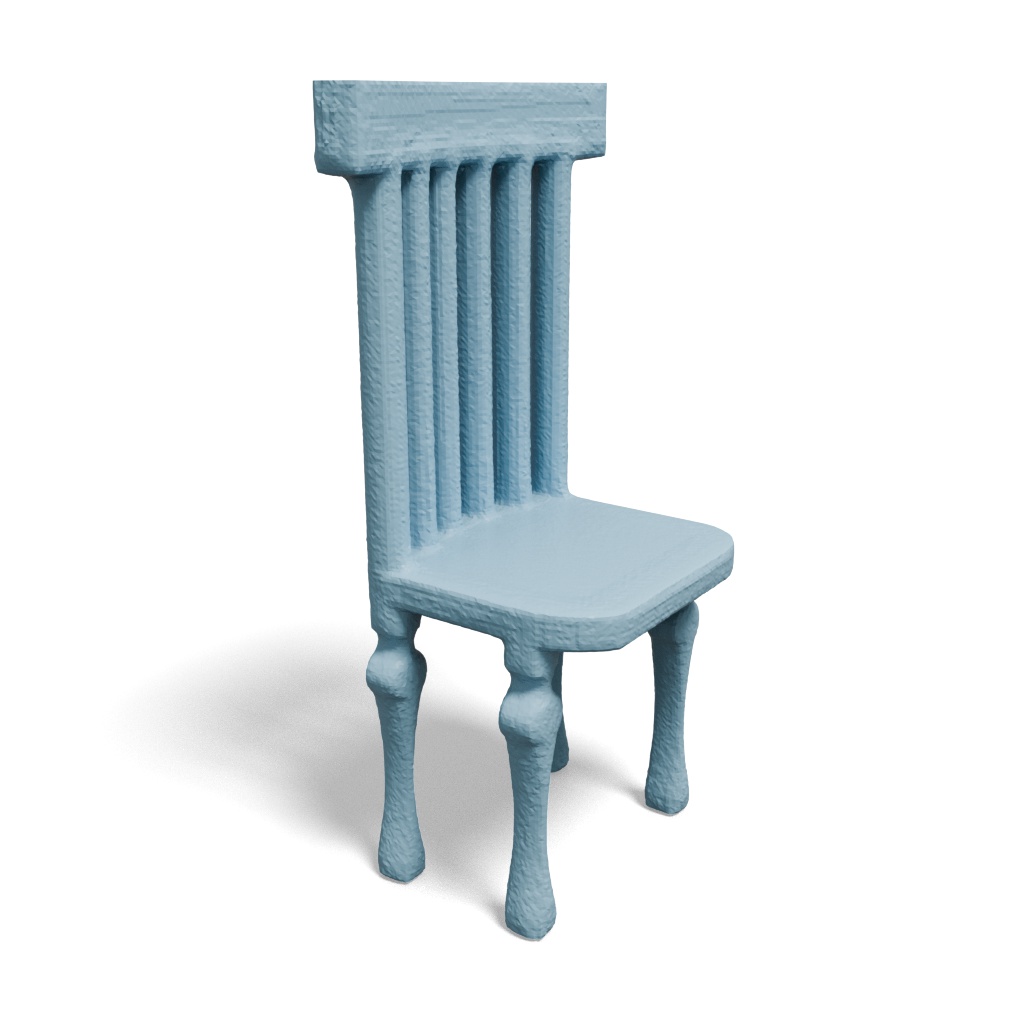} &
    \chairpair{figs/chairs/renders/vanilla.jpg}{figs/chairs/intersections/vanilla.jpg} &
    \chairpair{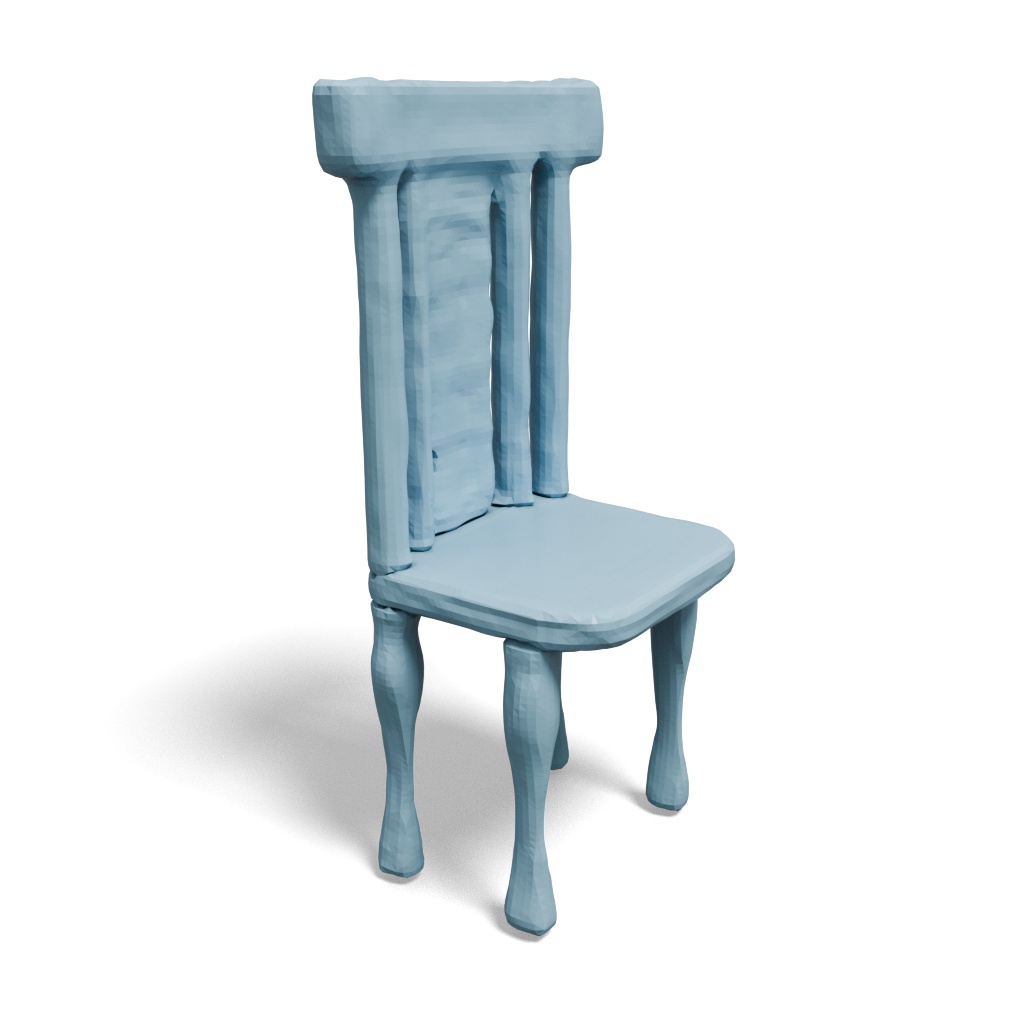}{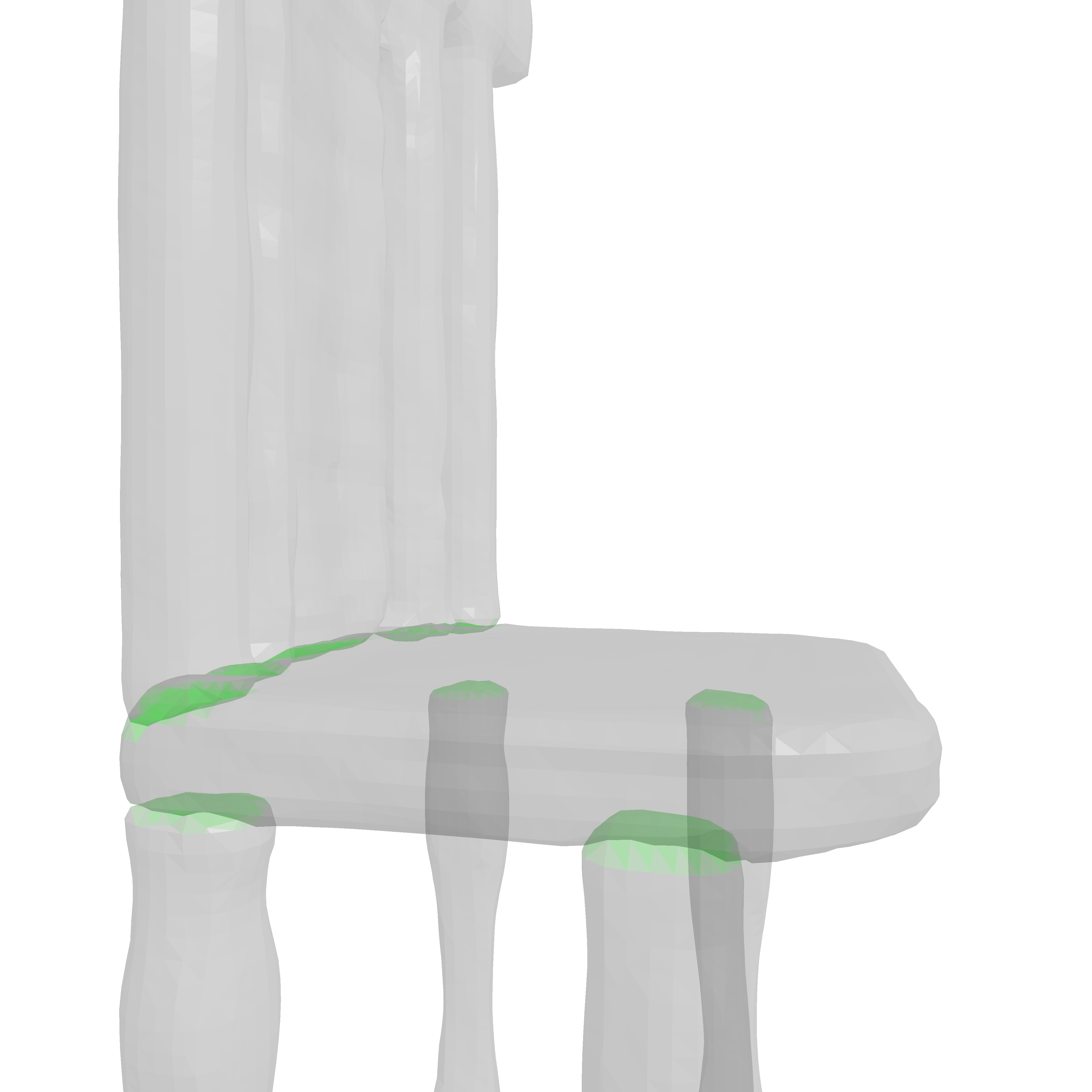} &
    \chairpair{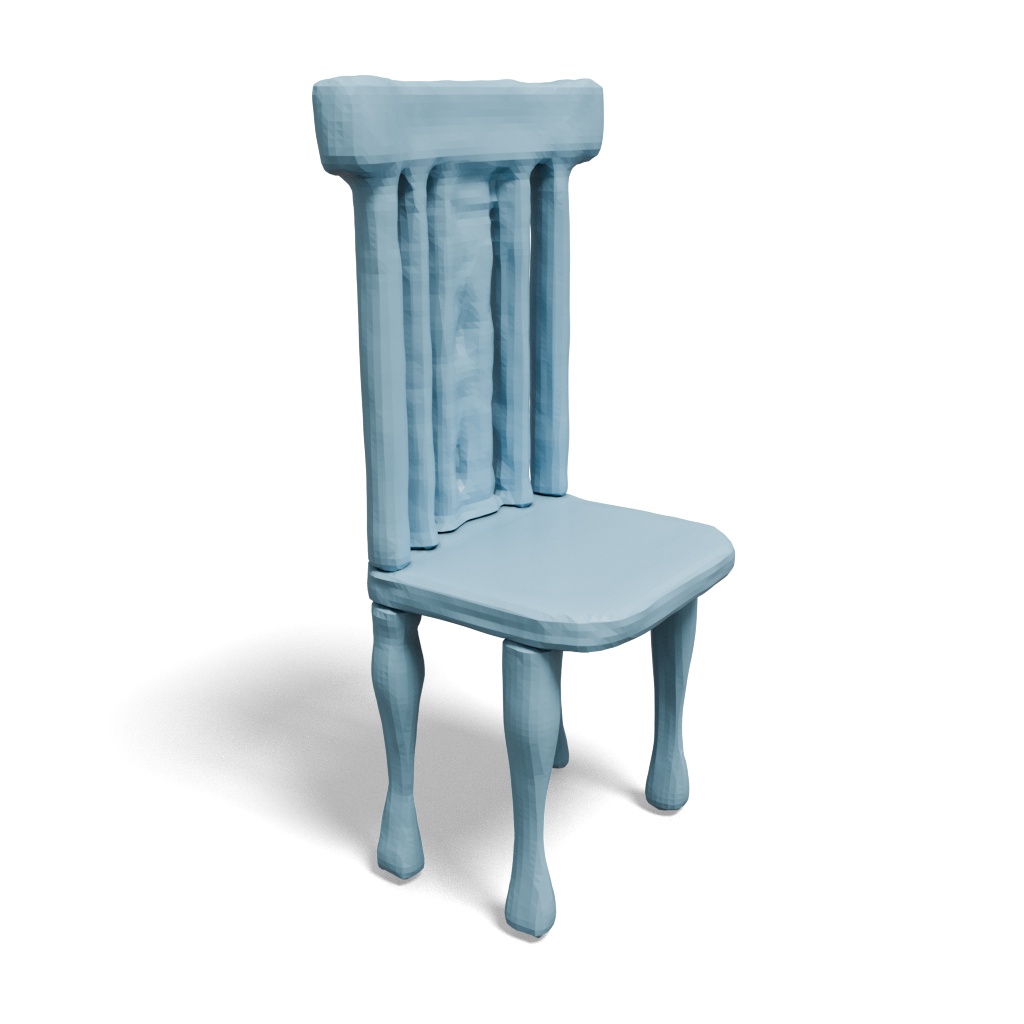}{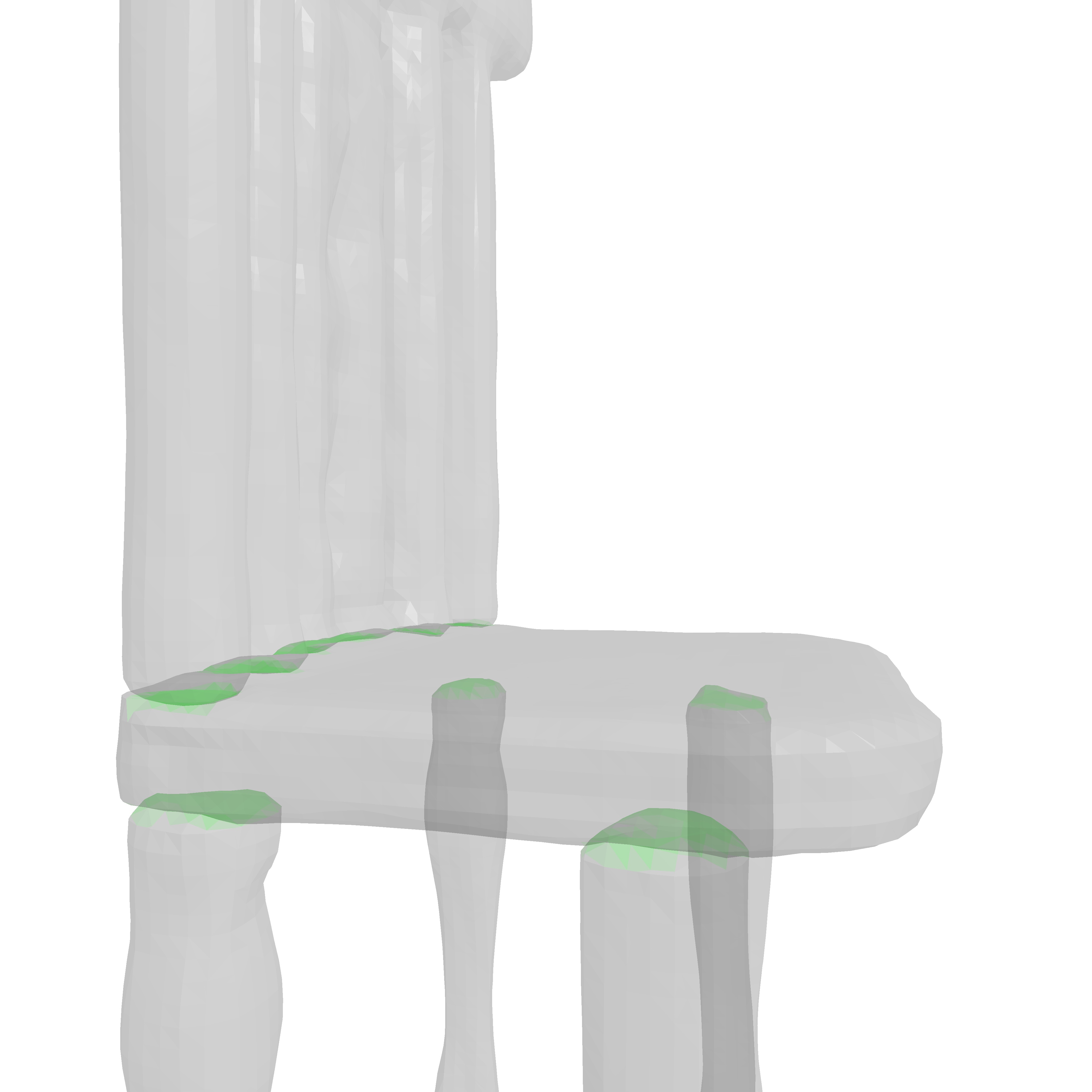} &
    \chairpair{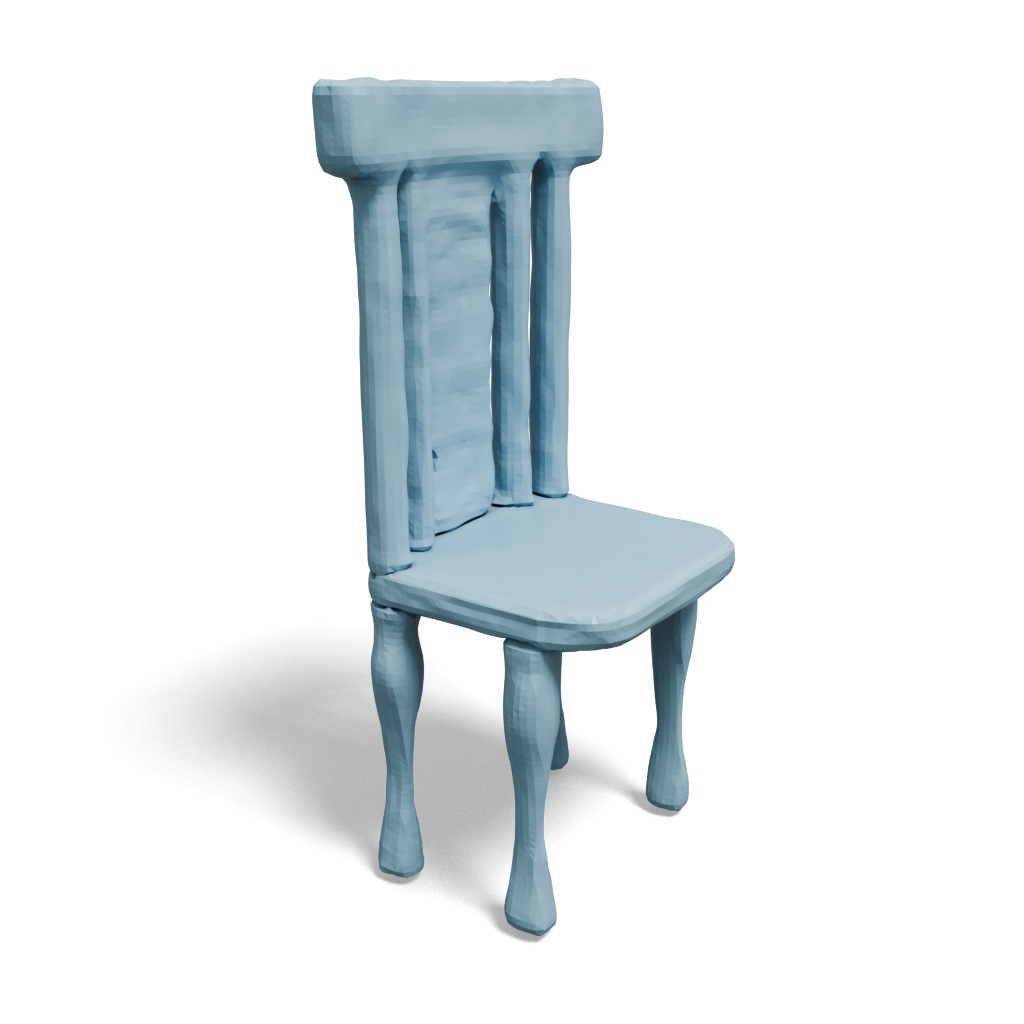}{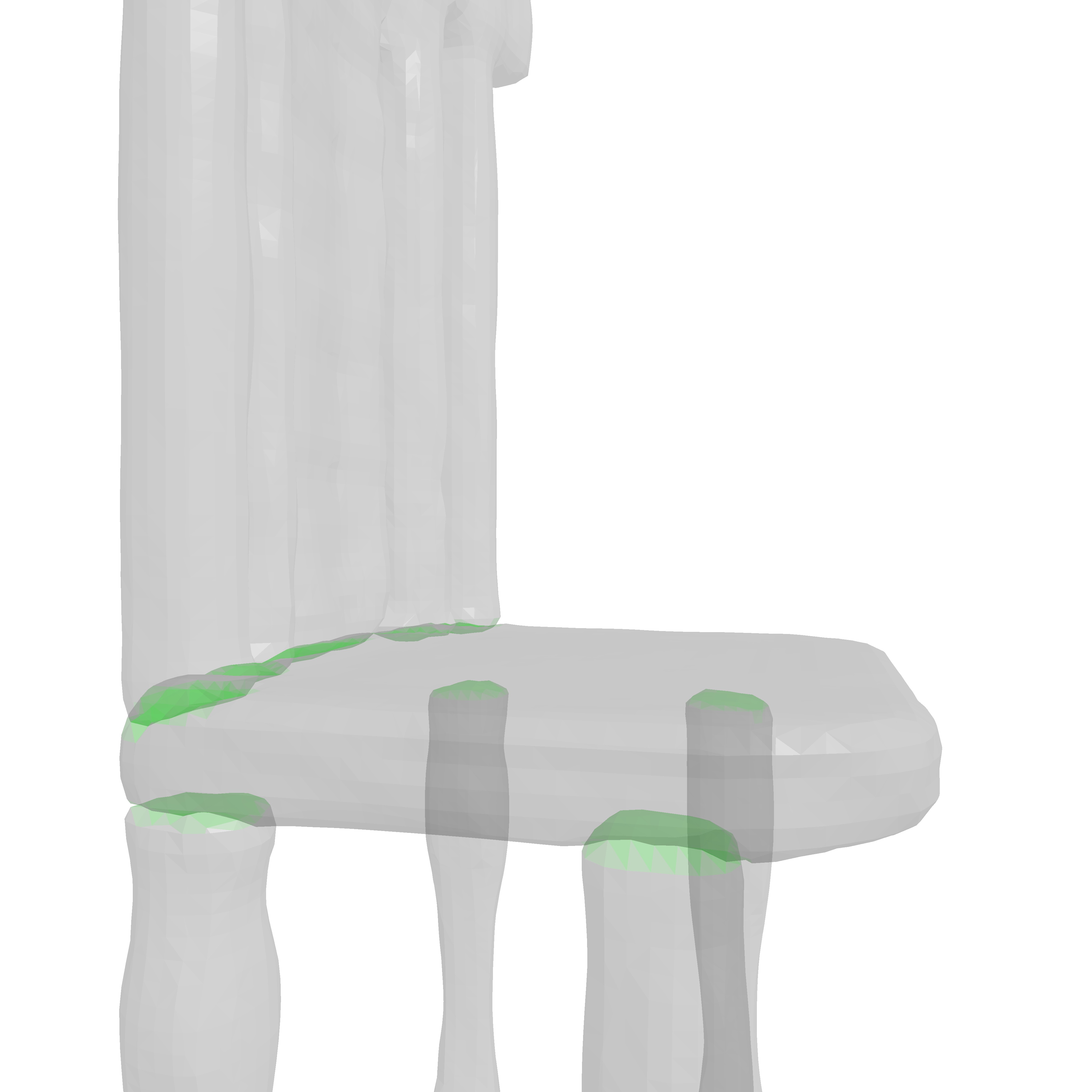} &
    \chairpair{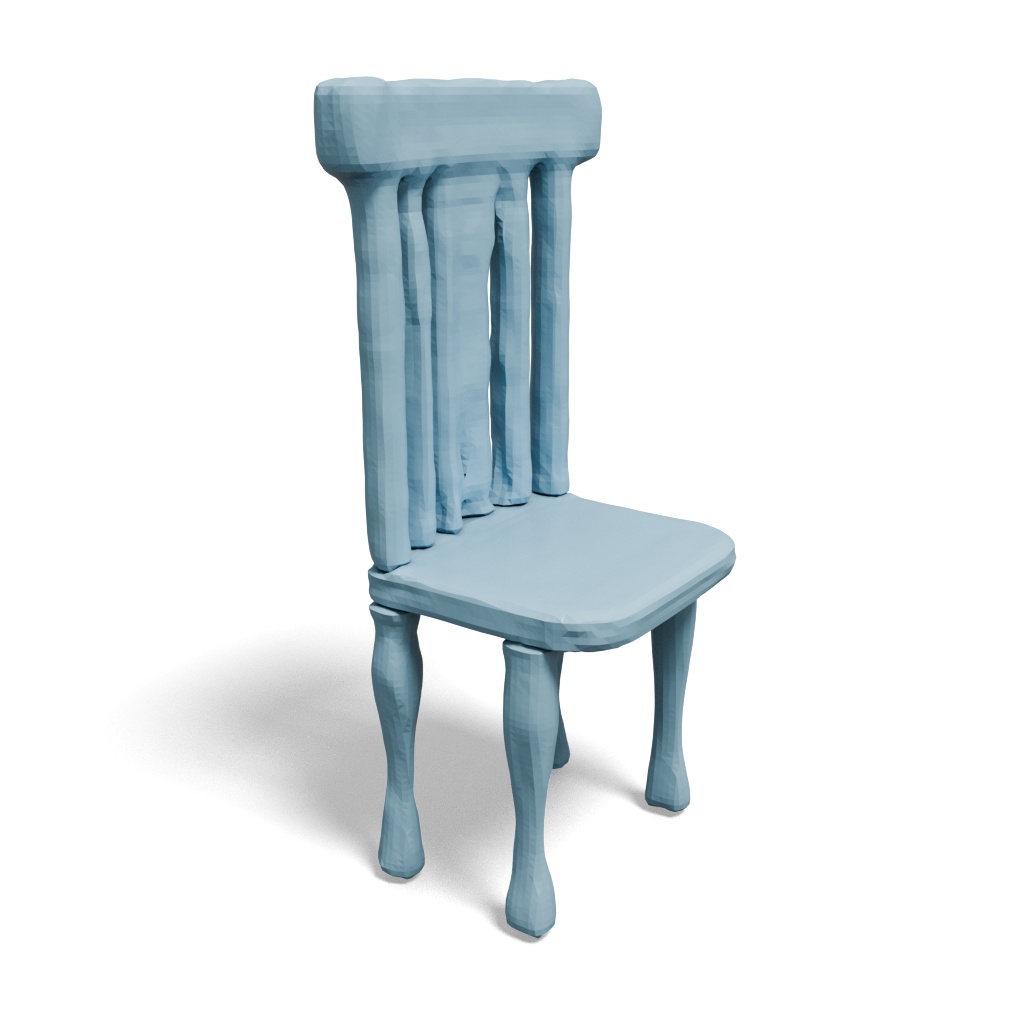}{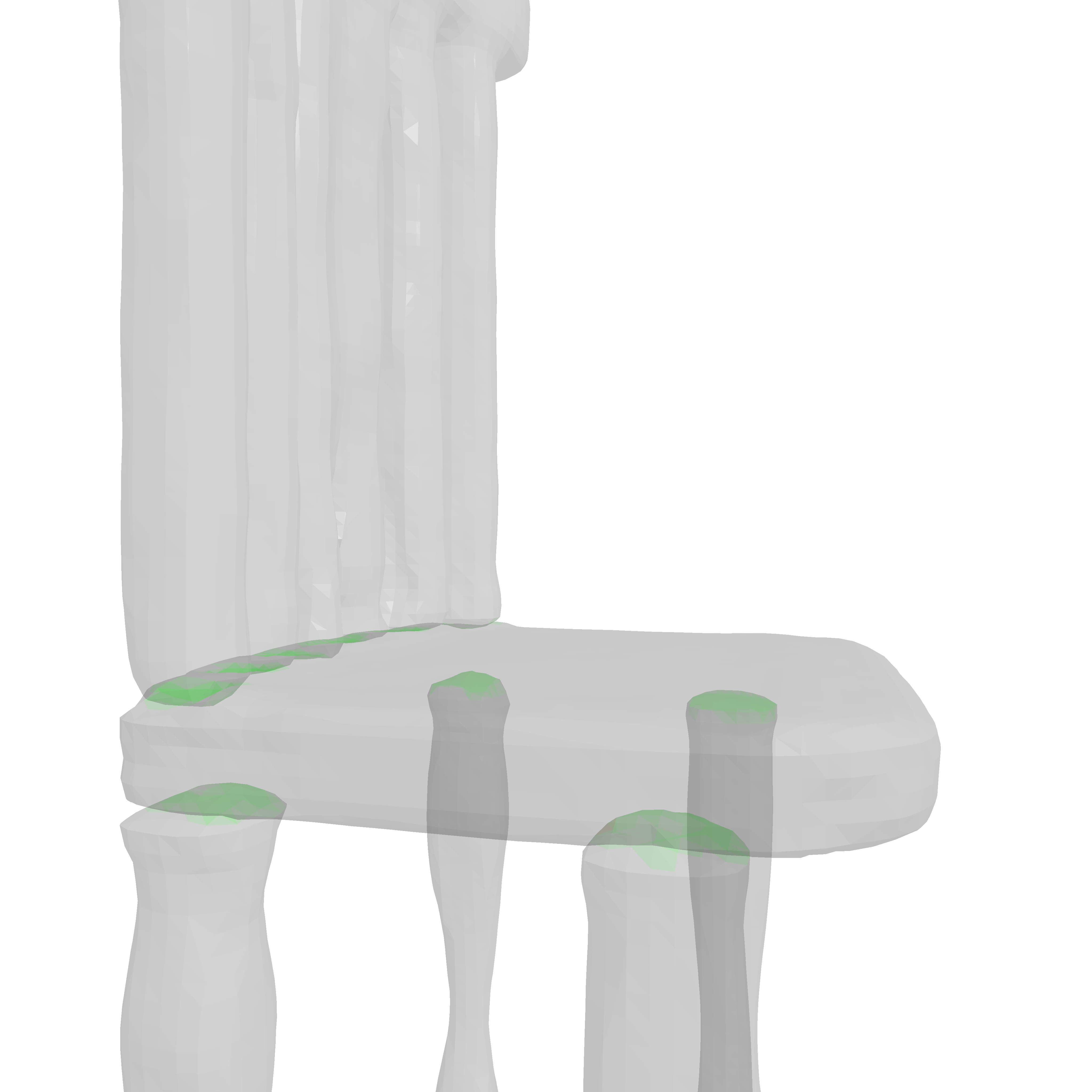} \\[-155pt]
  \end{tabular*}
  \caption{PartSDF~\cite{Talabot25} on PartNet~\cite{Mo19} chairs: the top row shows ray-traced renderings, the bottom row visualizes intersections. Intersecting surfaces are in red, while touching surfaces are in green.}
  \label{fig:chairs}
\end{figure*}

\newcommand{\inci}[1]{\includegraphics[width=0.13\textwidth, valign=c]{#1}}
\newcommand{\inciCrop}[1]{\inci{#1}}
\begin{figure*}[t]
    \centering
    \setlength{\tabcolsep}{1pt} 
    
    \begin{tabular}{c c c c c c c c}
        & Ground Truth & \multicolumn{2}{c}{Vanilla} & \multicolumn{2}{c}{\acro{} (QP) - PP} & \multicolumn{2}{c}{\acro{} (QP) - Train} \\
        & & \small Render & \small Intersections & \small Render & \small Intersections & \small Render & \small Intersections \\

        \raisebox{-0.5\height}{\rotatebox{90}{\quad \small TS-Lung}} &
        \inci{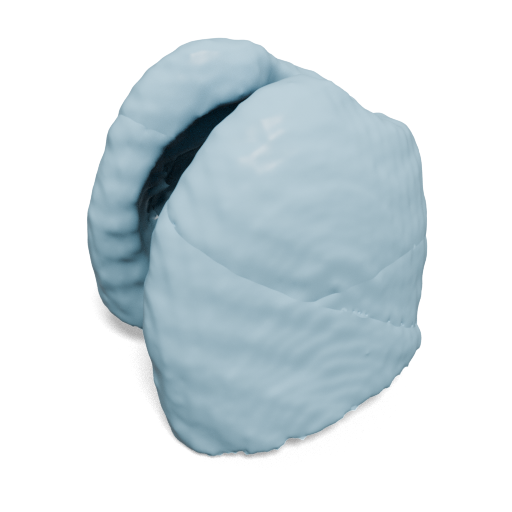} &
        \inci{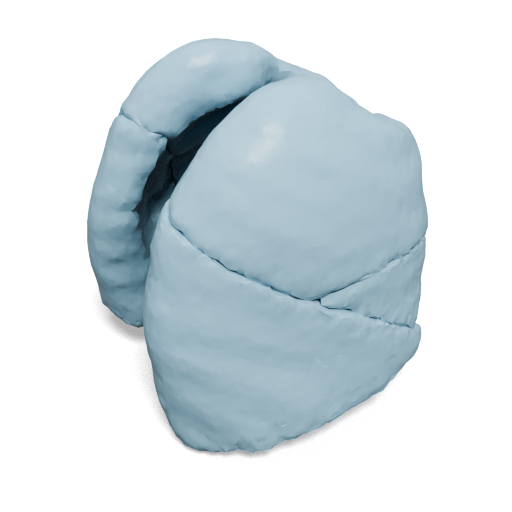} &
        \inciCrop{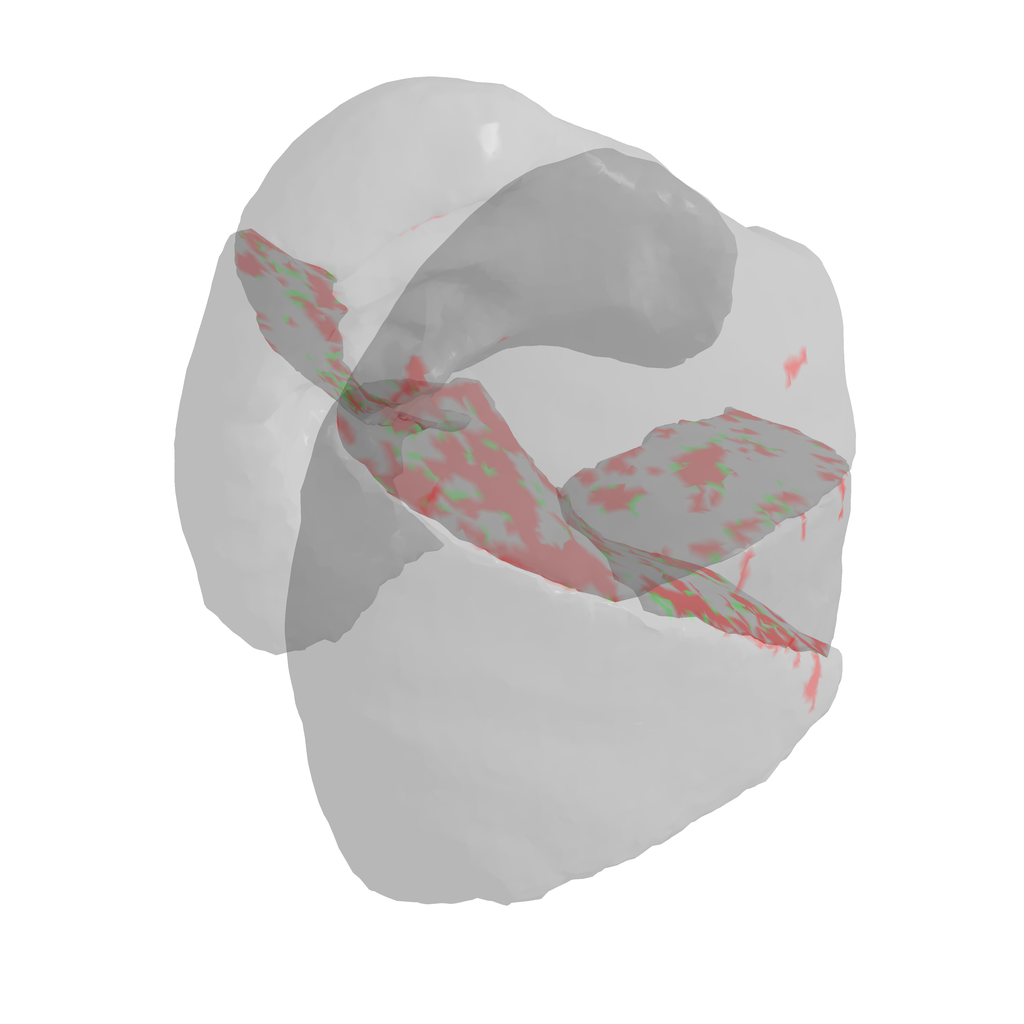} &
        \inci{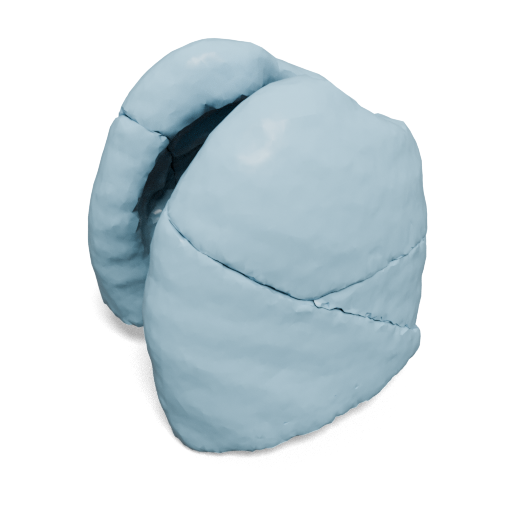} &
        \inciCrop{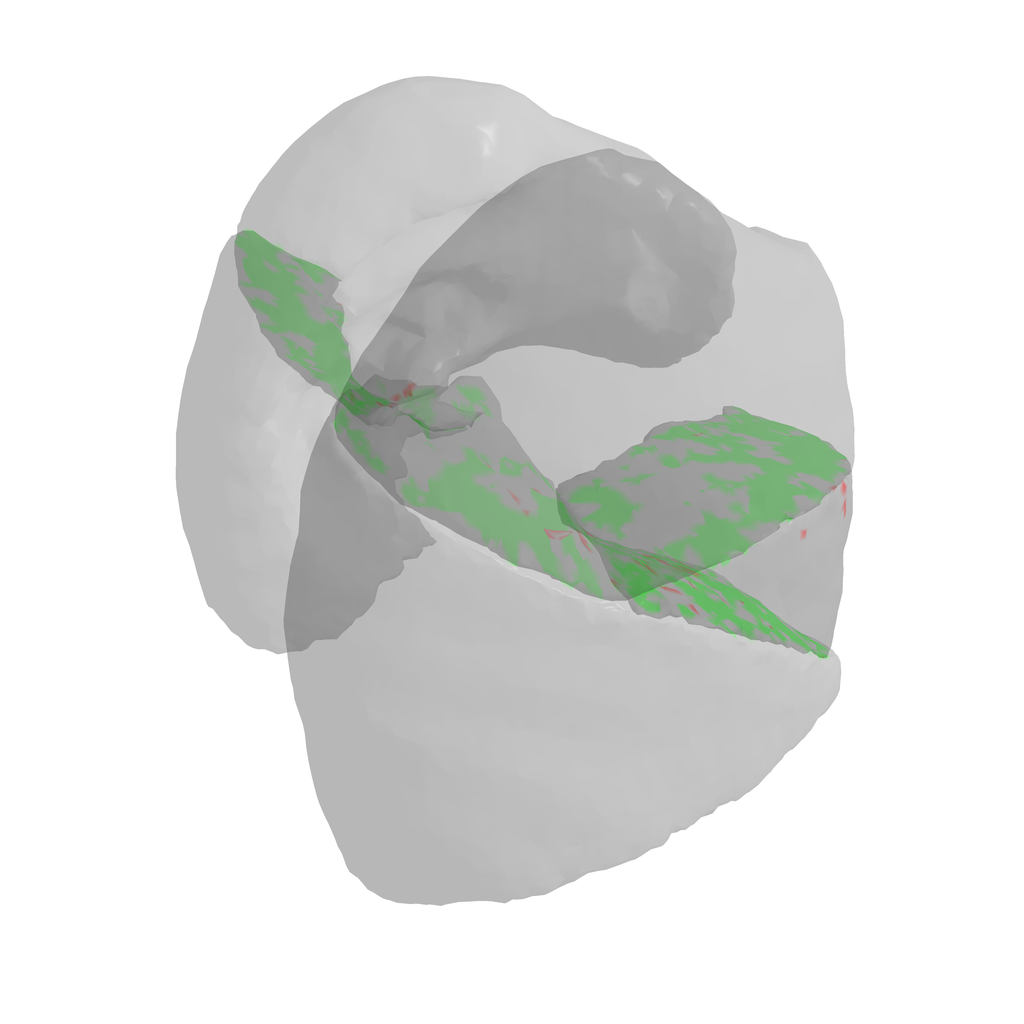} &
        \inci{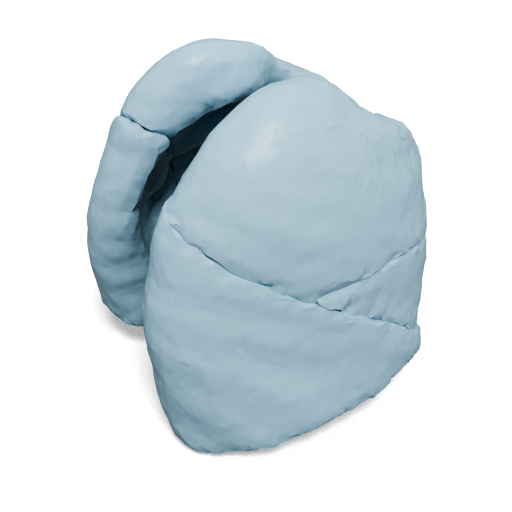} &
        \inciCrop{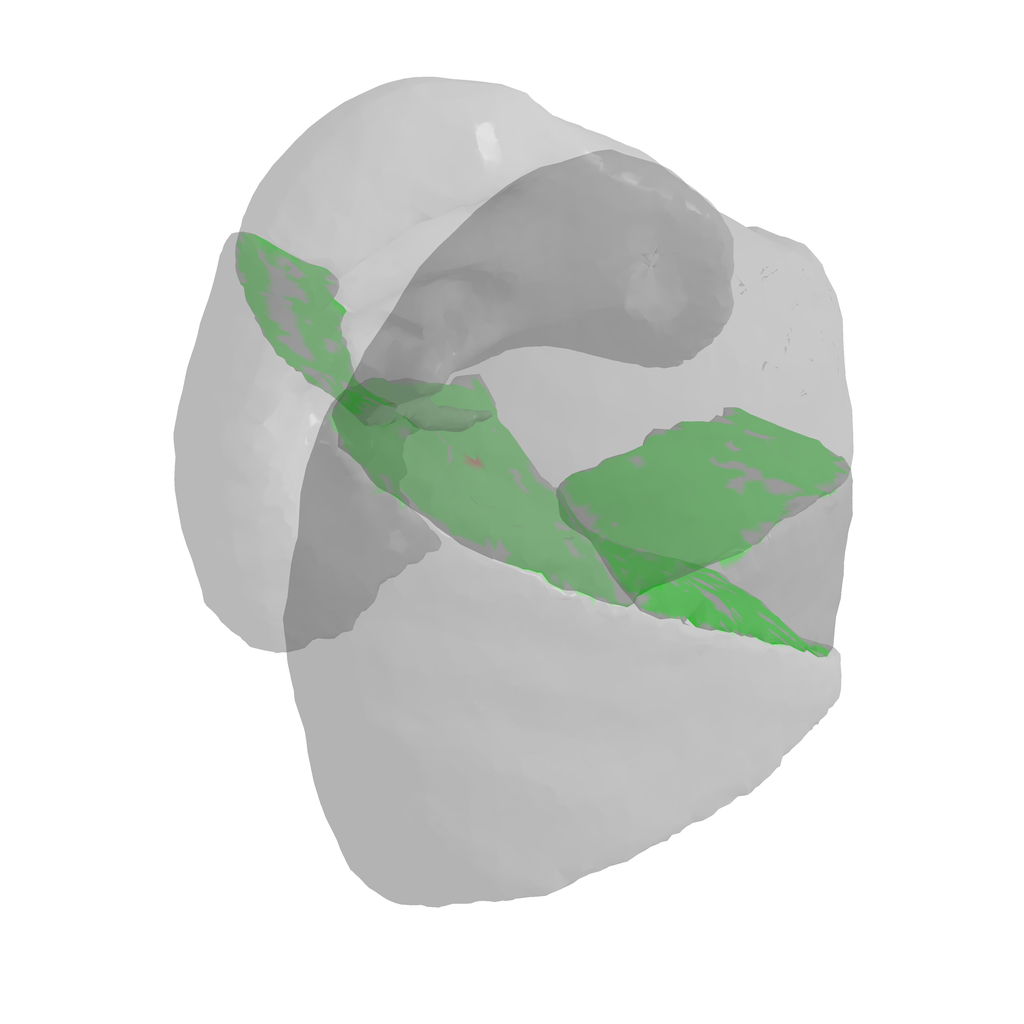} \\

        \raisebox{-0.5\height}{\rotatebox{90}{\quad \small MM-WHS}} &
        \inci{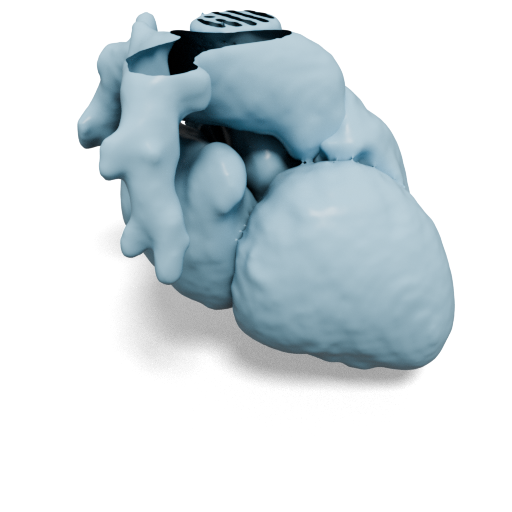} &
        \inci{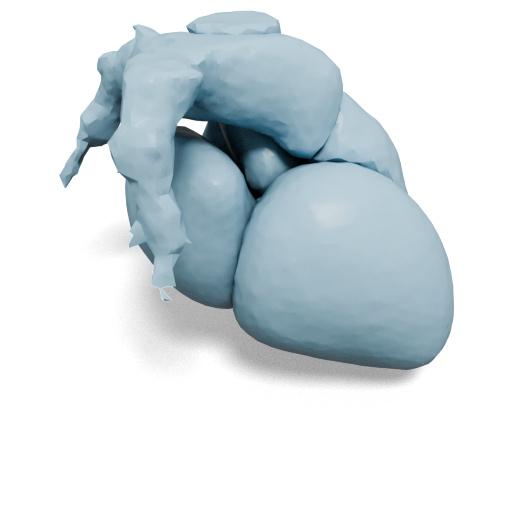} &
        \inciCrop{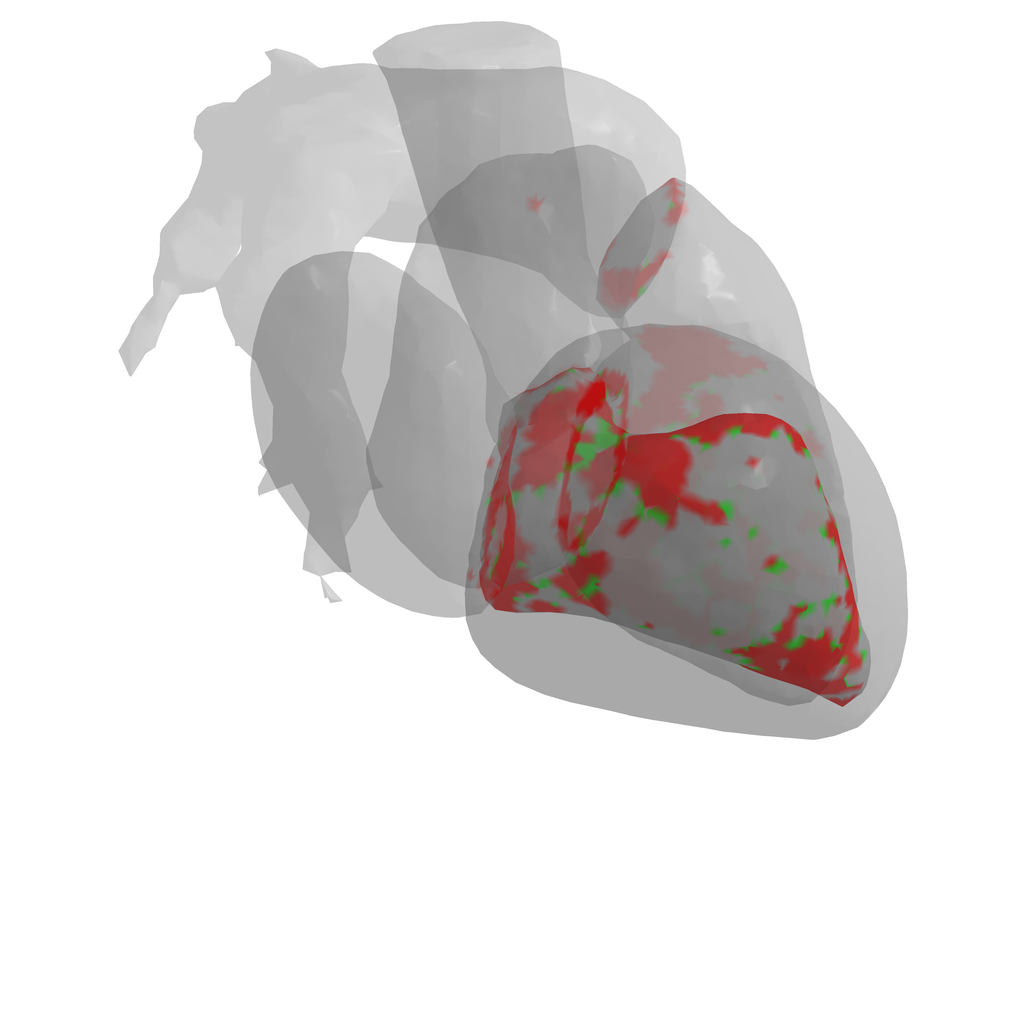} &
        \inci{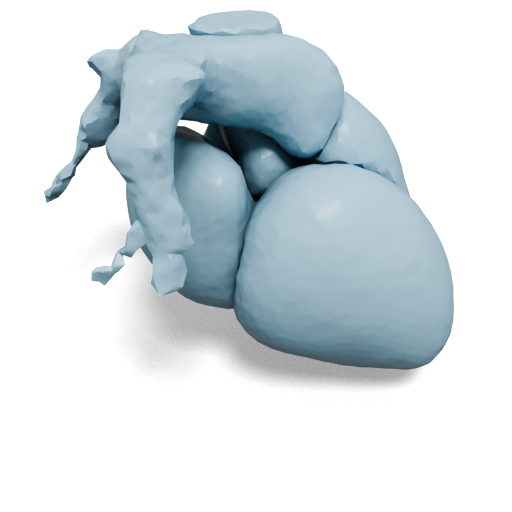} &
        \inciCrop{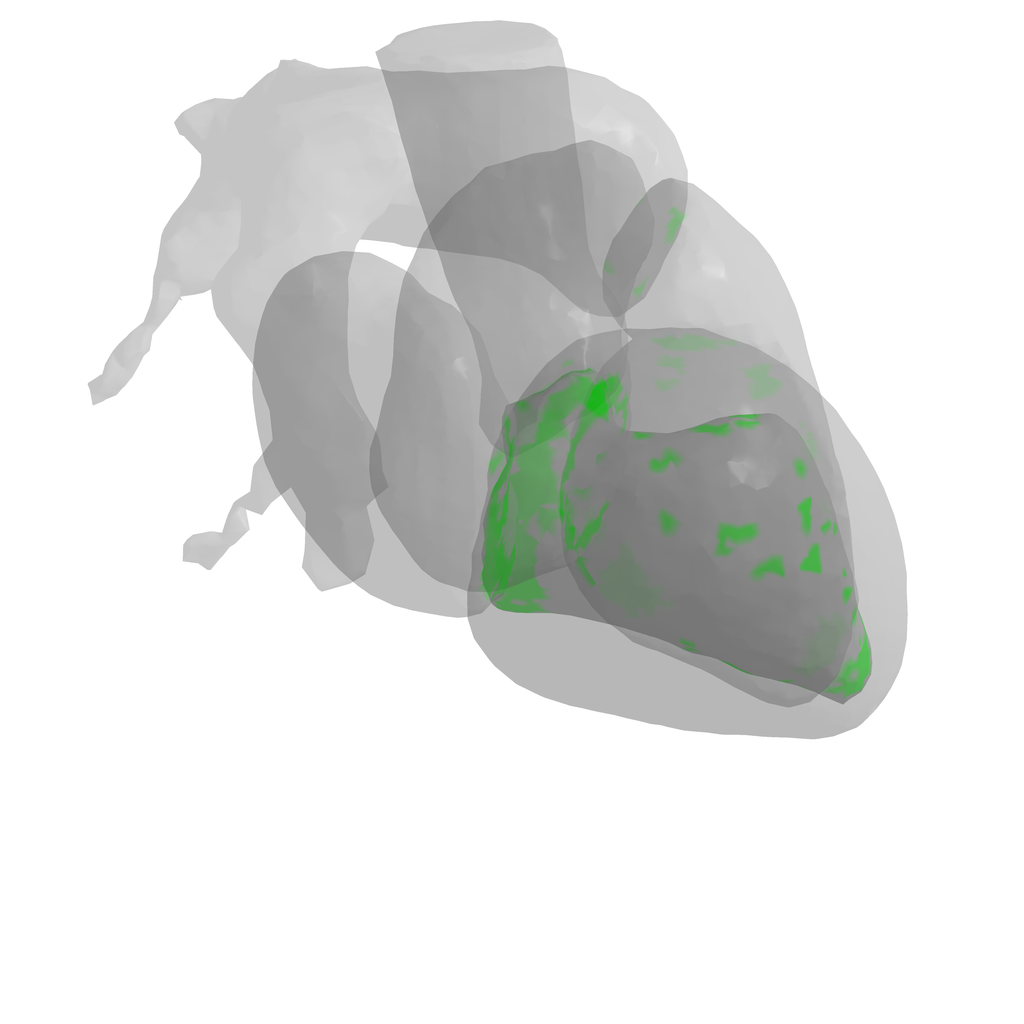} &
        \inci{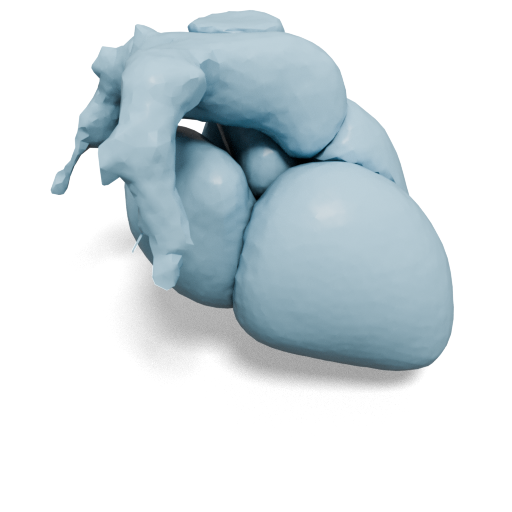} &
        \inciCrop{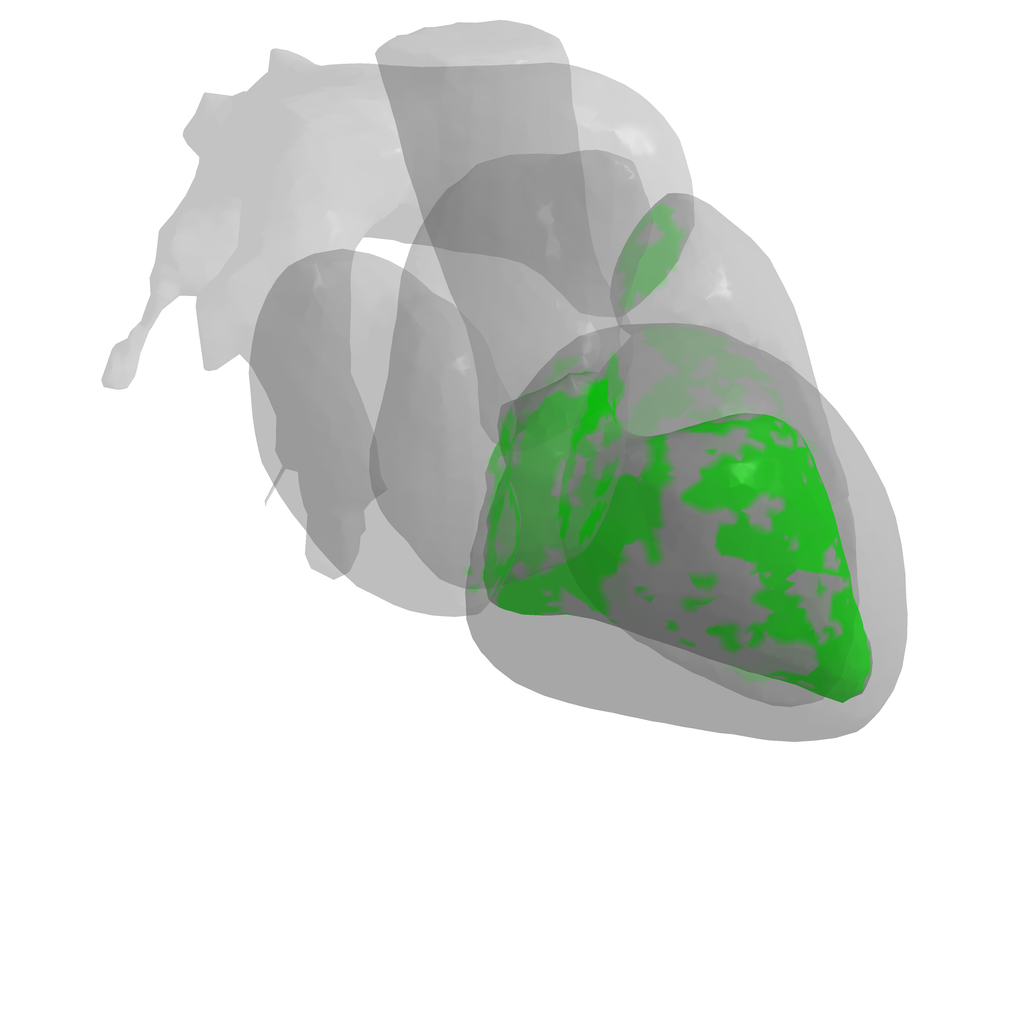}
        
    \end{tabular}
    \caption{\textbf{Qualitative Comparisons.} We compare the outputs of the vanilla approach to our QP variants on MedTet. The vanilla variants have significant intersections (red), while the constrained ones have only contact regions (green).}
    \label{fig:medtet}
\end{figure*}

\renewcommand{\inci}[1]{\includegraphics[width=0.24\textwidth, valign=c]{#1}}
\renewcommand{\inciCrop}[1]{\inci{#1}}
\begin{figure}[htbp]
  \centering
  
  \begin{minipage}[b]{0.49\textwidth}
    \centering
    \setlength{\tabcolsep}{0pt} 
    \begin{tabular}{c c c c}
        Vanilla & Loss & \shortstack{\acro{}\\(QP)\\- Train} & \shortstack{\acro{}\\(Shift-All)\\- Train} \\

        \inciCrop{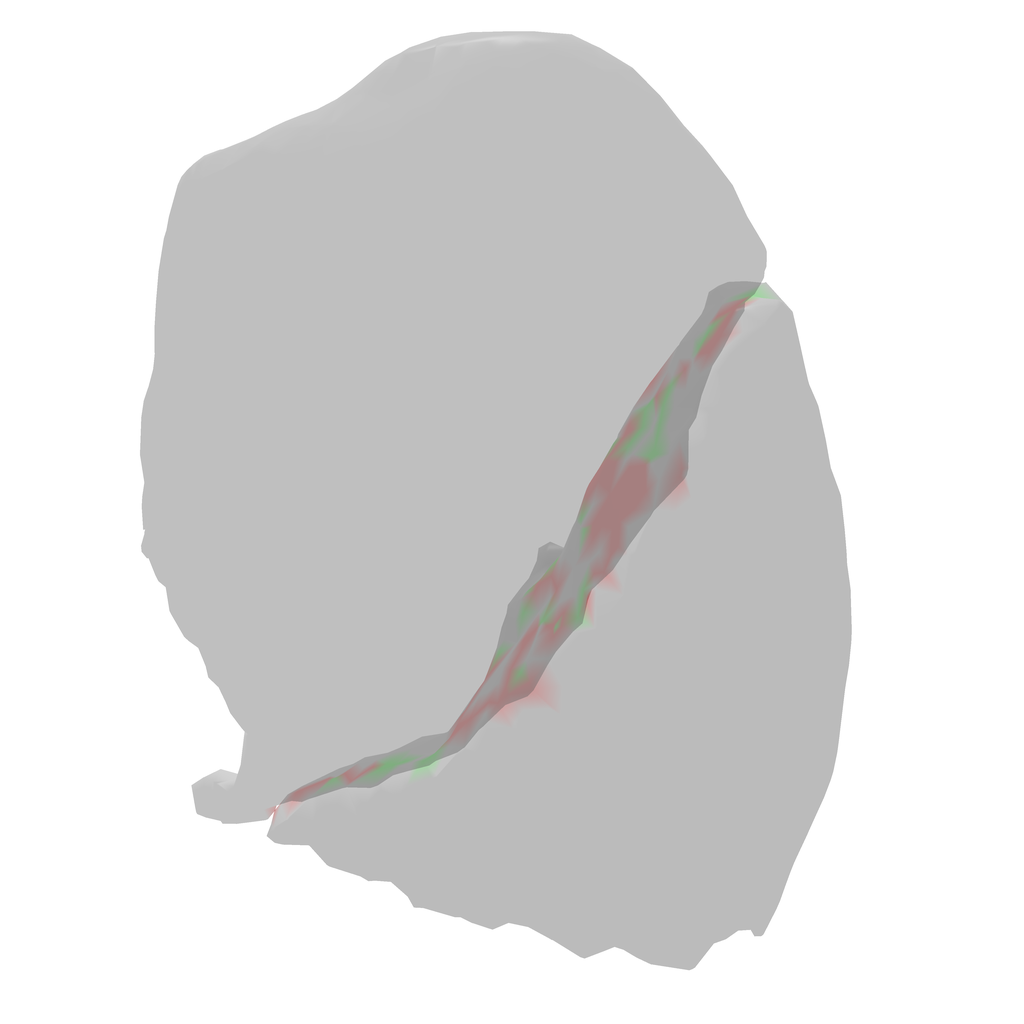} &
        \inciCrop{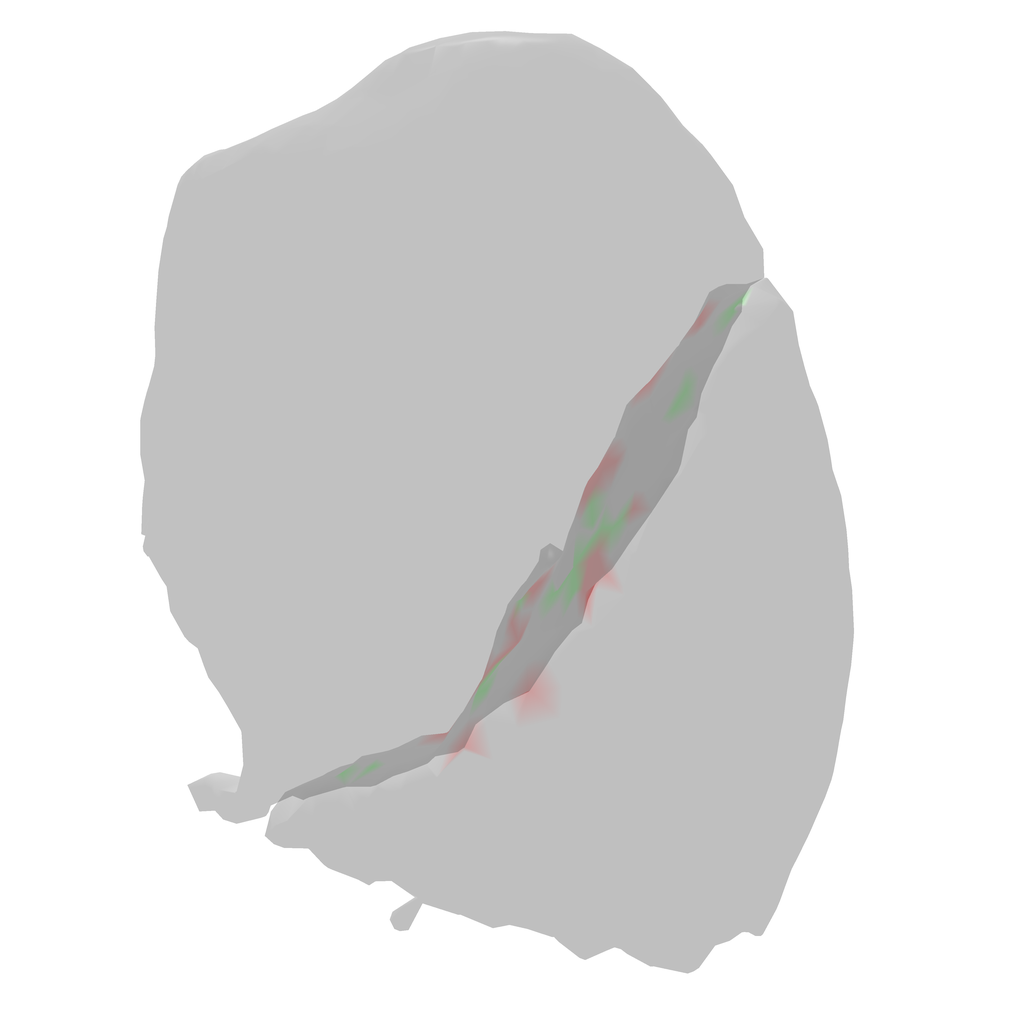} &
        \inciCrop{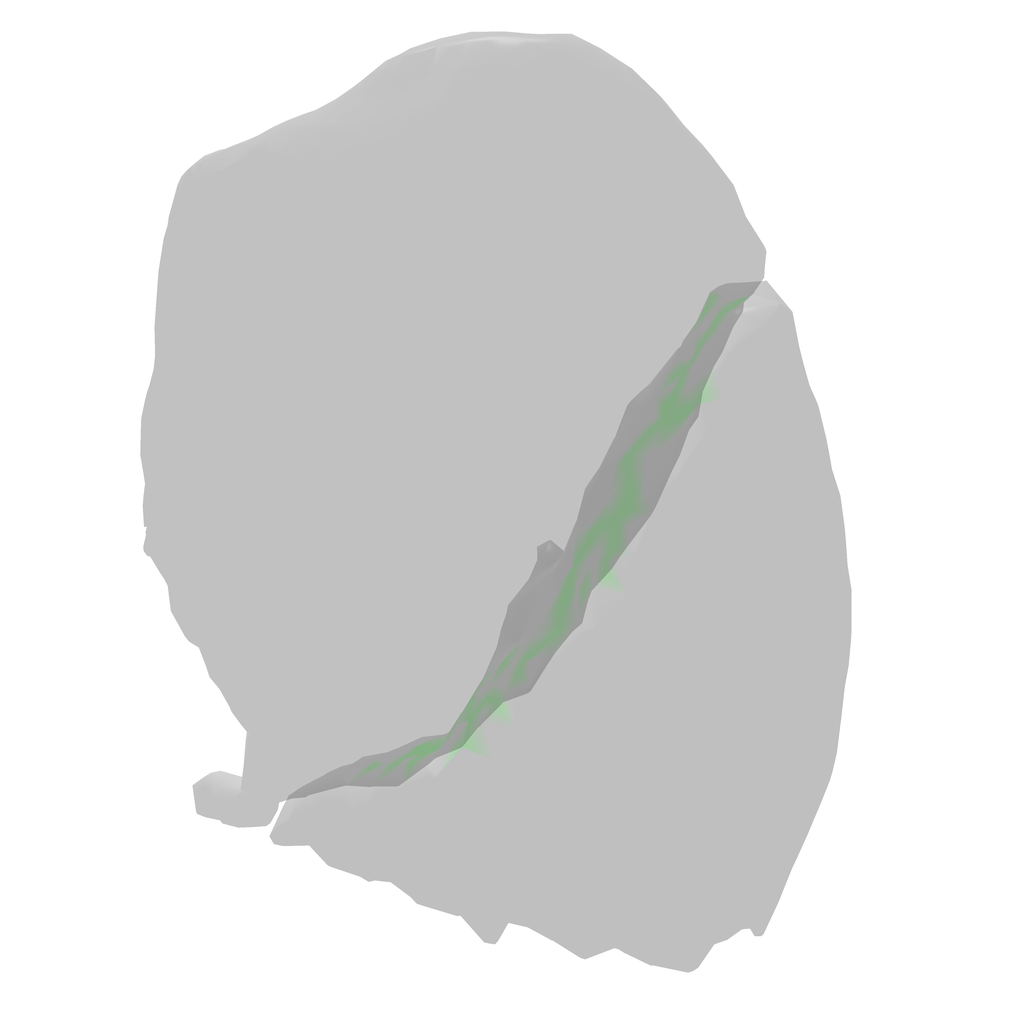} &
        \inciCrop{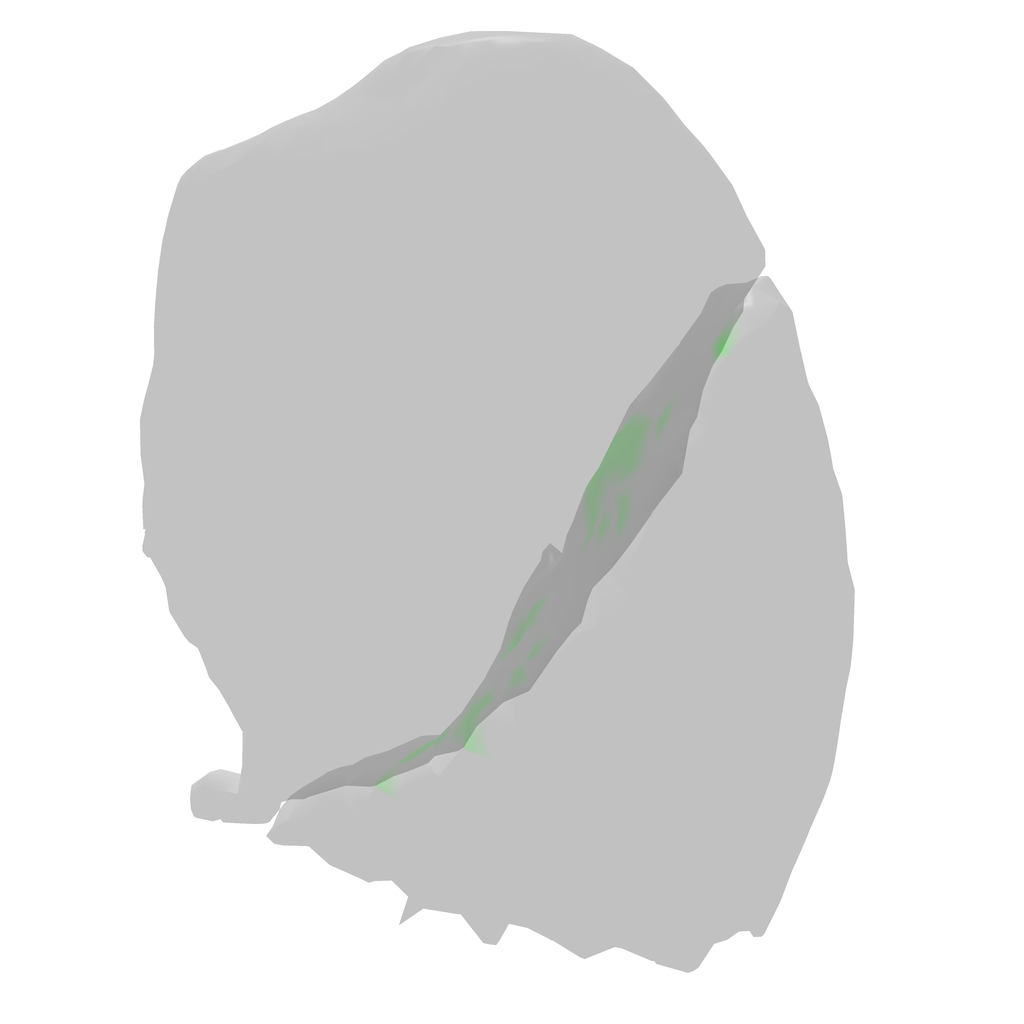}
        
    \end{tabular}
    \caption{Qualitative visualization on left-lung reconstructions. The loss term reduces the intersections compared to vanilla, but is far from completely satisfying non-intersection constraints.}
    \label{fig:abl-q}
  \end{minipage}
  \hspace{1mm}
  \begin{minipage}[b]{0.49\textwidth}
  \centering
  \resizebox{\textwidth}{!}{%
  \begin{tabular}{ccc}
    Model & MCD $(\times 10^{-4})$ $\downarrow$ & IV $\downarrow$ \\
    \midrule
    Vanilla & $6.270 \pm 0.201$ & $(2.282 \pm 1.060) \times 10^{-5}$ \\
    \cmidrule(lr){1-3}
    Loss & $6.469 \pm 0.372$ & $(2.069 \pm 0.772) \times 10^{-5}$ \\
    \cmidrule(lr){1-3}
    \acro{} (QP) - Train & $6.096 \pm 0.343$ & $(1.352 \pm 0.544) \times 10^{-9}$ \\
    \cmidrule(lr){1-3}
    \acro{} (Shift-All) - Train & $6.469 \pm 0.479$ & $(7.412 \pm 10.603) \times 10^{-12}$ \\
  \end{tabular}%
  }
    \captionof{table}{Metric comparison between the application of a soft intersection loss and our S2MDF layer during training for MedTet on the left lung. The soft loss only marginally decreases the intersection volume, while our approach reduces it by several orders of magnitude.}
    \label{tab:abl-q}
  \end{minipage}
  
\end{figure}

Our approach is designed to be generic and usable in conjunction with any signed distance function. To demonstrate this, we use it to impose non-intersection constraints on three modern multi-object modeling techniques~\citep{Wu23b,Chen24f,Talabot25}, some of which come with their own approaches to addressing the problem. This allows for comparisons between the different variants of our intersection prevention method and theirs, and shows the effectiveness of our method across different architectures, tasks and domains. We present the methods below.


\subsection{Target Methods}
\paragraph{MedTet~\citep{Chen24f}} is a medical imaging model designed to perform 3D organ reconstruction from CT/MRI scans using the Deep Marching Tetrahedra~\citep{Shen21a} algorithm. The model learns to deform and predict SDF values per-vertex on a tetrahedral grid, and then extracts surfaces using Marching Tetrahedra. Losses can be applied both on SDF predictions and the extracted mesh thanks to the surface extraction being differentiable. \fs{It does not include any intersection-mitigating mechanisms.} We evaluate MedTet in \cref{tab:medtet} and \cref{fig:medtet} on \fs{two} datasets:

MM-WHS~\citep{Zhuang18a}: a dataset for whole heart segmentation, containing 20 annotated CT samples with 7 heart components -- left ventricle, right ventricle, left atrium, right atrium blood cavities, left ventricle myocardium, aorta and pulmonary artery. 

TS-Lung~\citep{Wasserthal23}: the lung component of the TotalSegmentator dataset, containing 318 lung CT scans with 5 parts: 2 lobes of the left lung and 3 lobes of the right lung. While lung segmentation is considered a simple task with many approaches being able to achieve very high Dice scores~\citep{Turk25}, we find this dataset useful due to the large number of samples present and the possibility of multi-way intersections.

\textbf{ObjectSDF++~\citep{Wu23b}} is a neural implicit framework designed for high-fidelity, object-compositional 3D scene reconstruction. It utilises an occlusion-aware opacity rendering to reconstruct a scene as a collection of distinct, individual objects alongside a separate background. It includes an intersection-mitigating loss which substantially degrades the method if not used~\cite{Wu23b}, so we do not ablate it. We evaluate the model for scene reconstruction using the Replica~\citep{Straub19} dataset, adapted from MonoSDF~\citep{Yu2022MonoSDF}. Our evaluation covers 8 distinct scenes and is carried out as in the original paper, with each scene comprising dozens of individual objects, including the background. The ground-truth meshes used to compute our evaluation metrics are extracted with the vMAP framework \citep{Kong23}. We extract the meshes with Marching Cubes at resolution 256. The results are presented in \cref{tab:objectsdfplus_scenes_all} and \cref{fig:objectsdf_plus_render_comparisons}, excluding the IoU metric which cannot be computed in this setting due to non-watertight ground-truth meshes. 

\textbf{PartSDF}~\citep{Talabot25} is an SDF-based method for part-aware shape representation, with applications in shape reconstruction, generation and optimization. It represents shapes as a collection of parts, each represented by a separate SDF, and includes a loss term to mitigate intersections between parts akin to~\cite{Wu23b,Le25a}. We follow the experimental setup of the original paper for shape reconstruction and train one model on 100 chairs from PartNet~\citep{Mo19}, with 8 parts each, and a separate one on 100 mixers from~\citep{Vasu22}, with 4 parts each. We then evaluate the models by freezing the weights and optimizing the latent codes on 100 chairs and 100 mixers from the test set, and we extract them with Marching Cubes at resolution 128. Notice that, while the mixers are made of watertight parts, the chairs are not, which serves as an interesting test for the robustness of our approaches to a scenario in which supervision is missing specifically at the intersection between parts. We report the results in \cref{tab:partsdf} and \cref{fig:chairs}, with additional results in the Appendix.

\subsection{Experimental Protocol}

In all three cases discussed above, we compare the original version of the method against four variants of our \acro{} layer: \fs{\acro{} (QP) - PP applies the QP-based layer only at meshing time, effectively as a post-processing step for the SDF values; \acro{} (QP) - Train appends the QP-based layer to the end of the network, affecting both training and inference; \acro{} (Shift-All) - PP and \acro{} (Shift-All) - Train apply the Shift-All approach in the same way, respectively as a post-processing step and during training.}
%

\subsection{Evaluation Metrics}

To evaluate the results, we rely on metrics commonly used in 3D deep learning, along with \fs{one that computes intersection volume between objects and is key for assessing the validity of our method}. When possible, we base them on point cloud to facet matches rather than point cloud to point cloud for increased accuracy, and thus we add a \textbf{Mesh} prefix to the metric.

\textbf{Mesh Chamfer Distance}. When measuring Chamfer Distance between two meshes $M_1$ and $M_2$, points are uniformly sampled on the surface of the meshes: $P_{1} \sim \mathcal{U}(M_{1})$ and $P_{2} \sim \mathcal{U}(M_{2})$. However, this formulation is biased for a finite number of points (it is an overestimate), and inaccurate compared to finding the distance from a point to a mesh directly~\citep{Stella25}. We therefore define Mesh Chamfer Distance using sampled point to mesh distances rather than point to point, as
\begin{equation}
    d_{MCD}(M_1, M_2) = \frac{1}{|P_1|} \sum_{x \in P_1} \min_{y_m \in M_2} \|x - y_m\|_2^2 + \frac{1}{|P_2|} \sum_{y \in P_2} \min_{x_m \in M_1} \|x_m - y\|_2^2
\end{equation}
Note that $x_m$ and $y_m$ can be anywhere on the mesh: a vertex, edge, or on the surface of any of its constituent triangles. This distance can be calculated efficiently.

\textbf{Mesh Normal Consistency}. Measures agreement between the normals of two point clouds with cosine similarity, using nearest neighbor matches as with the Chamfer Distance. As with the Mesh Chamfer Distance, we find the nearest points directly on the target mesh and use interpolated vertex normals rather than per-triangle flat normals. 

\textbf{F1 score ($\tau$)}. As done in ObjectSDF++~\citep{Wu23b}, measures the ratio of points from $M_1$ at distance less than $\tau$ from $M_2$. Precision is from predicted mesh to ground truth, recall is vice-versa, and F1 score is the harmonic mean of the two, being bidirectional. We once again measure distances with point-to-mesh approaches. 

\textbf{Mesh IoU}. Measures the overall alignment between two meshes, as their volumetric intersection over union: $\frac{\text{vol}(M_1 \cap M_2)}{\text{vol}(M_1 \cup M_2)}$. We measure this approximately using a discrete voxel grid on the union bounding box of the two meshes. 

\textbf{Intersection Volume}. The volume of the intersection between two meshes, $M_1$ and $M_2$. This is calculated between predicted meshes for different parts of the same object or scene. It is always 0 in the ground truth. 


\subsection{Comparative Results}
We execute experiments on single NVIDIA A100 or V100 GPUs, ensuring consistency when measuring times. We refer to the original papers for more details on the setup and the compute requirements.
We report the results of our experiments in \cref{tab:medtet,tab:objectsdfplus_scenes_all,tab:partsdf} and \cref{fig:medtet,fig:objectsdf_plus_render_comparisons,fig:chairs} for MedTet, ObjectSDF++ and PartSDF respectively. We run each experiment 3 times, to compute mean and standard deviations. Across all results, we observe that both QP and Shift-All reduce intersections by several orders of magnitude with respect to the original methods, including PartSDF and ObjectSDF++ which already include a loss term to mitigate them. The reduction in intersections is qualitatively visible in \cref{fig:medtet,fig:objectsdf_plus_render_comparisons,fig:chairs}, where the red color indicates intersecting surfaces and the green color indicates touching surfaces. In terms of reconstruction quality, we observe that the metrics are quite noisy, but overall remain close to the original methods in mean and variance, with some exceptions. In particular, we notice that our S2MDF methods tend to perform slightly better when used during training than only as a post-processing step, as the network can adapt to the constraint during training. We also notice that the Shift-All approach performs as well as the QP-based one during post-processing and slightly better during training, while being much faster. For few objects, like in the case of PartSDF mixers which only have 4 parts, QP does not add significant overhead when used as post-processing, however it increases the training time from 3h30m to 5h40m when used during training. For more complex objects, like the PartSDF chairs with 8 parts, the overhead of the QP-based approach is much more significant, increasing the training time from 4 to 9 hours for 2000 epochs, and the inference time from 20 to 60 seconds per object. On ObjectSDF++, where there are several dozens of objects per scene, the overhead of the QP-based approach makes it infeasible to use it during training. On the other hand, the Shift-All approach did not add any measurable overhead during training or post-processing in any of our tests, across all methods. As such, we recommend the use of \acro{} with the Shift-All approach.

\paragraph{Comparison against loss-based solutions}
\fs{To compare our method against loss-based intersection mitigations we run an experiment on TS-Lung for MedTet using the left lung only}, where we apply an additional loss term to penalize intersections that does not exist in the base model: $\mathcal{L}_{inter} = \frac{1}{K-1} \sum_{k \ne min} \max(0, -(SDF_{min} + SDF_{k}))$. This equation penalizes overlapping geometry by enforcing that if a point is deeply inside one object ($SDF_{min} < 0$), it must be sufficiently outside all other objects ($SDF_{k} > 0$) so that their sum remains non-negative. This loss is the same as the one applied in \fs{ObjectSDF++ and similar to the one in PartSDF} and \cite{Le25a}. We show that the improvement offered by this loss in terms of intersection volume is marginal in \cref{fig:abl-q} and \cref{tab:abl-q}, whereas our approach reduces it by several orders of magnitude. Our approach remains compatible with non-intersection losses, as they can help regularizing the underlying field, but it still vastly improves the results. We refer to \cref{sec:nonintloss_partsdf} for an ablation on the non-intersection loss in PartSDF, showing similar results.

\paragraph{Bringing intersections to zero}
\label{sec:residual_iv}
The intersection volumes that remain after applying \acro{} come from numerical precision. Due to our use of mixed-precision training, the predicted SDFs are stored in 16-bit floating point, which limits how small the IV in \cref{tab:medtet} can get. On TS-Lung, storing them in 32-bit instead brings the IV down to around $10^{-22}$ with Shift-All and $10^{-12}$ with QP, where the latter is limited by the maximum number of solver iterations we allow. \fs{To show that our approach is capable of fully eliminating intersections,} we perform a small experiment comparing $\epsilon = 0$ and $\epsilon = 10^{-4}$ for the QP-based solution where we set $d_i + d_j \ge \epsilon$. We note that the first one evaluates to around $1.2 \times 10^{-8}$ in terms of intersection volume, while the second one evaluates to exactly $0$, \fs{showing that setting a small threshold effectively eliminates intersections, which can otherwise appear due to numerical errors.}

\paragraph{Eikonal loss}
Regularization losses such as the eikonal loss are commonly used in SDF-based methods to encourage smoothness and regularity of the learned fields. We perform an ablation study on ObjectSDF++ to evaluate the effect of applying the Shift-All approach during training with and without a skip connection that allows the eikonal loss to be computed on the original network outputs. We observe that our layer performs overall similarly to the original method in both cases, but produces a noticeably lower MCD when the eikonal loss is applied to the original network output. We conjecture that regularization losses are better suited for the original field, as they help the network learn a naturally better field before applying the constraint layer. \fs{Thus, we suggest to apply regularization losses to the original field}.





\section{Limitations}
\label{sec:limitations}

Our approach guarantees an intersection-free signed-distance field. However, the extracted meshes can only be intersection-free if they are extracted by linear interpolation of the field. This holds for the commonly used Marching Tetrahedra and Marching Cubes. For methods where vertex placement is learned rather than linear, such as FlexiCubes~\citep{Shen23} (\cref{sec:rebut_extractors}), our guarantee does not hold. Moreover, when \acro{} is only applied as a post-processing step, overlapping surfaces are both moved to where their fields cross, which may not be how the real geometry behaves. When it is applied during training instead, the network can adapt to the constraint. We analyze these limitations in more detail in \cref{sec:additional_analyses}, together with some edge cases of Marching Cubes.

\section{Conclusion}

In this work, we have introduced a constrained SDF formulation to guarantee intersection-free multi-object implicit surface representations, and we have proposed S2MDF, a fast, plug-and-play module that can be applied to any object-compositional SDF-based representation to enforce the constraint, whose output remains compatible with linearly-interpolated meshing techniques. We have shown that our method eliminates intersections in theory and reduces them to numerical error in practice, without degrading the accuracy of the reconstructions and sometimes even improving it, which existing loss-based solutions cannot do. While we can address the numerical issue by adding a small margin to our constraint, an even more principled solution would be to design meshing algorithms that take into account multiple objects and can handle them correctly, which is a topic for future work.

\begin{ack}
This work was funded in part by the Swiss National Science Foundation.
\end{ack}

\clearpage
\bibliographystyle{plain}
\bibliography{bib/string,bib/graphics,bib/vision,bib/biomed,bib/optim,new}

\clearpage
\appendix
\renewcommand\thefigure{A.\arabic{figure}} 
\renewcommand\thetable{A.\arabic{table}} 
\renewcommand\thesection{A.\arabic{section}}
\setcounter{figure}{0}
\setcounter{table}{0}
\setcounter{section}{0}
\section{Additional Results}
\label{sec:additional_results}
We show here additional results and figures that were not included in the main text. In \cref{fig:mixers}, we show visualizations of the intersections for the PartSDF~\cite{Talabot25} experiment on the Mixers category. Similar conclusions as the main text apply. 
\begin{figure}[]
  \centering
  \setlength{\tabcolsep}{2pt}
  \resizebox{\textwidth}{!}{%
    \begin{tabular}{cccccc}
      \includegraphics[width=0.15\textwidth]{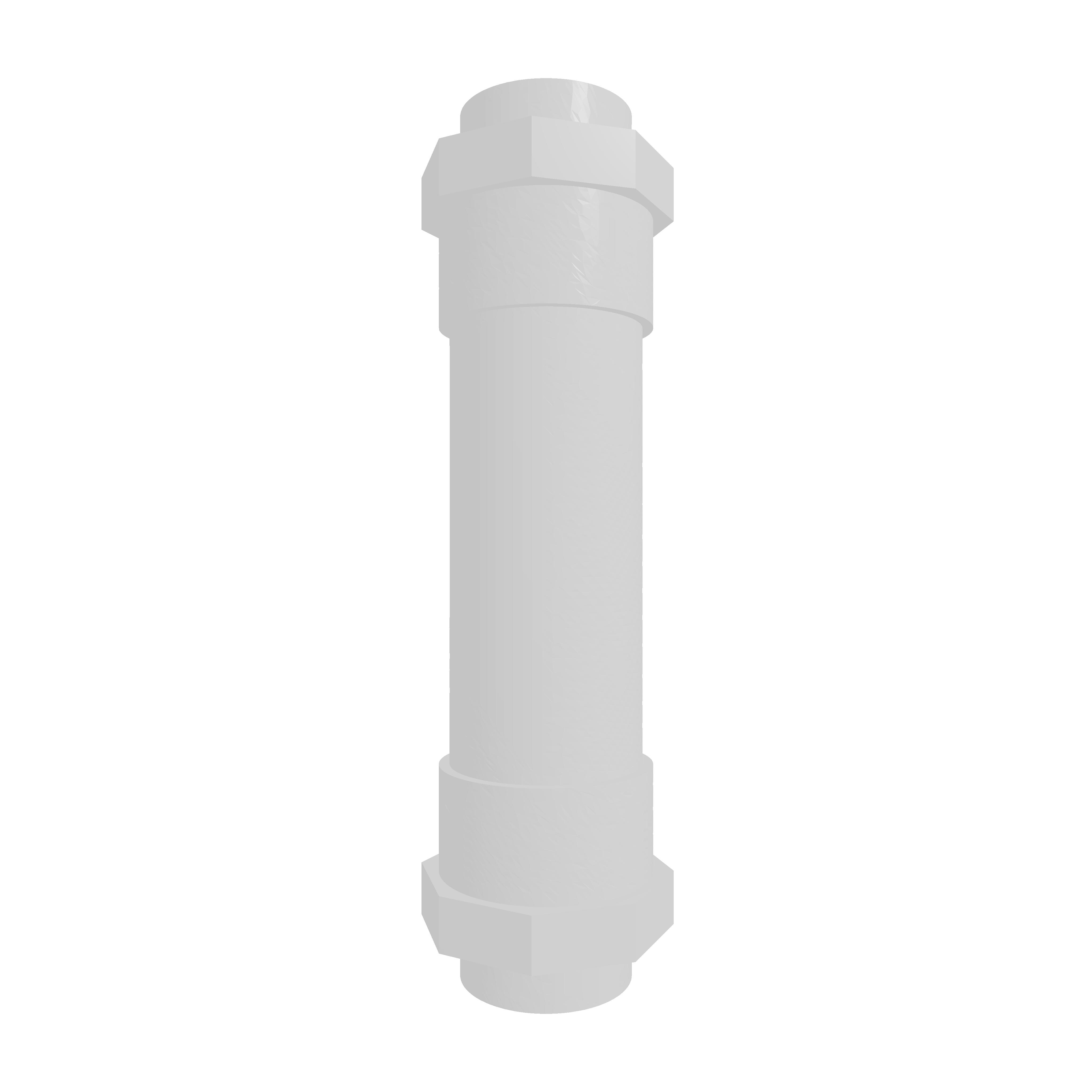} &
      \includegraphics[width=0.15\textwidth]{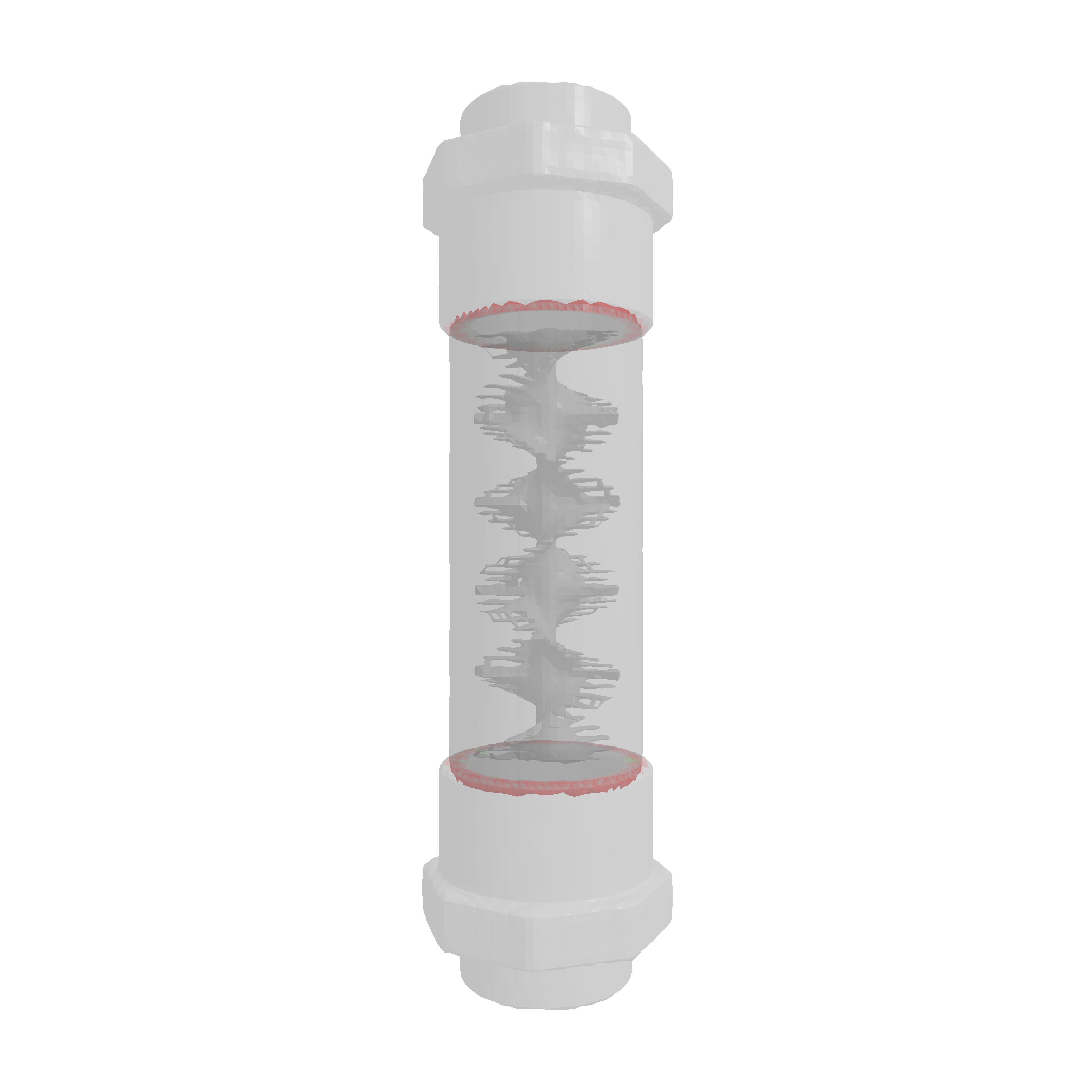} &
      \includegraphics[width=0.15\textwidth]{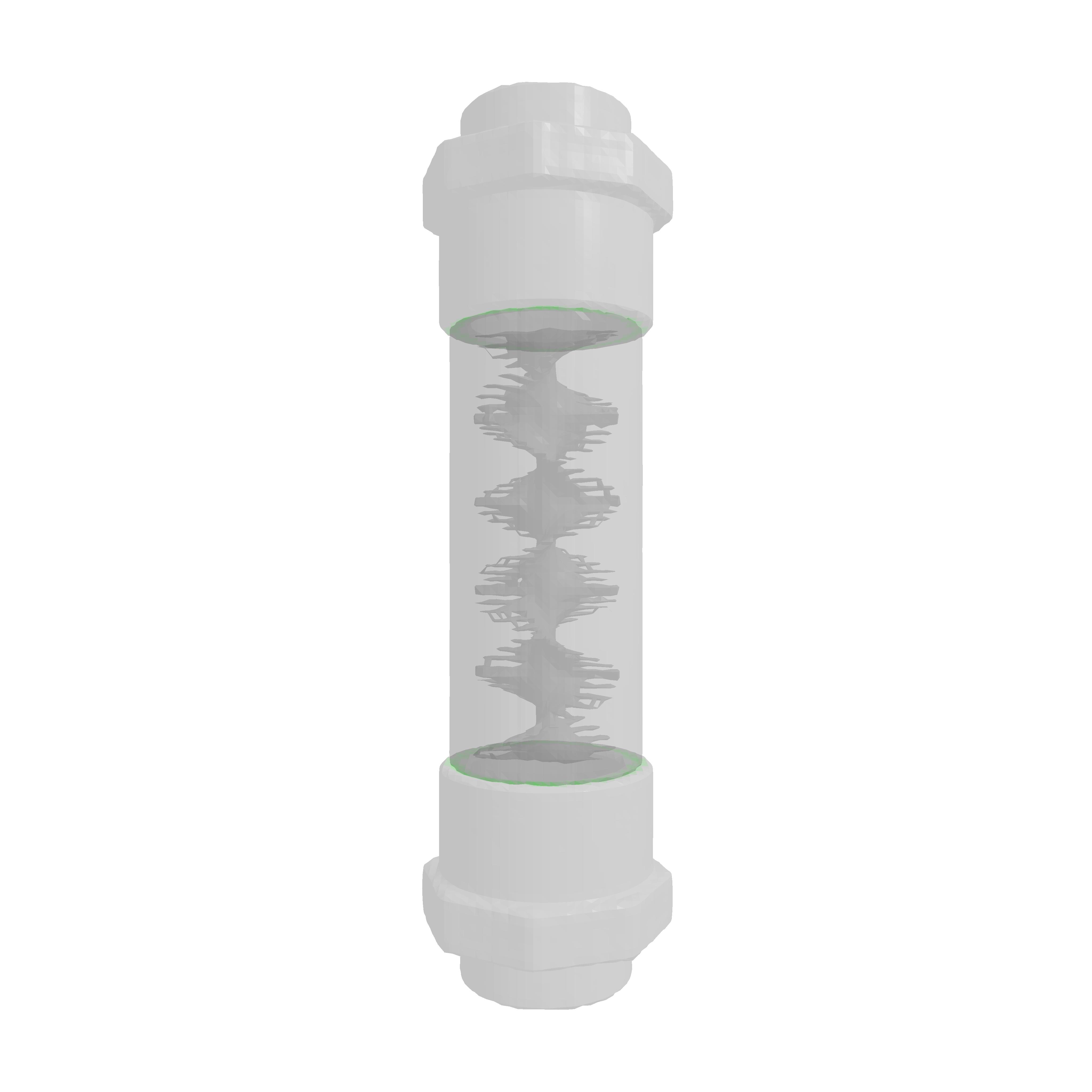} &
      \includegraphics[width=0.15\textwidth]{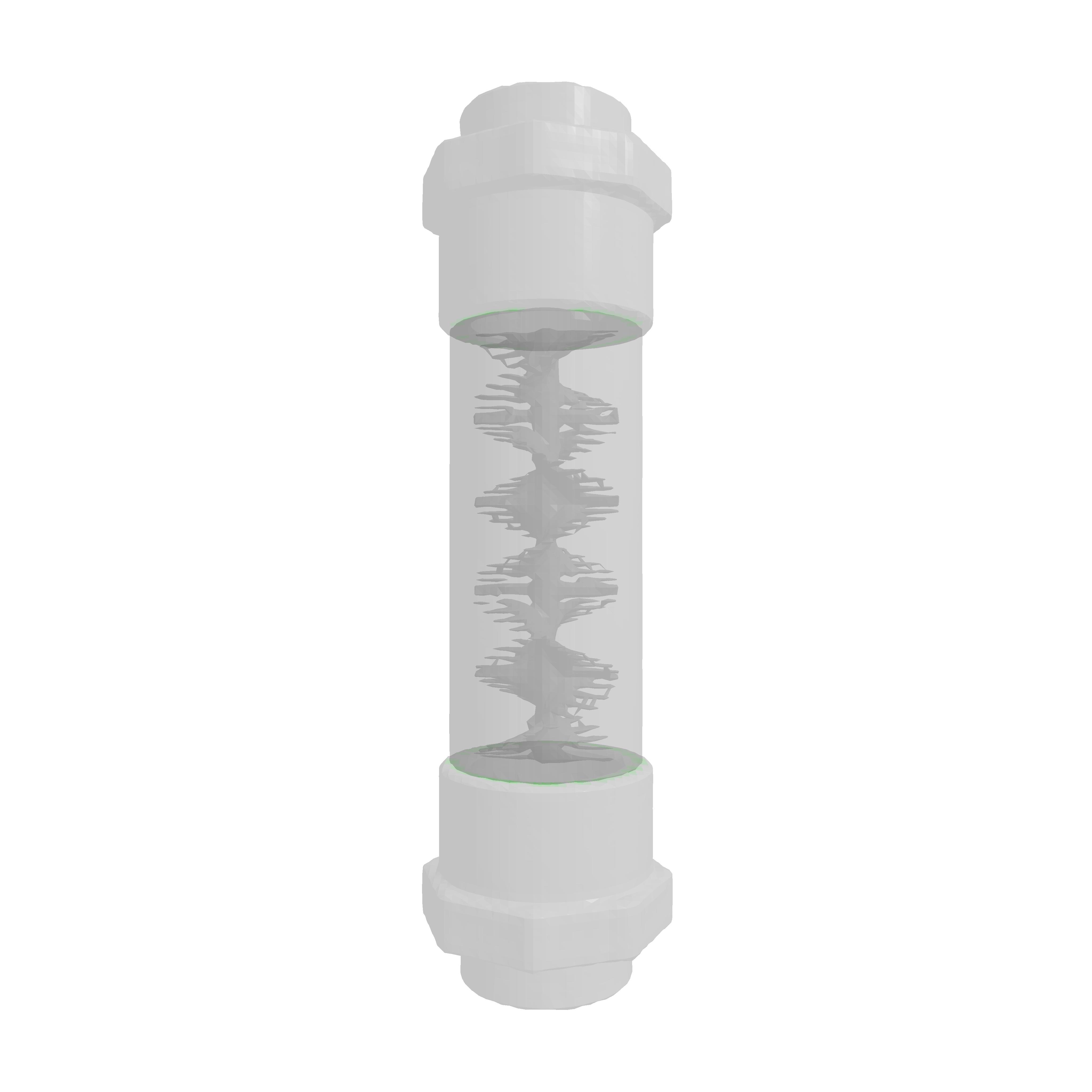} &
      \includegraphics[width=0.15\textwidth]{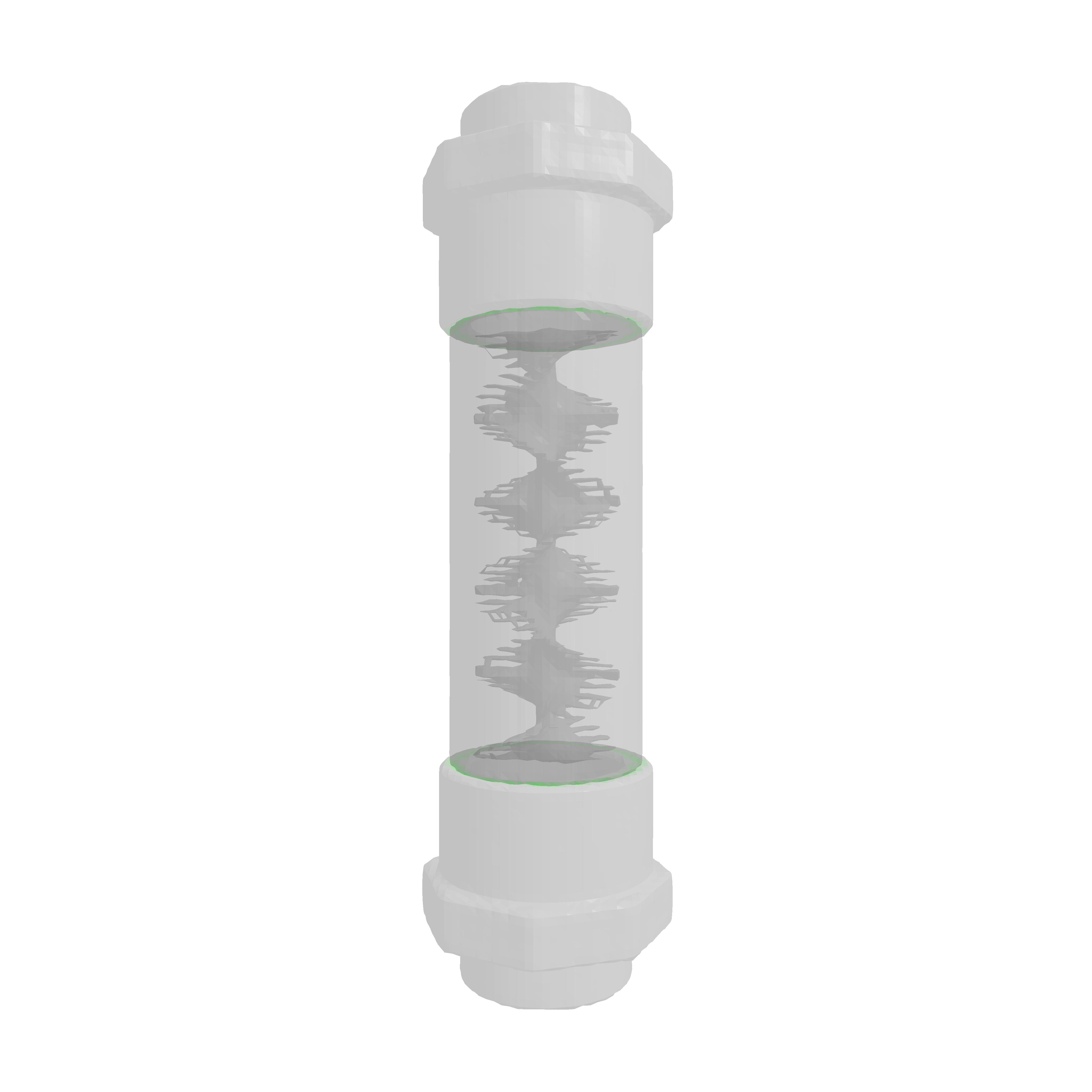} &
      \includegraphics[width=0.15\textwidth]{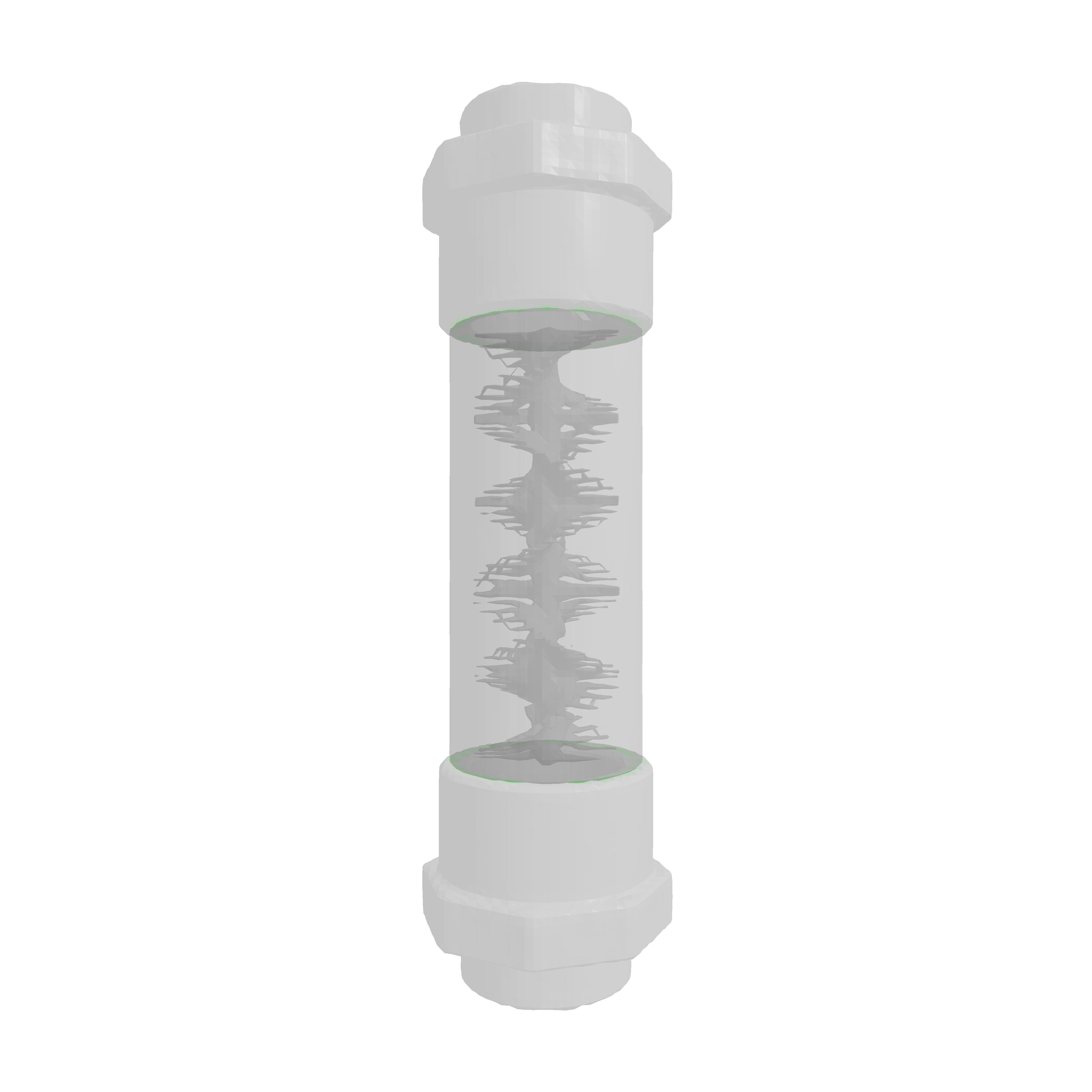} \\
      Ground Truth &
      Vanilla &
      \begin{tabular}[t]{@{}c@{}}\acro{}\\(QP)\\- PP\end{tabular} &
      \begin{tabular}[t]{@{}c@{}}\acro{}\\(QP)\\- Train\end{tabular} &
      \begin{tabular}[t]{@{}c@{}}\acro{}\\(Shift-All)\\- PP\end{tabular} &
      \begin{tabular}[t]{@{}c@{}}\acro{}\\(Shift-All)\\- Train\end{tabular} \\
    \end{tabular}%
  }
  \caption{PartSDF~\cite{Talabot25}: visualization of intersections. Intersecting surfaces are in red, while touching surfaces are in green.}
  \label{fig:mixers}
\end{figure}


\subsection{An alternative solution}
The solutions we presented are not the only possible ways to enforce the MDF constraint. A simple but more heuristic approach consists in always taking the minimum SDF value across all objects whenever the MDF constraint is violated. This can be expressed as a projection operator $\text{Min}(\vu)$, which takes as input the vector of SDF values $\vu \in \mathbb{R}^K$ for $K$ objects at a given point and outputs a projected vector that satisfies the MDF constraint. The projection operator can be defined as follows:
\begin{equation}
\text{Min}(\vu) = 
\begin{dcases}
\vu & \text{if } \mino (\vu) + \mint (\vu) \ge 0 \; ,\\
\mino(\vu) \cdot \mathbf{1}_{\argmino(\vu)} & \text{otherwise.}
\end{dcases}
\end{equation}

Where $\mathbf{1}_{\argmino(\vu)} \in \mathbb{R}^K$ is a vector with 1 at the index of the minimum entry of $\vu$ and -1 elsewhere. This makes the $Min(\vu)$ vector negative only at the index of the minimum entry of $\vu$, and positive (with an equal but opposite value) at all other entries, thus satisfying the MDF constraint. 

This approach works well in practice, but it does not supersede the QP and Shift-All approaches we presented in the main text, which are more principled and provide better results. We include a comparison of this approach with the other methods in \cref{tab:partsdf_chair100,tab:partsdf_mixer100}.

\subsection{Additional results on PartSDF}

\subsubsection{Inference-time optimization with \acro{}}
\fs{
Contrary to the other methods we evaluate, PartSDF~\cite{Talabot25} is not limited to applying our \acro{} module either during training or as a post-processing step for the field at meshing time. Instead, it can also be applied during the inference-time optimization process that PartSDF performs to refine the geometry of each part, freezing the weights and optimizing the latent code. In \cref{tab:partsdf_chair100,tab:partsdf_mixer100}, we present quantitative results for the same PartSDF~\cite{Talabot25} experiment as in \cref{tab:partsdf}, additionally comparing the application of \acro{} during the inference-time optimization ("Optim"). We observe that, while not disrupting the optimization process, applying \acro{} during the optimization does not provide advantages over applying it during training or as a simple post-processing step, and it is slightly outperformed by both in most occasions. This is reasonable, as the network weights have been trained without such layer, and the resulting latent space may not be optimal to allow for a good optimization with the \acro{} module.

\begin{table}[ht]
\centering
\resizebox{\textwidth}{!}{%
  \begin{tabular}{c|ccccc}
    & MCD $(\times 10^{-4})$ $\downarrow$ & MNC $\uparrow$ & F1 (0.001) $\uparrow$ & IoU $\uparrow$ & IV $\downarrow$ \\
    \hline
    Vanilla PartSDF & $4.331 \pm 0.690$ & $95.119 \pm 0.015$ & $96.731 \pm 0.044$ & $85.630 \pm 0.101$ & $(1.876 \pm 0.062) \times 10^{-3}$ \\
    \hline
    \acro{} (Min) - PP & $2.017 \pm 0.165$ & $94.728 \pm 0.011$ & $97.933 \pm 0.025$ & $89.172 \pm 0.106$ & $(1.500 \pm 0.032) \times 10^{-7}$ \\
    \acro{} (Shift-All) - PP & $1.973 \pm 0.166$ & $94.824 \pm 0.008$ & $98.005 \pm 0.022$ & $89.173 \pm 0.104$ & $(5.194 \pm 1.040) \times 10^{-9}$ \\
    \acro{} (QP) - PP & $2.126 \pm 0.165$ & $94.773 \pm 0.014$ & $97.902 \pm 0.042$ & $89.151 \pm 0.107$ & $(6.157 \pm 1.394) \times 10^{-9}$ \\
    \hline
    \acro{} (Min) - Optim & $2.200 \pm 0.160$ & $94.687 \pm 0.016$ & $97.846 \pm 0.024$ & $89.013 \pm 0.103$ & $(1.652 \pm 0.064) \times 10^{-7}$ \\
    \acro{} (Shift-All) - Optim & $2.616 \pm 0.576$ & $94.879 \pm 0.039$ & $98.116 \pm 0.021$ & $89.664 \pm 0.048$ & $(3.194 \pm 0.566) \times 10^{-9}$ \\
    \acro{} (QP) - Optim & $2.195 \pm 0.428$ & $94.842 \pm 0.035$ & $98.076 \pm 0.011$ & $89.618 \pm 0.086$ & $(8.740 \pm 7.756) \times 10^{-9}$ \\
    \hline
    \acro{} (Min) - Train & $2.154 \pm 0.137$ & $94.723 \pm 0.009$ & $97.896 \pm 0.035$ & $89.063 \pm 0.155$ & $(1.467 \pm 0.093) \times 10^{-7}$ \\
    \acro{} (Shift-All) - Train & $2.120 \pm 0.117$ & $94.823 \pm 0.032$ & $98.039 \pm 0.017$ & $89.297 \pm 0.262$ & $(5.665 \pm 2.370) \times 10^{-9}$ \\
    \acro{} (QP) - Train & $2.189 \pm 0.364$ & $94.885 \pm 0.078$ & $98.047 \pm 0.031$ & $89.289 \pm 0.126$ & $(7.000 \pm 0.816) \times 10^{-9}$ \\
 \end{tabular}%
}
\caption{Shape reconstruction metrics for PartSDF on 100 chairs from the test set. Mean and standard deviation across 3 runs.}
\label{tab:partsdf_chair100}
\end{table}
}

\begin{table}[ht]
\centering
\resizebox{\textwidth}{!}{%
  \begin{tabular}{c|ccccc}
    & MCD $(\times 10^{-4})$ $\downarrow$ & MNC $\uparrow$ & F1 (0.001) $\uparrow$ & IoU $\uparrow$ & IV $\downarrow$ \\
    \hline
    Vanilla PartSDF & $1.386 \pm 0.090$ & $88.013 \pm 0.074$ & $98.441 \pm 0.107$ & $76.667 \pm 0.168$ & $(4.317 \pm 0.221) \times 10^{-5}$ \\
    \hline
    \acro{} (Min) - PP & $1.358 \pm 0.105$ & $88.113 \pm 0.059$ & $98.441 \pm 0.107$ & $76.752 \pm 0.145$ & $(1.251 \pm 0.170) \times 10^{-7}$ \\
    \acro{} (Shift-All) - PP & $1.362 \pm 0.102$ & $88.000 \pm 0.135$ & $98.441 \pm 0.107$ & $75.774 \pm 0.887$ & $(1.020 \pm 0.349) \times 10^{-7}$ \\
    \acro{} (QP) - PP & $1.362 \pm 0.102$ & $88.000 \pm 0.135$ & $98.441 \pm 0.107$ & $75.769 \pm 0.892$ & $(1.018 \pm 0.349) \times 10^{-7}$ \\
    \hline
    \acro{} (Min) - Optim & $1.358 \pm 0.107$ & $88.108 \pm 0.060$ & $98.440 \pm 0.116$ & $76.747 \pm 0.143$ & $(9.715 \pm 2.166) \times 10^{-8}$ \\
    \acro{} (Shift-All) - Optim & $1.395 \pm 0.127$ & $88.066 \pm 0.069$ & $98.407 \pm 0.146$ & $76.524 \pm 0.159$ & $(1.694 \pm 1.634) \times 10^{-9}$ \\
    \acro{} (QP) - Optim & $1.397 \pm 0.122$ & $88.055 \pm 0.077$ & $98.391 \pm 0.130$ & $76.517 \pm 0.163$ & $(1.525 \pm 1.822) \times 10^{-9}$ \\
    \hline
    \acro{} (Min) - Train & $1.243 \pm 0.191$ & $88.190 \pm 0.193$ & $98.535 \pm 0.203$ & $76.749 \pm 0.357$ & $(1.200 \pm 0.604) \times 10^{-8}$ \\
    \acro{} (Shift-All) - Train & $1.234 \pm 0.137$ & $88.175 \pm 0.045$ & $98.593 \pm 0.153$ & $76.595 \pm 0.201$ & $(1.680 \pm 0.851) \times 10^{-8}$ \\
    \acro{} (QP) - Train & $1.247 \pm 0.032$ & $88.147 \pm 0.090$ & $98.595 \pm 0.026$ & $76.648 \pm 0.170$ & $(1.239 \pm 0.831) \times 10^{-8}$ \\
 \end{tabular}%
}
\caption{Shape reconstruction metrics for PartSDF on 100 mixers from the test set. Mean and standard deviation across 3 runs.}
\label{tab:partsdf_mixer100}
\end{table}

\subsubsection{Artifacts introduced by \acro}
\begin{figure*}[t]
  \centering
  \setlength{\tabcolsep}{1pt}
  \newcommand{\chairresimage}[1]{%
    \adjustbox{margin=2pt,valign=c}{\includegraphics[width=0.138\textwidth]{#1}}%
  }
  \begin{tabular}{@{}c@{\hspace{2pt}}ccc@{\hspace{4pt}}ccc@{}}
    & \small 128 & \small 256 & \small 512
    & \small 128 & \small 256 & \small 512 \\[4pt]
    \adjustbox{valign=c}{\shortstack[c]{\scriptsize Vanilla\\[-1pt]\scriptsize PartSDF}} &
    \chairresimage{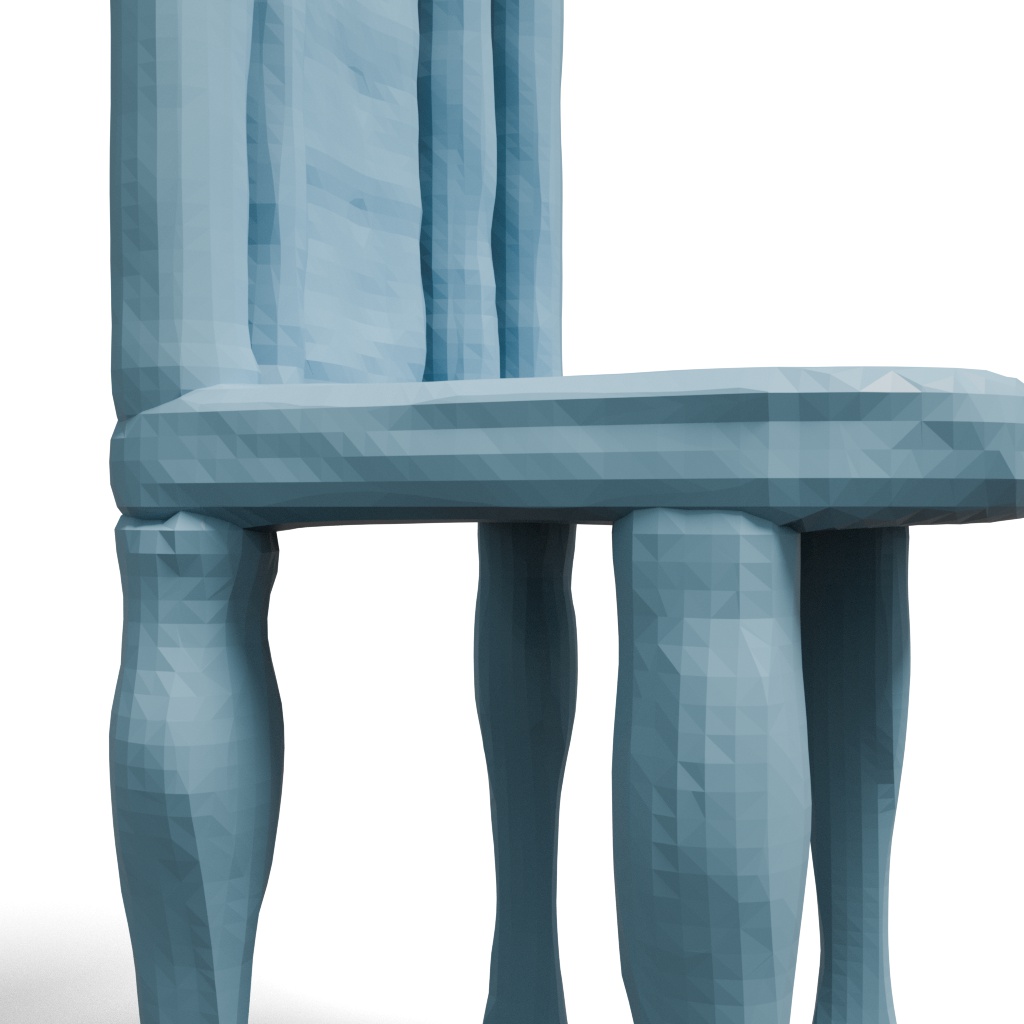} &
    \chairresimage{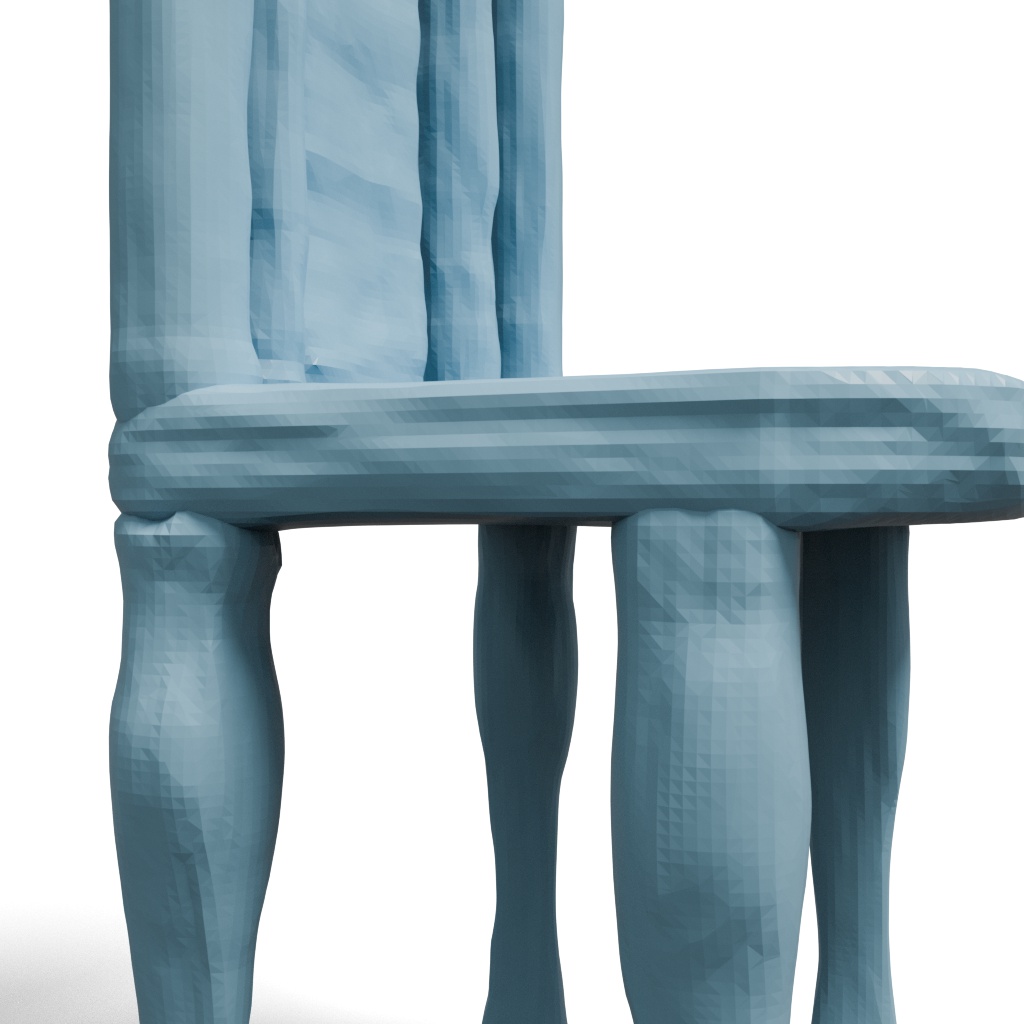} &
    \chairresimage{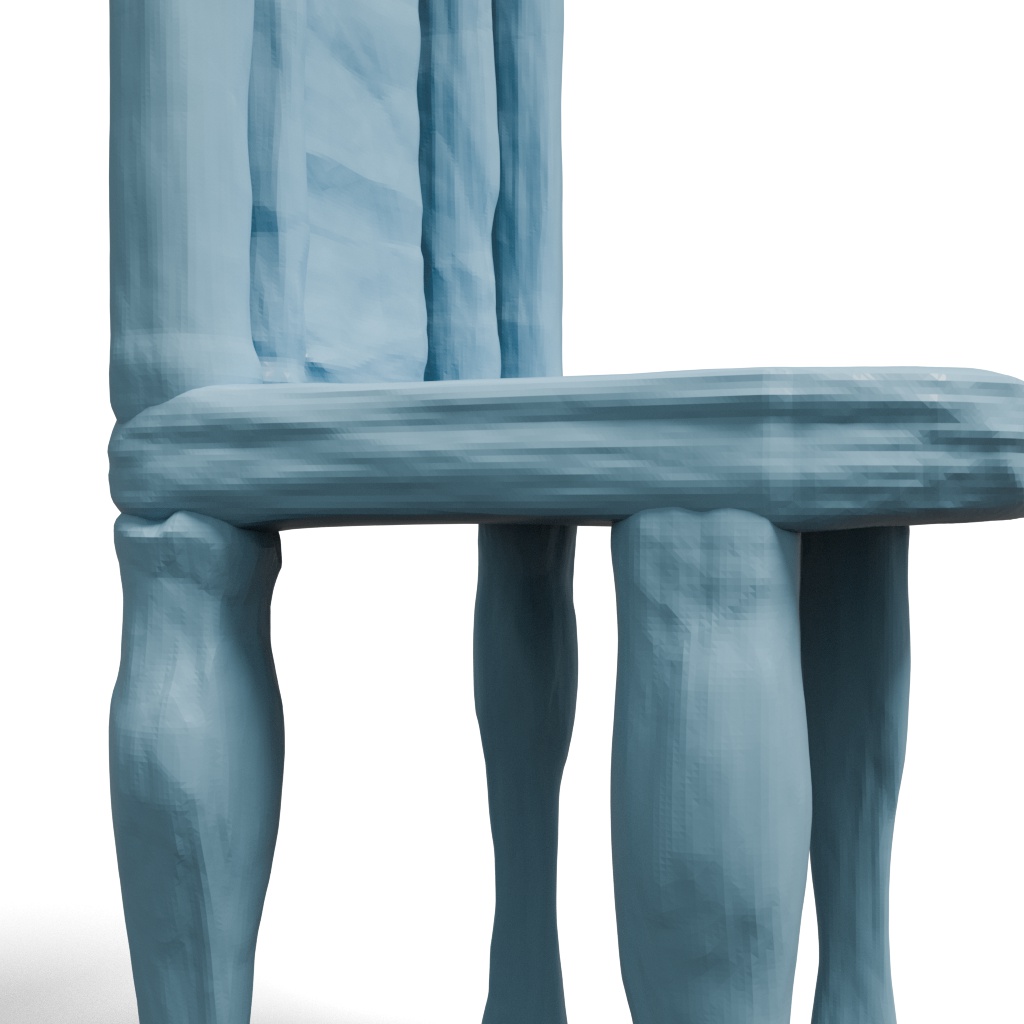} &
    \chairresimage{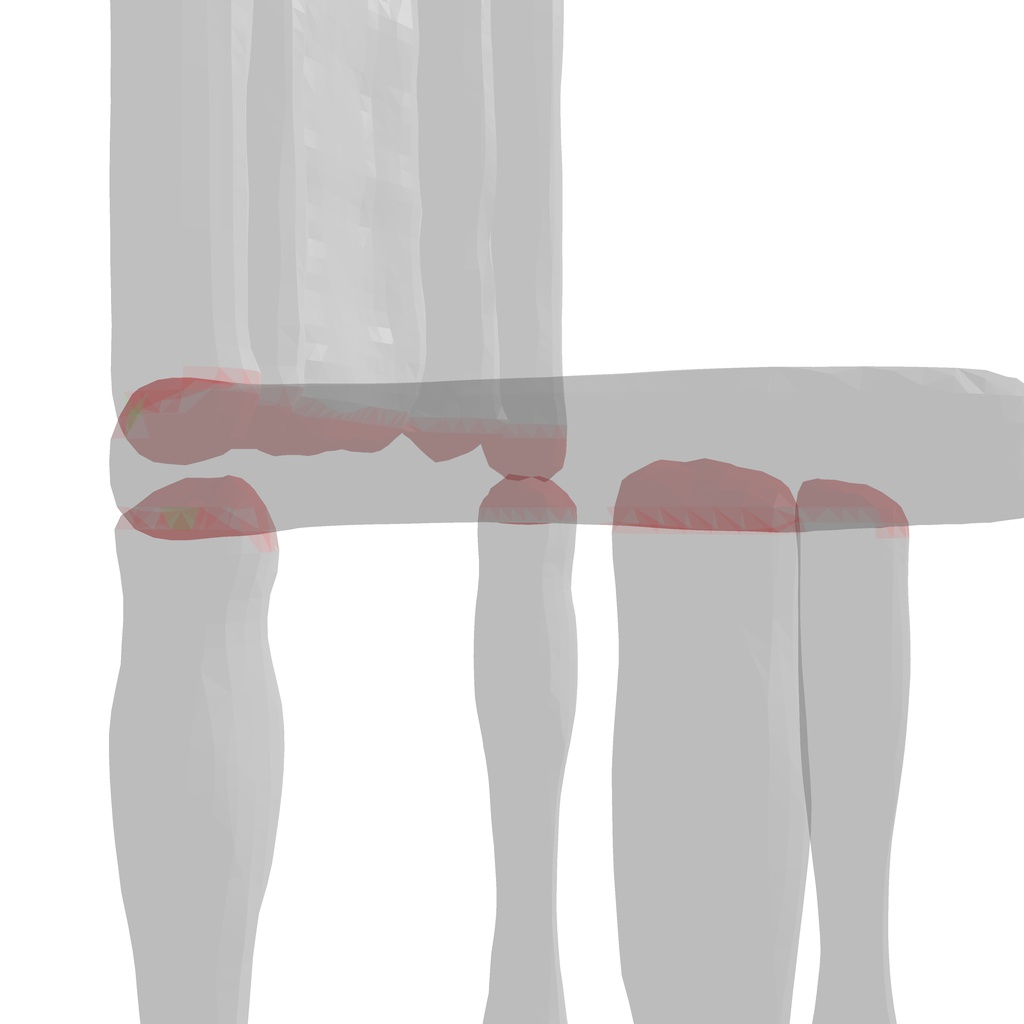} &
    \chairresimage{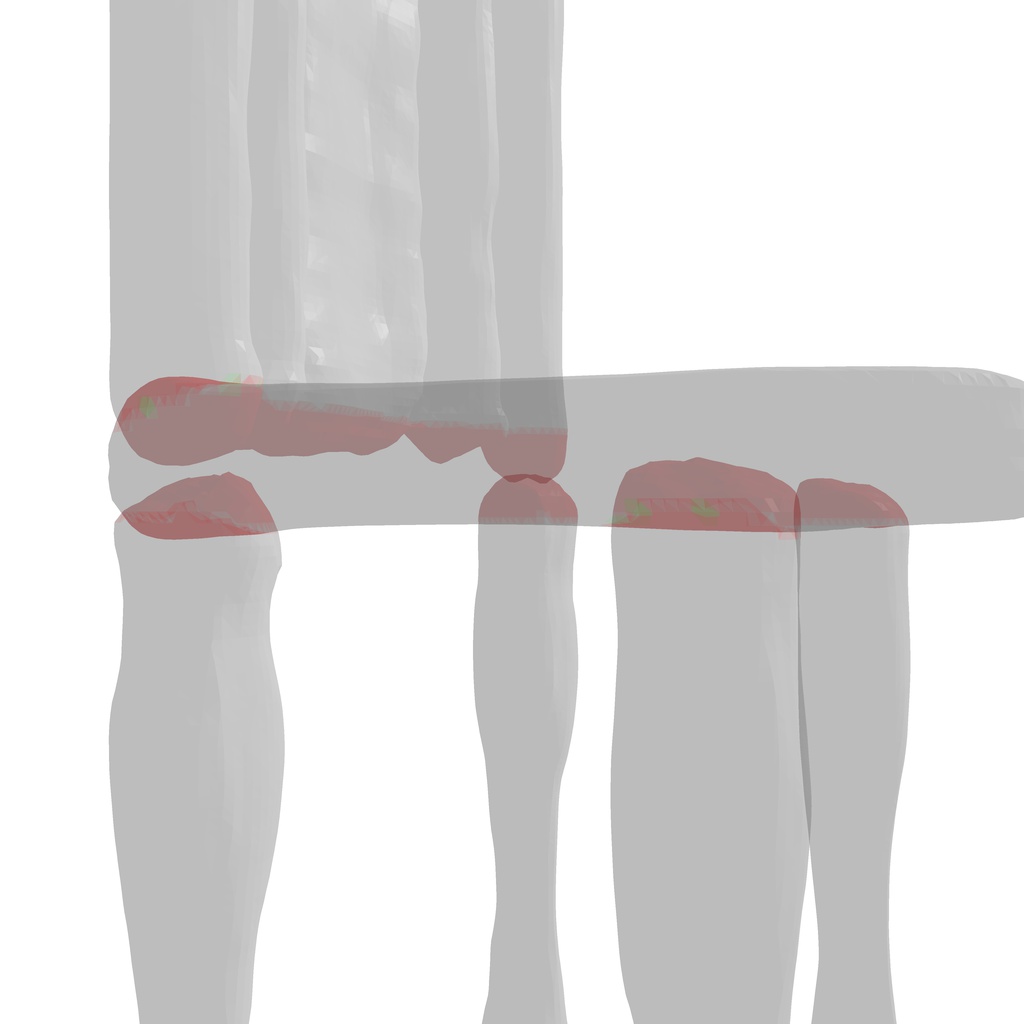} &
    \chairresimage{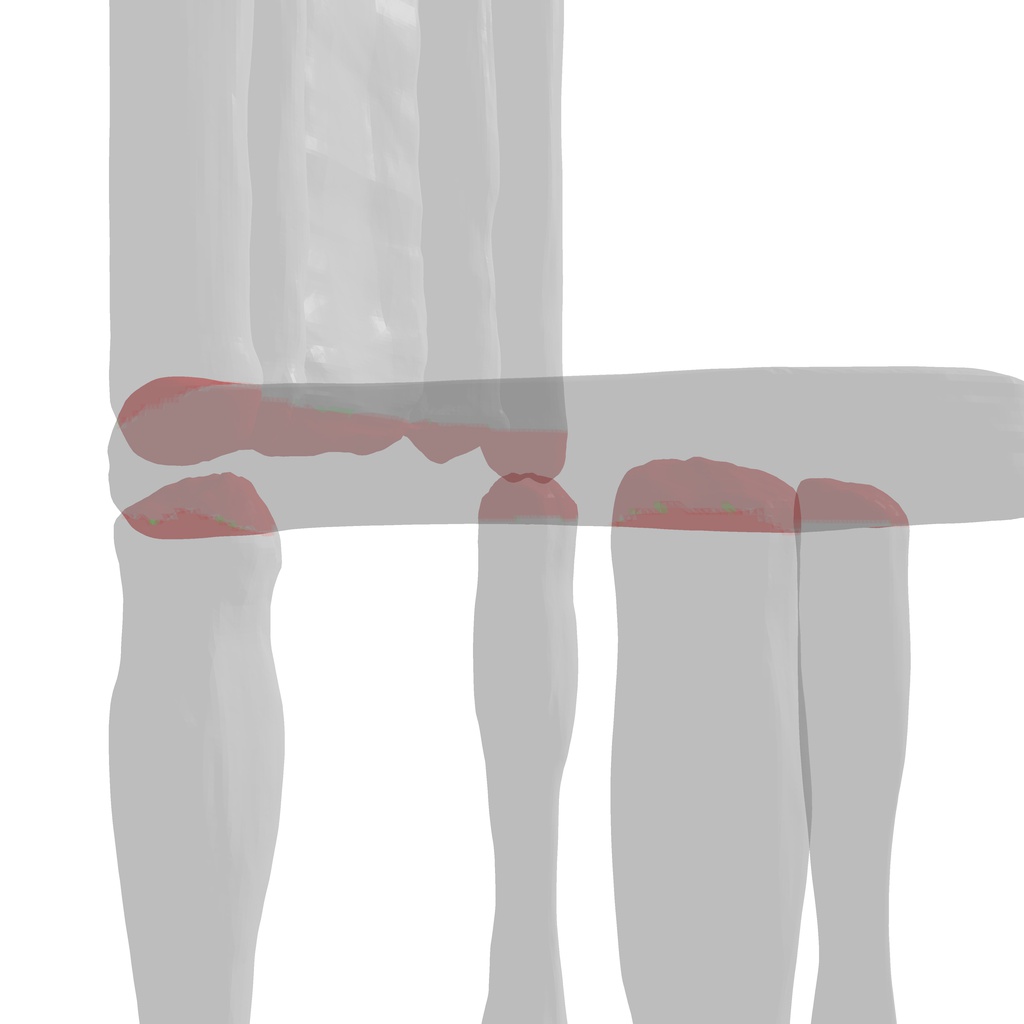} \\
    \adjustbox{valign=c}{\shortstack[c]{\scriptsize S2MDF\\[-1pt]\scriptsize (Shift-All)\\[-1pt]\scriptsize - Train}} &
    \chairresimage{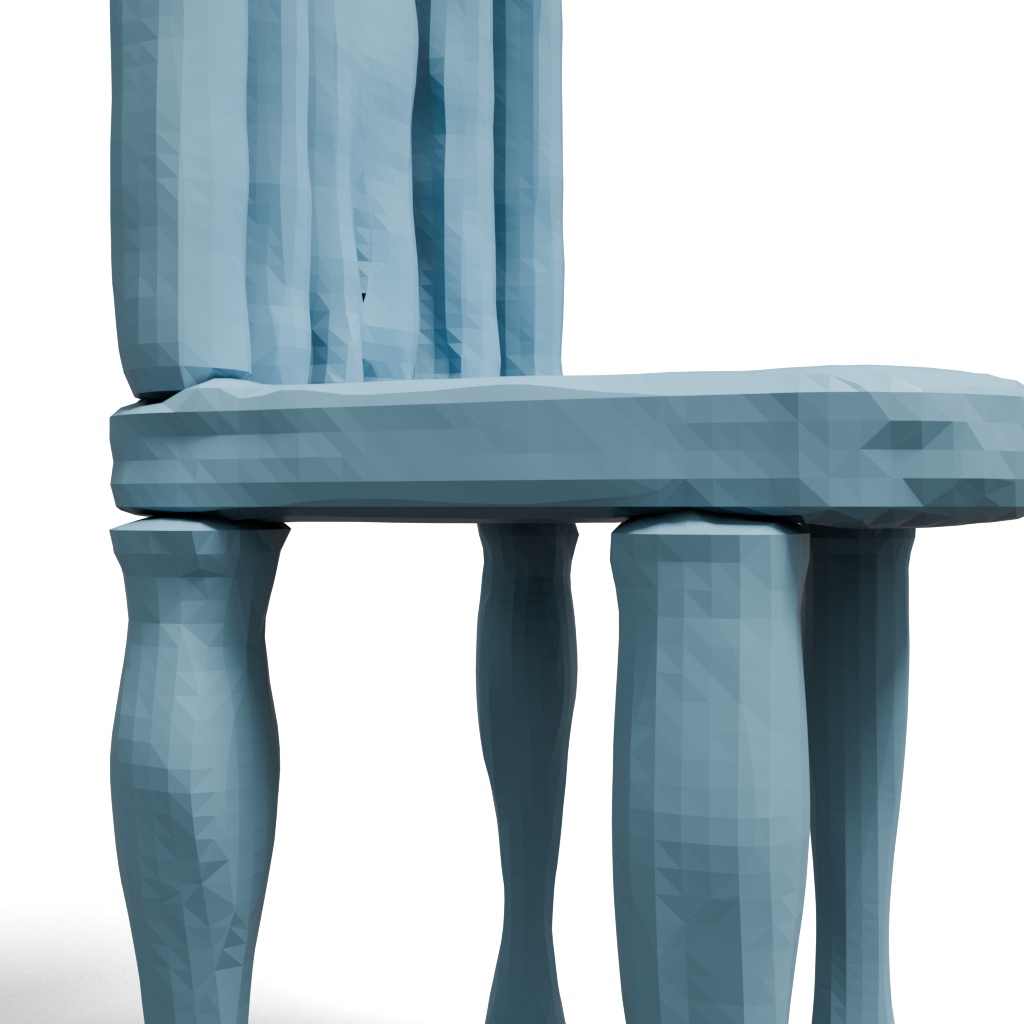} &
    \chairresimage{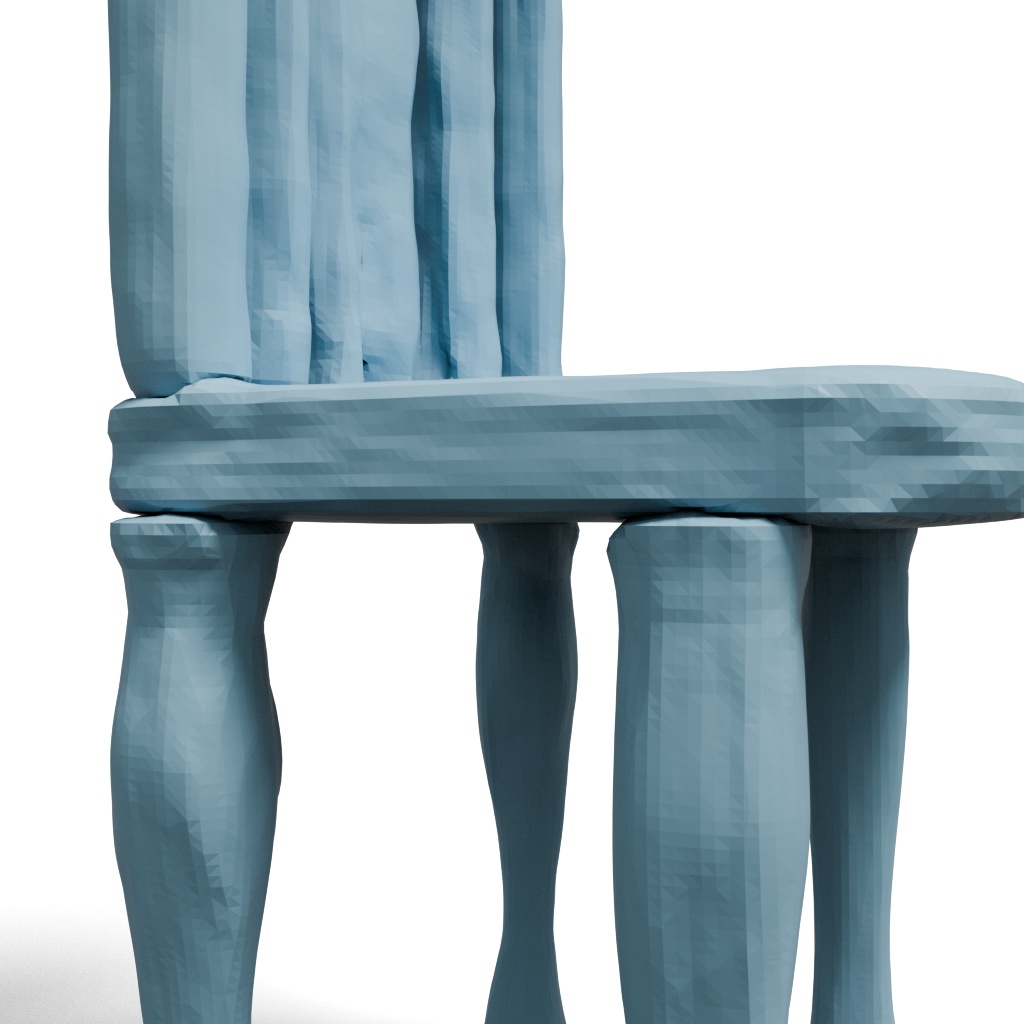} &
    \chairresimage{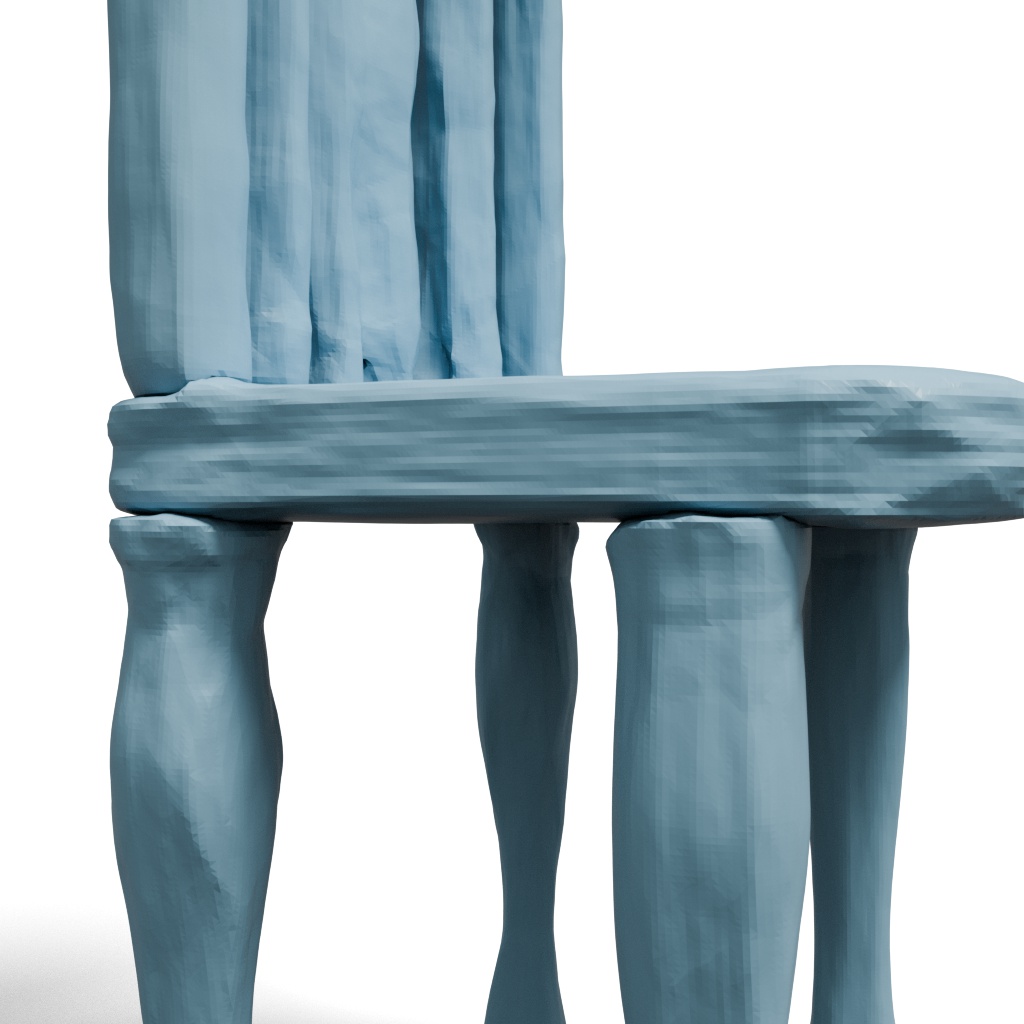} &
    \chairresimage{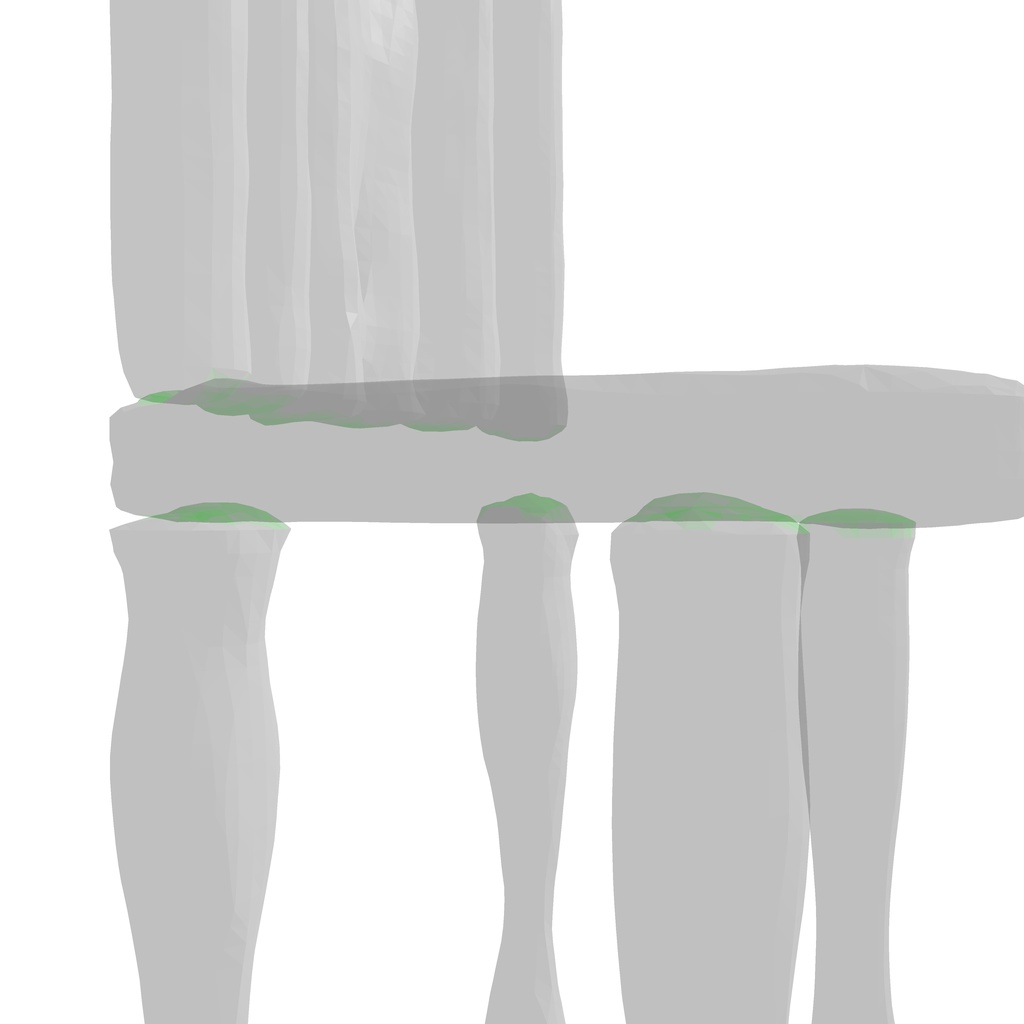} &
    \chairresimage{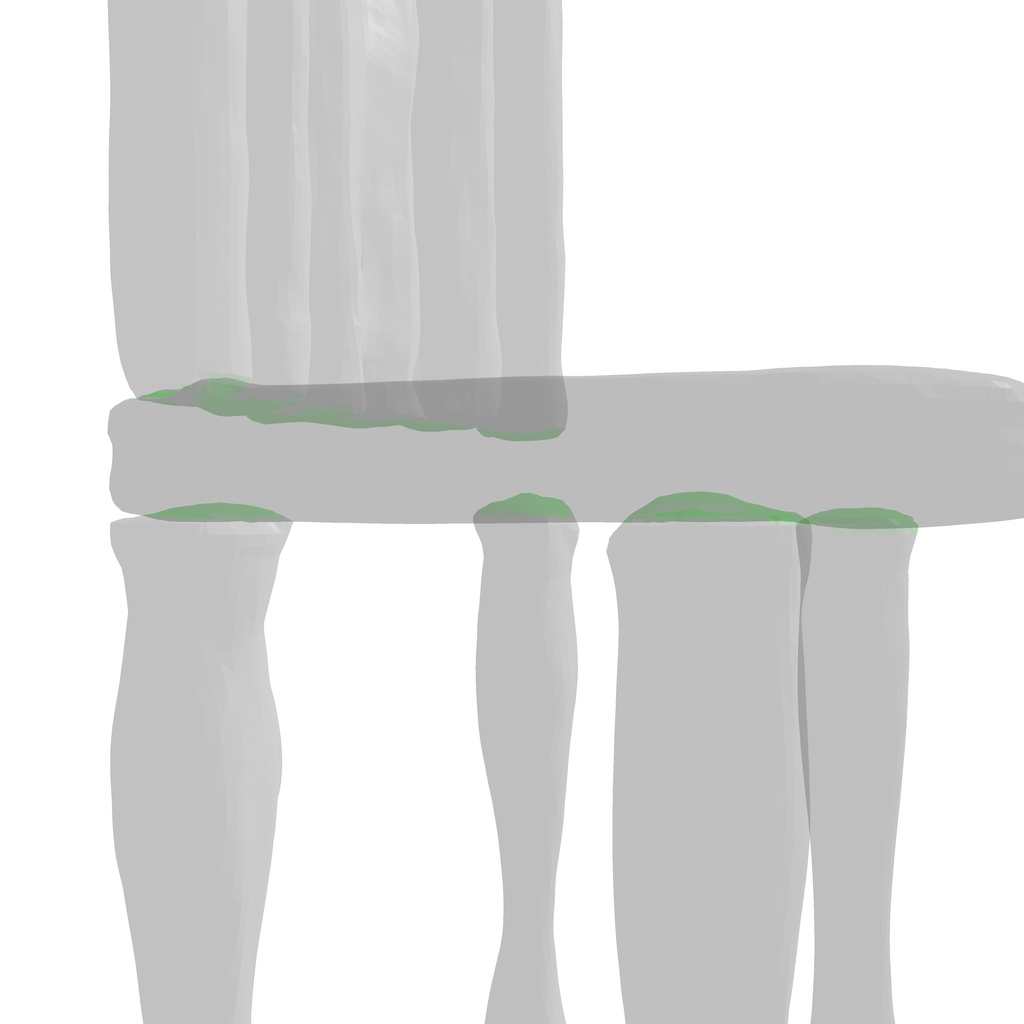} &
    \chairresimage{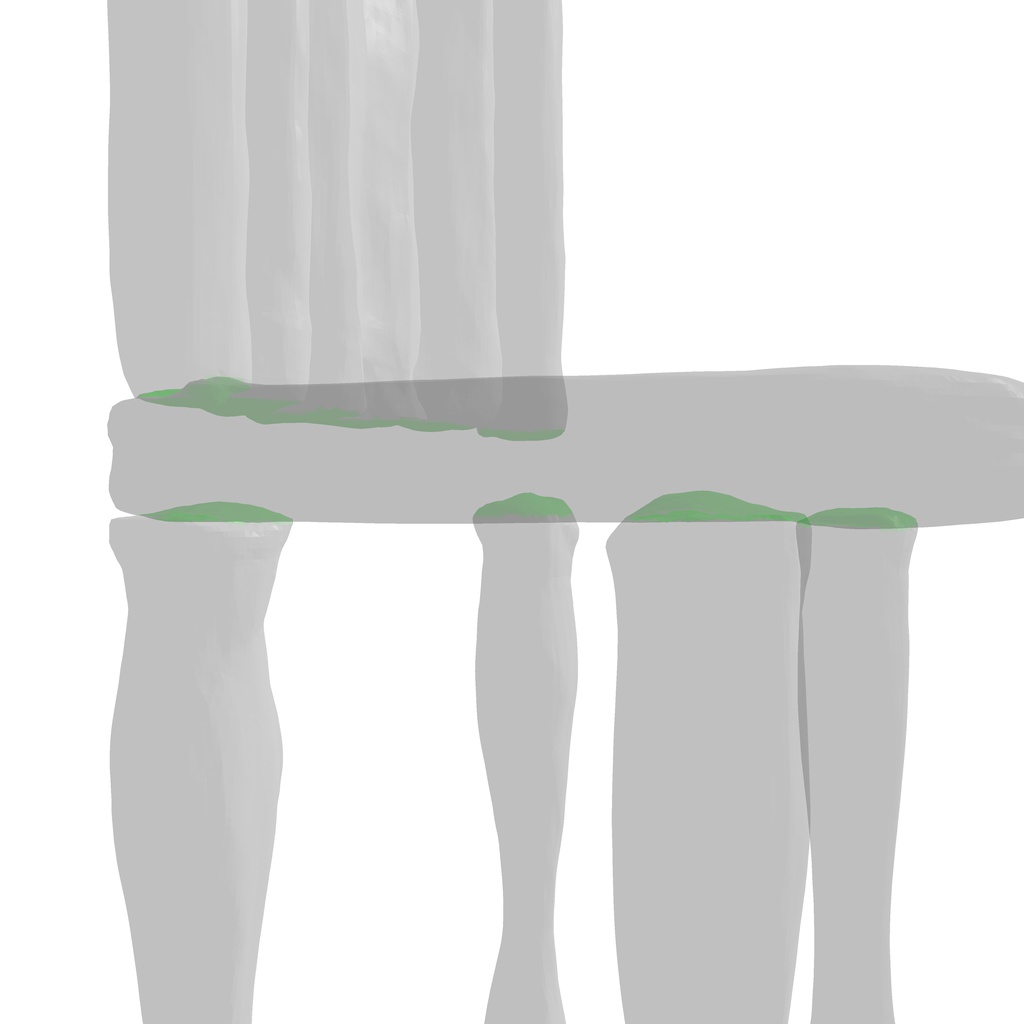} \\[2pt]
    & \multicolumn{3}{c}{\subcaptionbox{Renderings\label{fig:chairs_resolution_comparison_renders}}[0.45\textwidth]{\strut}}
    & \multicolumn{3}{c}{\subcaptionbox{Intersections\label{fig:chairs_resolution_comparison_intersections}}[0.45\textwidth]{\strut}}
  \end{tabular}
  \caption{PartSDF on PartNet chairs at increasing meshing resolutions. Intersecting surfaces are shown in red, while touching surfaces are shown in green.}
  \label{fig:chairs_resolution_comparison}
\end{figure*}

In a few situations, we observe that the application of \acro{} can introduce artifacts in the resulting meshes. This is quantitatively hinted at by the slightly worse MNC scores in the PartSDF chairs experiments (\cref{tab:partsdf,tab:partsdf_chair100}). We noticed that, when \acro{} is applied, some of the junctions between parts are not as straight as in the vanilla PartSDF meshes. This effect is visible in the junctions between the chair legs and the seat in \cref{fig:chairs_resolution_comparison}, where we show the meshes of a chair at increasing meshing resolutions. The issue appears because the GT parts of the chairs are not watertight: they are missing the junction surfaces, which means that PartSDF lacks supervision in the most delicate regions. This is not an issue in vanilla PartSDF because the legs become elongated, heavily intersecting with the seat (\cref{fig:chairs_resolution_comparison_intersections}), but it can create small artifacts when \acro{} is applied. We did not observe this issue in the Mixers dataset, as well as in the other experiments in the main paper, where the GT parts are watertight and the junctions are well defined.


\subsubsection{Ablating the non-intersection loss}
\label{sec:nonintloss_partsdf}

\begin{table}[ht]
\centering
\resizebox{\textwidth}{!}{%
  \begin{tabular}{c|ccccc}
    & MCD $(\times 10^{-4})$ $\downarrow$ & MNC $\uparrow$ & F1 (0.001) $\uparrow$ & IoU $\uparrow$ & IV $\downarrow$ \\
    \hline
    Vanilla PartSDF & $4.331 \pm 0.690$ & $95.119 \pm 0.015$ & $96.731 \pm 0.044$ & $85.630 \pm 0.101$ & $(1.876 \pm 0.062) \times 10^{-3}$ \\
    Vanilla PartSDF w/o int. loss & $40.891 \pm 28.217$ & $94.473 \pm 0.158$ & $92.493 \pm 0.928$ & $77.759 \pm 0.898$ & $(8.746 \pm 3.431) \times 10^{-3}$ \\
    \hline
    \acro{} (Shift-All) - PP & $1.973 \pm 0.166$ & $94.824 \pm 0.008$ & $98.005 \pm 0.022$ & $89.173 \pm 0.104$ & $(5.194 \pm 1.040) \times 10^{-9}$ \\
    \acro{} (Shift-All) - PP w/o int. loss & $3.735 \pm 0.459$ & $94.235 \pm 0.069$ & $96.579 \pm 0.187$ & $86.752 \pm 0.151$ & $(2.868 \pm 0.407) \times 10^{-8}$ \\
    \hline
    \acro{} (Shift-All) - Optim & $2.616 \pm 0.576$ & $94.879 \pm 0.039$ & $98.116 \pm 0.021$ & $89.664 \pm 0.048$ & $(3.194 \pm 0.566) \times 10^{-9}$ \\
    \acro{} (Shift-All) - Optim w/o int. loss & $3.683 \pm 2.270$ & $94.813 \pm 0.028$ & $98.023 \pm 0.084$ & $89.381 \pm 0.156$ & $(3.304 \pm 0.208) \times 10^{-9}$ \\
    \hline
    \acro{} (Shift-All) - Train & $2.120 \pm 0.117$ & $94.823 \pm 0.032$ & $98.039 \pm 0.017$ & $89.297 \pm 0.262$ & $(5.665 \pm 2.370) \times 10^{-9}$ \\
    \acro{} (Shift-All) - Train w/o int. loss & $13.040 \pm 1.498$ & $94.737 \pm 0.100$ & $97.939 \pm 0.046$ & $88.957 \pm 0.156$ & $(8.697 \pm 0.612) \times 10^{-9}$ \\
 \end{tabular}%
}
\caption{Non-intersection loss ablation: shape reconstruction metrics for PartSDF on 100 chairs from the test set. Mean and standard deviation across 3 runs.}
\label{tab:partsdf_chair100_nonintloss}
\end{table}
PartSDF includes a non-intersection loss term in its training objective, which is designed to penalize intersections between parts. In \cref{tab:partsdf_chair100_nonintloss}, we present quantitative results for the same PartSDF~\cite{Talabot25} experiment as in \cref{tab:partsdf}, on PartNet~\cite{Mo19} chairs, additionally comparing the application of \acro{} with and without the non-intersection loss term. We observe that, while the non-intersection loss term does improve the results of PartSDF, it is not sufficient to prevent intersections between parts, and applying \acro{} provides a significant improvement over both cases. We also observe that, due to the non-watertight GT parts of the chairs dataset, the MCD scores become significantly worse when the non-intersection loss is removed. Our method is able to remove intersections nonetheless, but it can benefit from the non-intersection loss to improve the overall quality of the reconstruction.


\newtcolorbox{claimbox}{
    colback=white,      
    colframe=black,     
    boxrule=0.5pt,      
    rounded corners,
    arc=2pt,            
    boxsep=5pt,
    left=10pt,
    right=10pt,
    top=5pt,
    bottom=5pt
}

\section{Proofs}

\subsection{\cref{thm:constraint_equivalence} (Constraint Equivalence for SDFs)}
\label{app:constraint-equivalence}
\nt{
\begin{claimbox}
	A vector-valued $SDF = [SDF_1, \ldots , SDF_K] : \mathbb{R}^3 \to \mathbb{R}^K$ whose components are exact SDFs as in \cref{eq:sdf} satisfies constraint $\mint(\bu) \ge 0$ at every point $\mathbf{p} \in \mathbb{R}^3$, with $\bu = SDF(\bp)$, if and only if it satisfies $\mino(\bu) + \mint(\bu) \ge 0$ at every point $\mathbf{p} \in \mathbb{R}^3$.
\end{claimbox}

\begin{proof}
	Through the second constraint and the definition of $\min^{(n)}$, we have both $\mint(\bu) \ge -\mino(\bu)$ and $\mint(\bu) \ge \mino(\bu)$  and, therefore, $\mint(\bu) \ge 0$ at every point $\bp\in\real^3$.
	
	For the inverse implication, we consider an arbitrary point $\bp\in\real^3$. If $\bp$ lies outside all objects, the second condition is trivially satisfied as all SDFs are positive. If it is inside an object, we can use the Lipschitz continuity of SDFs to show that the second constraint must also be satisfied.
	
	For this, we first define some notation for point $\bp$ with SDF vector $\bu$. Let:
	\begin{itemize}
		\item $a$ and $b$ be the indices of the objects with the smallest and second smallest SDF, respectively.
		\item $\by_{a}$ and $\by_{b}$ be points on the surfaces of these objects that are closest to $\bp$.
		\item $s_{a} : \real^3\to\real$ and $s_{b} : \real^3\to\real$ be the signed distance functions of these objects.
	\end{itemize}
	
	Based on the definition of SDFs of \Cref{eq:sdf} and using the first constraint, we can write
	\begin{align*}
		\mino(\bu) &= s_{a}\left(\bp\right) = - \left\lVert \bp - \by_{a} \right\rVert \; \mathrm{and} \\
		\mint(\bu) &= s_{b}\left(\bp\right) = \left\lVert \bp - \by_{b} \right\rVert \; .
	\end{align*}
	
	Using that SDFs are $1$-Lipschitz continuous,\NT{cite?} we have
	\begin{align*}
		&s_a(\by_b) - s_a(\bp) \le \left|s_a(\by_b) - s_a(\bp)\right| \le \left\lVert \by_b - \bp \right\rVert \\
		\Leftrightarrow \;\; &s_a(\by_b) \le s_a(\bp) + \left\lVert \by_b - \bp \right\rVert\\
		&\phantom{s_a(\by_b)} = s_a(\bp)+ s_b(\bp)\\
		&\phantom{s_a(\by_b)} = \mino(\bu) + \mint(\bu)\\
	\end{align*}
	Suppose now that the second constraint is not satisfied at $\bp$. Then
	$$
	s_a(\by_b) \le \mino(\bu) + \mint(\bu) < 0 \; ,
	$$
	which means that $\by_b$, a point on the surface of object $b$, is strictly inside $a$. Thus, in its vicinity, there must be a point that is both within object $a$ and $b$, which is impossible as it violates the first constraint.
	Therefore, there cannot be any point $\bp$ such that the second constraint is not satisfied.
\end{proof}
}


\subsection{Sum of two mins is equivalent to pairwise sums}
\label{app:pairwise-proof}

\begin{claimbox}
Consider an arbitrary vector $\vd \in \mathbb{R}^K$ and its elements $d_i,\ 1 \le i \le K$. The inequality $\mint(\vd) + \mino(\vd) \ge 0$ is satisfied if and only if $d_i + d_j \ge 0,\ 1 \le i < j \le K$.  
\end{claimbox}

\begin{proof}
Because the $\mint$, $\mino$ and the set of all pairwise sums are invariant under permutation of elements, we can assume without loss of generality that the elements of $\vd$ are sorted in ascending order $\vd = (d_1, d_2, \dots, d_K)$, $d_1 \le d_2 \le \dots \le d_K$. 

\textbf{Forward Proof}. We are given that $\mino(\vd) + \mint(\vd) \ge 0$, i.e. $d_1 + d_2 \ge 0$. 
\begin{itemize}
    \item Given $i \ge 1$, it follows that $d_i \ge d_1$. 
    \item As $j > i \ge 1$, $j \ge 2$. It follows that $d_j \ge d_2$. 
\end{itemize}
We can combine the two inequalities to get $d_i + d_j \ge d_1 + d_2$. This completes the forward proof. 

\textbf{Backward Proof}. We are given that $d_i + d_j \ge 0,\ 1 \le i < j \le K$. We can set $i = 1$ and $j = 2$, meaning $d_1 + d_2 \ge 0$. Because $\mino(\vd) = d_1$ and $\mint(\vd) = d_2$, it follows that $\mint(\vd) + \mino(\vd) \ge 0$. This completes the backward proof.
\end{proof}

\subsection{\cref{th:linearity}: Linear Interpolation under MDF constraint}
\label{app:linterp-proof}

\begin{claimbox}
Given two points $\mathbf{p_1}, \mathbf{p_2} \in \mathbb{R}^3$ and an arbitrary vector-valued function $f : \mathbb{R}^3 \to \mathbb{R}^K$ evaluated at those points, $\vd_1 = f(\mathbf{p_1})$ and $\vd_2 = f(\mathbf{p_2})$ that satisfy constraint \ref{eq:mdf_constraint}, i.e. $\mino(\vd) + \mint(\vd) \ge 0$, then the constraint is also satisfied for every point along the line segment connecting $\mathbf{p_1}$ and $\mathbf{p_2}$, where intermediate vectors are obtained via linear interpolation as $\vd_{\alpha} = \alpha \vd_1 + (1 - \alpha) \vd_2$ for $\alpha \in [0, 1]$.
\end{claimbox}

\begin{proof}
For notational convenience, let $S_2(\vd) = \mino(\vd) + \mint(\vd)$, allowing us to state the constraint as $S_2(\vd) \ge 0$. Example 3.6, p. 80 in \cite{Boyd04} proves that the function $L_r(\vd)$, representing the sum of the $r$ largest components of a vector, is convex. The sum of the $r$ smallest components can be expressed as $S_r(\vd) = -L_r(-\vd)$. Since $L_r$ is convex, $L_r(-\vd)$ is convex, and its negation $-L_r(-\vd)$ is concave. Because $S_2$ is a special case of $S_r$, it is also concave. Using the definition of concavity for $\alpha \in [0, 1]$:
\begin{equation}
S_2(\vd_\alpha) = S_2(\alpha \vd_1 + (1 - \alpha) \vd_2) \ge \alpha S_2(\vd_1) + (1 - \alpha) S_2(\vd_2)
\end{equation}
Since $S_2(\vd_1) \ge 0$ and $S_2(\vd_2) \ge 0$, and both $\alpha$ and $(1 - \alpha)$ are non-negative, it follows that $\alpha S_2(\vd_1) + (1 - \alpha) S_2(\vd_2) \ge 0$. Thus, $S_2(\vd_\alpha) \ge 0$, guaranteeing that $\vd_\alpha$ always satisfies the MDF constraint \cref{eq:mdf_constraint}.
\end{proof}

\subsection{Constraint \ref{eq:og_constraint} does not allow linear MDF interpolation}
\label{app:og_constraint_weaker}

\begin{claimbox}
Given two points $\mathbf{p_1}, \mathbf{p_2} \in \mathbb{R}^3$ and an arbitrary vector-valued function $f : \mathbb{R}^3 \to \mathbb{R}^K$ evaluated at those points, $\vd_1 = f(\mathbf{p_1})$ and $\vd_2 = f(\mathbf{p_2})$ that satisfy constraint \ref{eq:og_constraint}, i.e. $\mint(\vd) \ge 0$, then the constraint is not guaranteed to be satisfied for every point along the line segment connecting $\mathbf{p_1}$ and $\mathbf{p_2}$, where intermediate vectors are obtained via linear interpolation as $\vd_{\alpha} = \alpha \vd_1 + (1 - \alpha) \vd_2$ for $\alpha \in [0, 1]$.
\end{claimbox}

\begin{proof}
\begin{figure}[htbp]
    \centering
    \begin{tikzpicture}[>=stealth, font=\Large, scale=0.68, transform shape]
    \path[use as bounding box] (-3, -3) rectangle (7.05, 4);

    \draw[line width=0.3pt, cyan!60!black, fill=cyan!5!white] (0,0) circle (3);


    \coordinate (p0) at (0, 0);
    \node[circle, draw=gray, line width=1pt, fill=white, inner sep=3pt, label=below:{$\vp_1$}] at (p0) {};
    \node at (0, -1.0) {$\vd_1 = [-3, 2]$};
    \coordinate (p_alpha) at (1.61, 1.61);
    \coordinate (p2) at ({3/sqrt(2)}, {3/sqrt(2)});

    
    \draw[line width=1.2pt, dashed, red!80!black] 
        (-1, 4) 
        to[out=45, in=160]  (0.5, 3)       
        to[out=-20, in=135] (1.41, 1.41)   
        to[out=-45, in=135] (4, 1);        
    \node[text=red!80!black] at (4.6, 0.75) {implied};
    \node[text=red!80!black] at (4.6, 0.4) {surface};
    \node[circle, draw=gray, line width=1pt, fill=white, inner sep=3pt, label={[yshift=-14pt]above right:{$\vp_\alpha$}}] at (p_alpha) {};
    \node at (5.01, 2.61) {$\vd_2 = [0, -1]$};
    \node at (4.71, 1.81) {$\vd_{\alpha} = [-0.9, -0.1]$};

    
    \filldraw[black] (0, 0) circle (1.5pt) ;
    \filldraw[black] (p_alpha) circle (1.5pt) ;
    \draw[black] (0,0) -- node[above, sloped, text=black, font=\small] {$d=2.1$} (p_alpha);

    \node[circle, draw=gray, line width=1pt, fill=white, inner sep=3pt, label=above right:{$\vp_2$}] at (p2) {};

    \filldraw[black] (p2) circle (1.5pt) ;
    \draw[black] (p_alpha) -- (p2);

    \coordinate (v0) at (6,1.5);
    \coordinate (vprime_circle) at (7.5,1.5);


    \end{tikzpicture}
    \caption{When dealing with general functions (e.g. imperfect SDFs), the constraint \ref{eq:og_constraint} does not guarantee that a linear interpolation of two points $\vp_1$ and $\vp_2$ that satisfy the constraint will also satisfy the constraint. In fact, there exists a point $\vp_\alpha$ on the line segment between $\vp_1$ and $\vp_2$ that violates the constraint.}
    \label{fig:supp_min2_counterexample}
\end{figure}

We prove it by counterexample. Consider the two points $\vp_1$ and $\vp_2$ in \cref{fig:supp_min2_counterexample} with $\vd_1 = [-3, 2]$ and $\vd_2 = [0, -1]$. Both of these vectors satisfy the constraint $\mint(\vd) \ge 0$. However, the linearly interpolated point $\vp_\alpha$ between them has $\vd_\alpha = [-0.9, -0.1]$, which violates the constraint $\mint(\vd_\alpha) \ge 0$. This counterexample demonstrates that the constraint \ref{eq:og_constraint} does not guarantee that a linear interpolation of two points satisfying the constraint will also satisfy the constraint when dealing with general functions (e.g. imperfect SDFs).
\end{proof}

\subsection{Line Crossing}
\label{app:crossing-proof}

\begin{claimbox}
    Given two points $\mathbf{p_1}, \mathbf{p_2} \in \mathbb{R}^3$ and an arbitrary vector-valued function $f : \mathbb{R}^3 \to \mathbb{R}^K$ evaluated at those points, $\vd_1 = f(\mathbf{p_1})$ and $\vd_2 = f(\mathbf{p_2})$, if they both satisfy the MDF constraint (\cref{eq:mdf_constraint}), i.e. $\mino (\vd) + \mint (\vd) \ge 0$, and an edge between them is linearly interpolated by a meshing algorithm for surface reconstruction, we claim that there are at most two object surfaces between these points and that these surfaces do not cross.
\end{claimbox}

\begin{proof}
We already know from \cref{thm:constraint_equivalence}\DM{Ref min2 proof? Also needs a small figure.} that $\mino (\vd) + \mint (\vd) \ge 0$ implies $\mint(\vd) \ge 0$, which implies a vector $\vd$ satisfying this constraint can have at most one negative value. 

\fs{Let us assume for now that neither of the two vectors contain zero values.} Let us perform a case analysis on the possibilities of the signs of the values of the two vectors:
\begin{itemize}
    \item Neither $\vd_1$ nor $\vd_2$ contain negative values. In this case, there are no object surfaces between them. 
    \item One of them has a negative value. In this case, there is only one object surface, and we are outside every other object. 
    \item Both of them have a negative value, and these negative values are in the same index $i$. In this case, both points are inside the same object and there will be no surface between them. 
    \item Both of them have a negative value, and these are in different indices, $i$ and $j$. Then, there will be two zero-crossings and two object surfaces. 
\end{itemize}

The first three cases have no risk of having two surfaces intersect, so our interest is in the final case. Let $\vd_1 = \vp$ and $\vd_2 = \vq$ for notational convenience, and $i$, $j$ be the minimum indices in $\vp$ and $\vq$, respectively. Then, $p_i < 0$ and $q_j < 0$. We know from Appendix~\ref{app:pairwise-proof} that satisfaction of the MDF constraint implies all pairwise sums in a vector are positive. Thus, we have that:
\begin{align}
    p_i < 0,\ p_i + p_j \ge 0 &\implies p_j > 0 \label{eq:supplc-i1} \\
    q_j < 0,\ q_i + q_j \ge 0 &\implies q_i > 0 \label{eq:supplc-i2}
\end{align}
Surfaces along the line segment will appear at zero-crossings. If we parameterize points along this segment with $\alpha \in [0, 1]$, assuming $\alpha = 0$ at the point where $\vp$ is evaluated, and $\alpha = 1$ for $\vq$, we can find the zero-crossing for both object $i$ and $j$:
\begin{align}
    \alpha_{i} &= \frac{-p_i}{q_i - p_i} \\
    \alpha_{j} &= \frac{-p_j}{q_j - p_j}
\end{align}
We will now derive the relationship between $\alpha_i$ and $\alpha_j$ based on \cref{eq:supplc-i1} and \cref{eq:supplc-i2}. 
Given $p_j > 0$, we divide $p_i + p_j \ge 0$ to obtain$\dfrac{p_i}{p_j} + 1 \ge 0$. Rearranging this, and with the knowledge that $p_i < 0$, we obtain the bounds:
\begin{equation}
    1 \ge -\frac{p_i}{p_j} > 0
\end{equation}
With a similar approach, dividing $q_i + q_j \ge 0$ by $q_i > 0$, we can also obtain:
\begin{equation}
    1 \ge -\frac{q_j}{q_i} > 0
\end{equation}
Multiplying both inequalities:
\begin{equation}
    1 \ge \frac{p_i q_j}{p_j q_i} > 0
\end{equation}
Since $p_j > 0$ and $q_i > 0$, $p_j q_i > 0$ and we can multiply both sides with it, dropping the lower bound:
\begin{equation}
    p_j q_i \ge p_i q_j
\end{equation}
Now, subtract $p_i p_j$ from both sides and factor:
\begin{equation}
    p_j (q_i - p_i) \ge p_i (q_j - p_j) 
\end{equation}
We know $p_i < 0$ and $q_i > 0$, so $q_i - p_i > 0$, we divide both sides:
\begin{equation}
    p_j \ge (q_j - p_j) \frac{p_i}{q_i - p_i} 
\end{equation}
Similarly, $p_j > 0$ and $q_j < 0$, so $q_j - p_j < 0$, divide both sides and flip the inequality:
\begin{equation}
    \frac{p_j}{q_j - p_j} \le \frac{p_i}{q_i - p_i}
\end{equation}
Multiply both sides by -1:
\begin{equation}
    -\frac{p_j}{q_j - p_j} \ge -\frac{p_i}{q_i - p_i}
\end{equation}
And we finally obtain $\alpha_i$ and $\alpha_j$: 
\begin{equation}
    \alpha_j \ge \alpha_i \implies \alpha_i \le \alpha_j
\end{equation}
This proves that the surface of the object $j$ will always be closer to the point where $\vq$ is evaluated than object $i$. The surfaces can at most touch each other, or will be disjoint. \fs{This completes our proof in the case where neither of the vectors contain zero values. If they do contain zero values, the following cases can happen:
\begin{itemize}
    \item One of the vectors $\vd_1, \vd_2$ has one or more zero values, let us assume $\vd_1$, and the other vector does not. A surface point will be placed exactly at the point where $\vd_1$ is evaluated. Due to \cref{eq:mdf_constraint}, the zero value in $\vd_1$ implies the absence of negative values in $\vd_1$, which in turn implies that its point is on the surface of one or more objects but outside every other object.
    \item Both vectors have one or more zero values in the same indices. This implies that one or more surfaces lie exactly on the line segment, which does not cause a crossing.
    \item Both vectors have one or more zero values, and these zero values are potentially in different indices. This implies the presence of multiple surface points at the points where the vectors are evaluated, however neither of the vectors can have negative values, due to \cref{eq:mdf_constraint}, which implies that both points are on the surface of one or more objects but outside every other object.
\end{itemize}
}
\end{proof}

\subsection{QP analytical solution for the two-object case}
\label{app:analytical-2obj-proof}

\begin{claimbox}
Given the optimization problem:
\begin{equation}
{\rm minimize} \| \vd - \vu \|^2  \quad \mbox{subject to} \quad {\mino (\vd) + \mint (\vd)}{\ge 0 } \; . 
\end{equation}
in two dimensions, i.e. $\vd \in \mathbb{R}^2$, its optimal solution is equivalent to Shift-All, i.e. not changing $\vu$ in the satisfied case, and shifting both variables by $-(\mino (\vu) + \mint (\vu))/2$ in the unsatisfied case.
\end{claimbox}

\begin{proof}
Let us state the common way to formulate an optimization problem with inequality constraints:
\begin{equation}
{\rm minimize}\ f(\vd)  \quad \mbox{subject to} \quad g(\vd) \le 0 \; .
\end{equation}

Let $\vu = (u_1, u_2)$ and $\vd = (d_1, d_2)$. As there are only two elements, $\mino (\vd) + \mint(\vd) = d_1 + d_2 \ge 0 \implies -d_1 - d_2 \le 0$. 

Then, for our case, $f(\vd) = \| \vd - \vu \|^2 = (d_1 - u_1)^2 + (d_2 - u_2)^2$ and $g(\vd) = -d_1 - d_2$. Let's define the Lagrangian function:

\begin{equation}
\mathcal{L}(\vd, \lambda) = f(\vd) + \lambda g(\vd) = (d_1 - u_1)^2 + (d_2 - u_2)^2 - \lambda (d_1 + d_2)
\end{equation}

We now write down the KKT conditions:
\begin{enumerate}
\item \textbf{Stationarity:} $\nabla f(\vd^*) + \lambda \nabla g(\vd^*) = 0$.
\item \textbf{Primal feasibility:} $g(\vd^*) \le 0$.
\item \textbf{Dual feasibility:} $\lambda \ge 0$.
\item \textbf{Complementary slackness:} $\lambda g(\vd^*) = 0$.
\end{enumerate}

Before continuing our solution, we derive the gradients:
\begin{equation}
\nabla f(\vd) = 2
\begin{bmatrix}
d_1 - u_1 \\
d_2 - u_2
\end{bmatrix},
\quad
\nabla g(\vd) =
\begin{bmatrix}
-1 \\
-1
\end{bmatrix}
\end{equation}

\textbf{Inactive case}. Assume $g(\vd^*) < 0$, meaning the constraint is already satisfied. Then, by stationarity, $\nabla f(\vd^*) = 0$:
\begin{equation}
\nabla f(\vd^*) = 2
\begin{bmatrix}
d_1^* - u_1 \\
d_2^* - u_2
\end{bmatrix}
=
\begin{bmatrix}
    0 \\ 
    0
\end{bmatrix}
\implies
\begin{matrix}
d_1^* = u_1 \\
d_2^* = u_2 
\end{matrix},
\vd^* = \vu
\end{equation}
By complementary slackness, $\lambda = 0$. Let us check the feasibility:
\begin{itemize}
    \item \textbf{Primal:} $g(\vd^*) = g(\vu) = -u_1 -u_2 < 0 \implies u_1 + u_2 > 0$.
    \item \textbf{Dual:} $\lambda = 0 \ge 0$, satisfied. 
\end{itemize}
Feasibilities are satisfied and the inactive solution is valid when $u_1 + u_2 > 0$. The solution intuitively makes sense, as our projection does not need to modify the initial value $\vu$ in case the constraint is already satisfied. 

\textbf{Active case}. Assume $g(\vd^*) = 0$, meaning the solution is at the constraint boundary. By complementary slackness: 
\begin{equation}
    \lambda g(\vd^*) = -\lambda (d_1^* + d_2^*) = 0 \implies d_2^* = -d_1^*
\end{equation}

By stationarity, $\nabla f(\vd^*) + \lambda \nabla g(\vd^*) = 0$, and substituting the above $d_2^* = -d_1^*$:
\begin{equation}
2
\begin{bmatrix}
d_1^* - u_1 \\
-d_1^* - u_2
\end{bmatrix}
-
\lambda
\begin{bmatrix}
1 \\
1
\end{bmatrix}
=
\begin{bmatrix}
    0 \\ 
    0
\end{bmatrix}
\implies
\begin{matrix}
d_1^* - u_1 = \lambda / 2 \\
-d_1^* - u_2 = \lambda / 2
\end{matrix}
\end{equation}
Summing the two equations, we get $\lambda = -(u_1 + u_2)$. 
We can plug this result back into the equations to derive $\vd^*$:
\begin{align}
    d_1^* &= u_1 + \frac{\lambda}{2} = u_1 - \frac{u_1 + u_2}{2} \\
    d_2^* &= u_2 + \frac{\lambda}{2} = u_2 - \frac{u_1 + u_2}{2}
\end{align}
Let's consider the feasibilities:
\begin{itemize}
    \item \textbf{Primal:} $g(\vd^*) = -d_1^* -d_2^* = -u_1 + \frac{u_1 + u_2}{2} - u_2 + \frac{u_1 + u_2}{2} = 0 \le 0$, satisfied.
    \item \textbf{Dual:} $\lambda = - u_1 - u_2 \ge 0 \implies u_1 + u_2 \le 0$. 
\end{itemize}
Feasibilities are satisfied and the active solution is valid when $u_1 + u_2 \le 0$, meaning both values need to be shifted by $-\dfrac{u_1 + u_2}{2}$ when the constraint is not satisfied on $\vu$. 

Because our objective function $f(\mathbf{d}) = \|\mathbf{d} - \mathbf{u}\|^2$ is strictly convex and the inequality constraint $g(\mathbf{d}) = -d_1 -d_2 \le 0$ is affine, the optimization problem is strictly convex. Under these conditions, the KKT conditions are both necessary and sufficient for the $\mathbf{d}^*$ we derived to be the \textbf{global optimum} \citep[Section 5.5.3, p. 244]{Boyd04}.

This analytical solution leads to the exact same formulations as our Shift-All. 
\end{proof}

\fs{
\subsection{Solution to the modified QP problem}
\label{app:analytical-qpmod-proof}

\begin{claimbox}
Let $\vu \in \mathbb{R}^K$ be the $K$-dimensional SDF vector at a point in space, the following optimization problem:
\begin{mini!}|l|
    {\vd \in \mathbb{R}^K}{\| \vd - \vu \|^2}
    {\label{eq:qp-nonlin}}{}
    \addConstraint{\mino (\vd) + \mint (\vd)}{\ge 0, }
    \addConstraint{\vd_i - \vd_j = \vu_i - \vu_j,}{\forall i, j \in [K], i < j}
\end{mini!}
has Shift-All as its optimal solution, i.e. not changing $\vu$ in the satisfied case, and shifting all variables by $-(\mino (\vu) + \mint (\vu))/2$ in the unsatisfied case.
\end{claimbox}
\begin{proof}
    The unconstrained minimizer of the objective function is $\vd = \vu$. If $\vu$ satisfies the MDF constraint, then $\vd = \vu$ is also the optimal solution to the constrained problem.

    If $\vu$ does not satisfy the MDF constraint, we notice that the second constraint forces solutions to be of the form $\vd = \vu + c \mathbf{1}$, where $\mathbf{1}$ is the vector of all ones. For simplicity, we define $\vd_m = \mino (\vd)$ and $\vd_n = \mint (\vd)$. Then, we have that:
\begin{align}
    \vd_m &= \vu_m + c \\
    \vd_n &= \vu_n + c
\end{align}
From the first constraint of the optimization problem, it follows that:
\begin{align}
    \vd_m + \vd_n = \vu_m + \vu_n + 2c \ge 0 \implies c \ge -\frac{\vu_m + \vu_n}{2}
\end{align}
The objective function is convex so, if its global minimum is not feasible, the optimal solution $\vd^*$ will be at the boundary of the feasible region. In fact, if it wasn't at the boundary, then it would be in the interior of the feasible region. In such case, $\nabla\| \vd - \vu \|^2 = 2(\vd - \vu) = 0$ at $\vd^*$, otherwise it would not be optimal, but this implies that $\vd^* = \vu$, which is not feasible. Therefore, $\vd^*$ must be at the boundary, meaning $c = -\frac{\vu_m + \vu_n}{2}$.

In short, if $\vu$ satisfies the MDF constraint, then $\vd^* = \vu$ is the optimal solution. If $\vu$ does not satisfy the MDF constraint, then $\vd^* = \vu - \frac{\mino (\vu) + \mint (\vu)}{2} \mathbf{1}$ is the optimal solution. This is exactly the Shift-All solution as defined in the main text.

\end{proof}
}


\subsection{Proof for trilinearity}
\label{app:trilinear-proof}

\textbf{Claim:}
Let $\vd_{i,j,k} \in \mathbb{R}^K$, $i,j,k \in \{0,1\}$, be the vector-valued SDF evaluations at the vertices of a voxel cell, each satisfying the MDF constraint $\mino(\vd_{i,j,k}) + \mint(\vd_{i,j,k}) \ge 0$. Then the trilinearly interpolated field satisfies the constraint at every point of the cell.

\textbf{Proof:}
Let $S_2(\vd) = \mino(\vd) + \mint(\vd)$. We are given that $S_2(\vd_{i,j,k}) \ge 0$ for all $i,j,k \in \{0,1\}$. As established in \cref{th:linearity} (\cref{app:linterp-proof}), the linear interpolation $\vd_\alpha = \alpha \vd_a + (1-\alpha) \vd_b$ between any two valid vectors guarantees $S_2(\vd_\alpha) \ge 0$ for $\alpha \in [0,1]$. Trilinear interpolation for a point at local coordinates $(x,y,z) \in [0,1]^3$ decomposes into three sequential linear interpolations. By applying \cref{th:linearity} on $S_2$ at each step, the constraint is preserved:

\textbf{Step 1: $x$-axis.} Interpolating the 8 vertices along the $x$-axis yields four vector-valued functions along the four edges in the $x$ direction:
\[ \vd_{0,0}(x) = (1-x)\vd_{0,0,0} + x \vd_{1,0,0} \]
\[ \vd_{1,0}(x) = (1-x)\vd_{0,1,0} + x \vd_{1,1,0} \]
\[ \vd_{0,1}(x) = (1-x)\vd_{0,0,1} + x \vd_{1,0,1} \]
\[ \vd_{1,1}(x) = (1-x)\vd_{0,1,1} + x \vd_{1,1,1} \]

Since the endpoints satisfy the constraint, \cref{th:linearity} dictates that $S_2(\vd_{j,k}(x)) \ge 0$ for $j,k \in \{0,1\}$.

\textbf{Step 2: $y$-axis.} Interpolating these 4 along the $y$-axis yields two vector-valued functions on the $z=0$ and $z=1$ faces:
\[ \vd_0(x,y) = (1-y)\vd_{0,0}(x) + y \vd_{1,0}(x) \]
\[ \vd_1(x,y) = (1-y)\vd_{0,1}(x) + y \vd_{1,1}(x) \]

Applying \cref{th:linearity} again, $S_2(\vd_0(x,y)) \ge 0$ and $S_2(\vd_1(x,y)) \ge 0$.

\textbf{Step 3: $z$-axis.} Interpolating the 2 along the $z$-axis gives the composite field inside the full cube, defined at $(x,y,z)$:
\[ \vd(x,y,z) = (1-z)\vd_0(x,y) + z \vd_1(x,y) \]

By a final application of \cref{th:linearity}, $S_2(\vd(x,y,z)) \ge 0$, which proves that trilinearly interpolated values also satisfy the constraint.

\section{Additional Analyses}
\label{sec:additional_analyses}

\subsection{Compatibility with learning-based isosurface extraction methods}
\label{sec:rebut_extractors}

The deciding property for the use of our method is how the extractor produces geometry between the discrete samples: Our guarantee applies to any algorithm that places vertices by linear interpolation of the sampled values, since \cref{th:linearity} extends constraint~\eqref{eq:mdf_constraint} exactly along linearly interpolated segments.

This includes learning-based pipelines. DMTet~\cite{Shen21a} learns both the grid deformation and the per-vertex SDF values, but extracts surfaces with Marching Tetrahedra via linear interpolation, so our guarantee applies in full. Our MedTet experiments (\cref{tab:medtet}) run on exactly this setting. The same holds for deformable and hierarchical grids in general, provided extraction remains linearly interpolated.

Methods that instead learn the geometry placement itself fall outside our guarantee: Neural Marching Cubes~\cite{Chen21c} and Neural Dual Contouring~\cite{Chen22b} predict vertex positions with a network, and FlexiCubes~\cite{Shen23} repositions crossings with learned weights, so the extracted triangles are no longer tied to the linearly interpolated field. Even there, our constraint would guarantee an intersection-free interpolated field if used, but the transfer of this property to the extracted mesh would be lost.

Finally, linear interpolation is a choice we made to be compatible with existing algorithms that mesh each SDF channel independently. An alternative direction would be to mesh all channels jointly, as in multi-object Dual Contouring~\cite{Rashid23}. This would complement our approach without substituting for it. Such a mesh extraction algorithm would enforce non-intersection, but only in the extracted mesh. In contrast, our constraint yields an intersection-free field, which is what enables its use during training in SDF-based rendering and optimization where no mesh is involved, and with any linearly interpolated extractor.

\subsection{Marching Tetrahedra and Marching Cubes}
\label{sec:rebut_mc}

For \textbf{Marching Tetrahedra}, the field inside each tetrahedron is linearly interpolated from the values at its four vertices. As in \cref{app:trilinear-proof}, this interpolation decomposes into three sequential linear interpolations, adding one vertex at a time, so that by \cref{th:linearity} the interpolated field satisfies the constraint everywhere inside the tetrahedron. The extracted triangles always lie directly on the interpolated surfaces, and the absence of intersections is guaranteed.

For \textbf{Marching Cubes}, the interpolated field inside each cell is trilinear and satisfies the constraint everywhere, as proven in \cref{app:trilinear-proof}. The vertices lie exactly on the interpolated surface, but the full triangles do not always. This introduces two considerations, both of which are properties of Marching Cubes itself and not related to our method:

\begin{itemize}
    \item There are ambiguous configurations in Marching Cubes. When these are picked inconsistently between objects in a cell, as can happen in the original approach~\cite{Lorensen87}, unwarranted intersections can appear even in valid fields. However, newer implementations, such as~\cite{Lewiner03}, use a calculation called ``Asymptotic Decider'' to choose a particular configuration in ambiguous cases. This yields consistent reconstructions without intersections in all practical cases.
    \item Intersections are still possible on very low-resolution grids, even on fields that are guaranteed not to intersect, due to the placement of triangles. This can happen when one object is locally concave within a single cell, and its triangles ``bow'' outward towards the other object. However, in practice, this requires significant bending inside a single cell and does not happen when the grid resolution is large enough.
\end{itemize}

In short, if theoretical guarantees are needed, we recommend using Marching Tetrahedra. Nevertheless, for all practical purposes, our method is also compatible with Marching Cubes as shown in our experiments.

\subsection{Surface analysis for Shift-All}
\label{sec:rebut_shiftall}
\begin{table}[ht]
\centering
\small
\begin{tabular}{ll|c}
    Property & Metric & Value \\
    \midrule
    Thickening & Min. SDF change over grid & $0.000$ \\
    \midrule
    \multirow{2}{*}{Thinning} & IV & $2.0 \times 10^{-4} \rightarrow {\sim}0$ \\
     & Removed volume $\div$ initial IV & $\approx 1.0$ \\
    \midrule
    \multirow{3}{*}{Displacement} & Unmoved surface & $89\%$ \\
     & Surface moved $< 10^{-2}$ & ${>}99.9\%$ \\
     & Max. move & $1.9 \times 10^{-2}$ \\
    \midrule
    \multirow{5}{*}{Contact} & Chamfer dist. in contact regions & $2.5 \times 10^{-3} \rightarrow 2.5 \times 10^{-3}$ \\
     & Surface inside another part & $6.40\% \rightarrow 0.00\%$ \\
     & Surface within $10^{-2}$ of another part & $30.1\% \rightarrow 30.2\%$ \\
     & Surface within $10^{-3}$ of another part & $6.6\% \rightarrow 9.1\%$ \\
     & Surface within $10^{-4}$ of another part & $0.6\% \rightarrow 3.2\%$ \\
\end{tabular}
\setlength{\abovecaptionskip}{7pt}
\caption{Shift-All as post-processing on our trained heart model (MM-WHS). Arrows indicate values before $\rightarrow$ after Shift-All. Distances are in the $[-1, 1]$ object scale, so the largest move is about 1\% of it. Distances to other parts are unsigned, so penetrating surface is already counted in a band when its depth is below the threshold.}
\label{tab:shiftall_heart}
\end{table}

We perform an additional experiment on our trained heart model (MM-WHS) to analyze the effect of Shift-All on the surfaces, and report the results in \cref{tab:shiftall_heart}. Shift-All does not thicken the shapes and removes their overlaps, while the resulting displacement and thinning are small, local to the contact regions, and do not reduce accuracy near contacts.

At every grid point where two shapes overlap (the two smallest signed distances sum to a negative value), Shift-All adds the same positive constant to all channels so that the sum becomes zero; points not in violation are left untouched.

\textbf{Thickening.} Shift-All only adds non-negative constants to the signed distances, so every shape can only shrink, and previously separated shapes cannot be pushed into new overlaps (row 1).

\textbf{Thinning.} Shapes lose volume only where they were overlapping: summed over shapes, the removed volume matches the overlap that was present, so each shape roughly contributes half the reduction (rows 2--3).

\textbf{Displacement.} Points not in violation are untouched, so most of the surface does not move. The part that moves lies inside the overlap region and moves by at most $1.9 \times 10^{-2}$ (rows 4--6).

\textbf{Contact.} Shift-All moves both overlapping surfaces to where their fields cross, so that the shapes are brought into contact and not pushed apart. Accuracy in contact regions, measured against the ground truth that has no overlaps, is unchanged (row 7).

This experiment uses a vanilla model with Shift-All as post-processing; for models trained with Shift-All, the pre-shift field lacks supervision and is not meaningful for this comparison.

\end{document}